\documentclass{article}

     \PassOptionsToPackage{numbers, compress}{natbib}

     \usepackage[final]{neurips_2019}

\usepackage[utf8]{inputenc} 
\usepackage[T1]{fontenc}    
\usepackage{hyperref}       
\usepackage{url}            
\usepackage{booktabs}       
\usepackage{amsfonts}       
\usepackage{nicefrac}       
\usepackage{microtype}      

\usepackage{multicol}
\usepackage[utf8]{inputenc}
\usepackage{times}
\usepackage{graphicx} 
\usepackage{subfigure} 
\usepackage{amsmath}
\usepackage{algorithm}
\usepackage{algorithmic}
\usepackage{amssymb,amsmath,amsfonts}
\usepackage{mathrsfs}
\usepackage{color}
\usepackage{dsfont}
\newcommand{\norm}[1]{\left\lVert#1\right\rVert}
\newcommand{\expect}[1]{\mathbb{E}\left[{#1}\right]}
\newcommand{\prob}[1]{\mathbb{P}\left[{#1}\right]}
\newcommand{\given}{\; \big\vert \;} 
\newcommand{\bydef}{:=}
\newcommand{\inner}[2]{\langle #1, #2 \rangle}

\usepackage{subfigure}
\usepackage{mathnotations2}
\graphicspath{{./Figures/}}

\newtheorem{mytheorem}{Theorem}

\newtheorem{mylemma}{Lemma}

\newtheorem{mydefinition}{Definition}

\newcommand{\BlackBox}{\rule{1.5ex}{1.5ex}}
\newenvironment{proof}{\par\noindent{\bf Proof\ }}{\hfill\BlackBox\\[2mm]}
\newcommand{\beq}{\begin{equation}}
\newcommand{\eeq}{\end{equation}}
\newcommand{\beqn}{\begin{equation*}}
\newcommand{\eeqn}{\end{equation*}}
\newcommand{\beqa}{\begin{eqnarray}}
\newcommand{\eeqa}{\end{eqnarray}}
\newcommand{\beqan}{\begin{eqnarray*}}
\newcommand{\eeqan}{\end{eqnarray*}}

\newcommand{\argmax}{\mathop{\mathrm{argmax}}}

\title{Bayesian Optimization under Heavy-tailed Payoffs}

\author{
  Sayak Ray Chowdhury\\
  Department of ECE\\
  Indian Institute of Science\\
  Bangalore, India 560012\\
  \texttt{sayak@iisc.ac.in} 
  \And
  Aditya Gopalan \\
  Department of ECE\\
  Indian Institute of Science\\
  Bangalore, India 560012 \\
  \texttt{aditya@iisc.ac.in} 
}

\begin{document}

\maketitle
\vspace*{-1mm}
\begin{abstract}
\vspace*{-3mm}
 	We consider black box optimization of an unknown function in the nonparametric Gaussian process setting when the noise in the observed function values can be heavy tailed. This is in contrast to existing literature that typically assumes sub-Gaussian noise distributions for queries. Under the assumption that the unknown function belongs to the Reproducing Kernel Hilbert Space (RKHS) induced by a kernel, we first show that an adaptation of the well-known GP-UCB algorithm with reward truncation enjoys sublinear $\tilde{O}(T^{\frac{2 + \alpha}{2(1+\alpha)}})$ regret even with only the $(1+\alpha)$-th moments, $\alpha \in (0,1]$, of the reward distribution being bounded ($\tilde{O}$ hides logarithmic factors). However, for the common squared exponential (SE) and Mat\'{e}rn kernels, this is seen to be significantly larger than a fundamental $\Omega(T^{\frac{1}{1+\alpha}})$ lower bound on regret. We resolve this gap by developing novel Bayesian optimization algorithms, based on kernel approximation techniques, with regret bounds matching the lower bound in order for the SE kernel. We numerically benchmark the algorithms on environments based on both synthetic models and real-world data sets.  
\end{abstract}

\vspace*{-6mm}
\section{Introduction}
\label{sec:intro}
\vspace*{-2mm}
Black-box optimization of an unknown function $f:\Real^d \ra \Real$ with expensive, noisy queries is a generic problem arising in domains such as hyper-parameter tuning for complex machine learning models \cite{BerBen12:randomHypParamTun}, sensor selection \cite{GarOsbRob10:BOsensor}, synthetic gene design \cite{gonzalez2015bayesian}, experimental design
etc. The popular Bayesian optimization (BO) approach, towards solving this problem, starts with a prior distribution, typically a nonparametric Gaussian process (GP), over a function class, uses function evaluations to compute the posterior
distribution over functions, and chooses the next function evaluation adaptively -- using a sampling strategy -- towards reaching the optimum. Popular sampling strategies include expected improvement \cite{movckus1975bayesian}, probability of improvement \cite{wang2016optimization}, upper confidence bounds \cite{srinivas2010gaussian}, Thompson sampling \cite{chowdhury2017kernelized}, predictive-entropy search \cite{hernandez2014predictive}, etc. 


The design and analysis of adaptive sampling strategies for BO typically involves the assumption of bounded, or at worst sub-Gaussian, distributions for rewards (or losses) observed by the learner, which is quite light-tailed. Yet, many real-world environments are known to exhibit heavy-tailed behavior, e.g., the distribution of delays in data networks is inherently heavy-tailed especially with highly variable or bursty traffic flow distributions that are well-modeled with heavy tails \cite{JagMarModTsi14:heavytailsched}, heavy-tailed price fluctuations are common in finance and insurance data \cite{resnick2007heavy}, properties of complex networks often exhibit heavy tails such as degree distribution \cite{strogatz2001exploring}, etc. This motivates studying methods for Bayesian optimization when observations are significantly heavy tailed compared to Gaussian.



A simple version of black box optimization -- in the form of online learning in finite multi-armed bandits (MABs) -- with heavy-tailed payoffs, was first studied rigorously by \citet{bubeck2013bandits}, where the payoffs are assumed to have bounded $(1+\alpha)$-th moment for $\alpha \in (0,1]$. They showed that for MABs with only finite variances (i.e., $\alpha = 1$), by using statistical estimators that are
more robust than the empirical mean, one can still recover the optimal regret rate for MAB under the sub-Gaussian assumption. 
Moving further, \citet{medina2016no} consider these estimators for the problem of linear (parametric) stochastic bandits under heavy-tailed rewards and \citet{shao2018almost} show that almost optimal algorithms can be designed by using an optimistic, data-adaptive truncation of rewards. Some other important works include pure exploration under heavy-tailed noise \cite{yu2018pure}, payoffs with bounded kurtosis \cite{lattimore2017scale}, extreme bandits \cite{carpentier2014extreme}, heavy tailed payoffs with $\alpha \in (0,\infty)$  \cite{vakili2013deterministic}.


Against this backdrop, we consider regret minimization with heavy-tailed reward distributions in bandits with a potentially continuous arm set, and whose (unknown) expected reward function is nonparametric assumed to have smoothness compatible with a kernel on the arm set. Here, it is unclear if existing BO techniques relying on statistical confidence sets based on sub-Gaussian observations can be made to work to attain nontrivial regret, since it is unlikely that these confidence sets will at all be correct. 
%
It is worth mentioning that in the finite dimensional setting, \citet{shao2018almost} solve the problem almost optimally, but their results do not carry over to the general nonparametric kernelized setup since their algorithms and regret bounds depend crucially on the finite feature dimension.  We answer this affirmatively in this work, and formalize and solve BO under heavy tailed noise almost optimally. Specifically, this paper makes the following contributions.
\begin{itemize}
\vspace*{-1mm}
\item  We adapt the GP-UCB algorithm to heavy-tailed payoffs by a truncation step, and show that it enjoys a regret bound of $\tilde{O}(\gamma_TT^{\frac{2 + \alpha}{2(1+\alpha)}})$ where $\gamma_T$ depends on the kernel associated with the RKHS and is generally sub-linear in $T$. This regret rate, however, is potentially sub-optimal due to a $\Omega(T^{\frac{1}{1+\alpha}})$ fundamental lower bound on regret that we show for two specific kernels, namely the squared exponential (SE) kernel and the Mat\'{e}rn kernel. 
\vspace*{-1mm}
\item We develop a new Bayesian optimization algorithm by truncating rewards in each direction of an approximate, finite-dimensional feature space. We show that the feature approximation can be carried out by two popular kernel approximation techniques: Quadrature Fourier features \cite{mutny2018efficient} and  Nystr\"{o}m approximation \cite{calandriello2019gaussian}. The new algorithm under either approximation scheme gets regret $\tilde{O}(\gamma_TT^{\frac{1}{1+\alpha}})$, which is optimal upto log factors for the SE kernel.
\vspace*{-1mm}
\item  Finally, we report numerical results based on experiments on synthetic as well as real-world based datasets, for which the algorithms we develop are seen to perform favorably in the harsher heavy-tailed environments.
%
\end{itemize}
\vspace*{-2mm}
{\bf Related work.}
An alternative line of work uses approaches for black box optimization based on Lipschitz-type smoothness structure \cite{kleinberg2008multi,bubeck2011x,azar2014online,sen2019noisy}, which is qualitatively different from RKHS smoothness type assumptions. Recently, \citet{bogunovic2018adversarially} consider GP optimization under an adversarial perturbation of the query points. But, the observation noise is assumed to be Gaussian unlike our heavy-tailed environments. Kernel approximation schemes in the context of BO usually focuses on reducing the cubic cost of gram matrix inversion \cite{wang2017max,wang2017batched,mutny2018efficient,calandriello2019gaussian}. However, we crucially use these approximations to achieve optimal regret for BO under heavy tailed noise, which, we believe, might not be possible without resorting to the kernel approximations.

\vspace*{-3mm}
\section{Problem formulation}
\label{sec:Problem-statement}
\vspace*{-2mm}
Let $f:\cX \ra \Real$ be a fixed but unknown function over a domain $\cX \subset \Real^d$ for some $d \in \mathbb{N}$. At every round, a learner queries $f$ at a single point $x_t \in \cX$, and observes a noisy payoff $y_t=f(x_t) + \eta_t$. 
Here the noise sequence $\eta_t,t\ge 1$ are assumed to be zero mean i.i.d. random variables such that the payoffs satisfy $\expect{\abs{y_t}^{1+\alpha}|\cF_{t-1}} \le v$ for some $\alpha \in (0,1]$ and $v \in (0,\infty)$, where $\cF_{t-1} = \sigma(\lbrace x_\tau,y_\tau)\rbrace_{\tau=1}^{t-1},x_t)$ denotes the $\sigma$-algebra generated by the events so far\footnote{If instead the moment bound holds for each $\eta_t$ then this can be translated to a moment bound for each $y_t$ using, say, a bound on $f(x)$.}. Observe that this bound on the $(1+\alpha)$-th moment at best yields  bounded variance for $y_t$, and does not necessarily mean that $y_t$ (or $\eta_t$) is sub-Gaussian as is assumed typically. The query point $x_t$ at round $t$ is chosen causally depending upon the history $\lbrace (x_s,y_s)\rbrace_{s=1}^{t-1}$ of query and payoff sequences available up to round $t-1$. The learner's goal is to maximize its (expected) cumulative reward $\sum_{t=1}^{T} f(x_t)$ over a time horizon $T$ or equivalently minimize its cumulative {\em regret} $R_T = \sum_{t=1}^T \left(f(x^\star)-f(x_t)\right)$, where $x^\star \in \argmax_{x\in \cX}f(x)$ is a maximum point of $f$ (assuming the maximum is attained; not necessarily unique). A sublinear growth of $R_T$ with $T$ implies the time-average regret $R_T/T \ra 0$ as $T\ra \infty$. 

\textbf{Regularity assumptions:} Attaining sub-linear regret is impossible in general for arbitrary reward functions $f$, and thus some regularity assumptions are needed. In this paper, we assume smoothness for $f$ induced by the structure of a kernel on $\cX$. Specifically, we make the standard assumption of a p.s.d. kernel $k: \cX \times \cX \to \mathbb{R}$ such that $k(x,x) \le 1$ for all $x \in \cX$, and $f$ being an element of the reproducing kernel Hilbert space (RKHS) $\cH_k(\cX)$ of smooth real valued functions on $\cX$. Moreover, the RKHS norm of $f$ is assumed to be bounded, i.e., $\norm{f}_\cH \le B$ for some $B < \infty$. Boundedness of $k$ along the diagonal holds for any stationary kernel, i.e., where $k(x,x')=k(x-x')$, e.g., the \textit{Squared Exponential} kernel $k_{\text{SE}}$ and the \textit{Mat$\acute{e}$rn} kernel $k_{\text{Mat\'{e}rn}}$:
\beqn
k_{\text{SE}}(x,x') = \exp\left(-\frac{r^2}{2l^2}\right)\quad \text{and} \quad
k_{\text{Mat\'{e}rn}}(x,x') = \frac{2^{1-\nu}}{\Gamma(\nu)}\left(\frac{r\sqrt{2\nu}}{l}\right)^\nu B_\nu\left(\frac{r\sqrt{2\nu}}{l}\right),
\eeqn
where $l > 0$ and $\nu > 0$ are hyperparameters of the kernels, $r = \norm{x-x'}_2$ is the distance between $x$ and $x'$, and $B_\nu$ is the modified Bessel function.

\vspace*{-3mm}
\section{Warm-up: the first algorithm}
\label{sec:warmup}
\vspace*{-2mm}
Towards designing a BO algorithm for heavy tailed observations, we briefly recall the standard GP-UCB algorithm for the sub-Gaussian setting. GP-UCB at time $t$ chooses the point $x_t = \argmax_{x \in \cX} \mu_{t-1}(x)+\beta_t \sigma_{t-1}(x)$ where $\mu_t(x) = k_t(x)^T(K_t + \lambda I_t)^{-1}Y_t$ and $\sigma_t^2(x) =  k(x,x) - k_t(x)^T(K_t + \lambda I_t)^{-1} k_t(x)$ are the posterior mean and variance functions after $t$ observations from a function drawn from the GP prior $GP_\cX(0,k)$, with additive i.i.d. Gaussian noise $\cN(0,\lambda)$. Here $Y_t = [y_1,\ldots,y_t]^T$ is the vector formed by observations, $K_t = [k(u,v)]_{u,v \in \cX_t}$ is the kernel matrix, $k_t(x) =  [k(x_1,x),\ldots,k(x_t,x)]^T$ and $I_t$ is the identity matrix of order $t$. If the noise $\eta_t$ is assumed conditionally $R$-sub-Gaussian, i.e., $ \expect{e^{\gamma \eta_t} \given \cF_{t-1}} \le \exp\left(\frac{\gamma^2R^2}{2}\right)$ for all $\gamma \in \Real$, then using $\beta_{t+1}=O\left(R\sqrt{\ln \abs{I_t+\lambda^{-1}K_t}}\right)$ ensures $\tilde{O}(\sqrt{T})$ regret \cite{chowdhury2017kernelized}, as the posterior GP concentrates rapidly on the true function $f$. However, when the sub-Gaussian assumption does not hold, we cannot expect the posterior GP to have such nice concentration property. In fact, it is known that the ridge regression estimator $\mu_t \in \cH_k(\cX)$ of $f$ is not robust when the noise exhibits heavy fluctuations \cite{hsu2014heavy}. So, in order to tackle heavy tailed noise, one needs more robust estimates $\hat{\mu}_t$ of $f$ along with suitable confidence sets. A natural idea to curb the effects of heavy fluctuations is to truncate high rewards \cite{bubeck2013bandits}. Our first algorithm Truncated GP-UCB (Algorithm \ref{algo:tgp-ucb}) is based on this idea.

\textbf{Truncated GP-UCB (TGP-UCB) algorithm:} 

\begin{minipage}[t]{7cm}
\vspace*{-3mm} 
  \vspace{0pt}  
At each time $t$, we truncate the reward $y_t$ to zero if it is larger than a suitably chosen truncation level $b_t$, i.e., we set the truncated reward $\hat{y}_t=y_t\mathds{1}_{\abs{y_t}\le b_t}$. Then, we construct the truncated version of the posterior mean as $\hat{\mu}_t(x)=k_t(x)^T(K_t+\lambda I_t)^{-1}\hat{Y}_t$ where $\hat{Y}_t=[\hat{y}_1,\ldots,\hat{y}_t]^T$ and simply run GP-UCB with $\hat{\mu}_t$ instead of $\mu_t$. The truncation level $b_t$ can be adapted with time $t$. We choose an increasing sequence of $b_t$'s, i.e., as time progresses and confidence interval shrinks, we truncate more and more 
   \end{minipage}
   \begin{minipage}[t]{7cm}
   \vspace*{-12mm} 
     \vspace{0pt}
     \begin{algorithm}[H]
        \renewcommand\thealgorithm{1}
        \caption{Truncated GP-UCB (TGP-UCB)}\label{algo:tgp-ucb}
        \begin{algorithmic}
        \STATE \textbf{Input:} Parameters $\lambda > 0$, $\lbrace b_t \rbrace_{t \ge 1}$, $\lbrace \beta_t \rbrace_{t \ge 1}$
        \STATE Set $\hat{\mu}_0(x) = 0$ and $\sigma_0^2(x) = k(x,x) \forall x \in \cX$ 
        \FOR{$t = 1, 2, 3 \ldots$}
        \STATE Play $x_t = \argmax_{x\in \cX} \hat{\mu}_{t-1}(x) + \beta_{t}\sigma_{t-1}(x)$ and observe payoff $y_t$
        \STATE Set $\hat{y}_t=y_t \mathds{1}_{\abs{y_t}\le b_t}$ and $\hat{Y}_t = [\hat{y}_1,\ldots,\hat{y}_t]^T$
        \STATE Compute  $\hat{\mu}_t(x) = k_t(x)^T(K_t+\lambda I_t)^{-1}\hat{Y}_t$ and $\sigma_t^2(x)=k_t(x)^T(K_t+\lambda I_t)^{-1}k_t(x)$
        \ENDFOR
        \end{algorithmic}
        \addtocounter{algorithm}{-1}
        \end{algorithm}
        \end{minipage}
         aggressively. Finally, in order to account for the bias introduced by truncation, we blow up the confidence width $\beta_t$ of GP-UCB by a multiplicative factor of $b_t$ so that $f(x)$ is contained in the interval $\hat{\mu}_{t-1}(x) \pm\beta_{t} \sigma_{t-1}(x)$ with high probability. This helps us to obtain a sub-linear regret bound for TGP-UCB given in the Theorem \ref{thm:regret-bound-tgp-ucb}, with a full proof deferred to appendix B.
         \vspace*{-1.5mm}
\begin{mytheorem}[Regret bound for TGP-UCB]
  Let $f \in \cH_k(\cX)$, $\norm{f}_\cH \le B$ and $k(x,x) \le 1$ for all $x \in \cX$. Let $\expect{\abs{y_t}^{1+\alpha}|\cF_{t-1}} \le v < \infty$ for some $\alpha \in (0,1]$ and for all $t \ge 1$. Then, for any $\delta \in (0,1]$, TGP-UCB, with $b_t=v^{\frac{1}{1+\alpha}}t^{\frac{1}{2(1+\alpha)}}$ and $\beta_{t+1} = B + \frac{3}{\sqrt{\lambda}} \;b_t\sqrt{\ln \abs{I_t+\lambda^{-1}K_t}+2\ln(1/\delta)}$, enjoys, with probability at least $1-\delta$, the regret bound
   \beqn
   R_T = O\left(B\sqrt{T\gamma_T}+v^{\frac{1}{1+\alpha}}\sqrt{\gamma_T\left(\gamma_T+\ln(1/\delta)\right)} T^{\frac{2+\alpha}{2(1+\alpha)}}\right),
   \eeqn
   where $\gamma_T\equiv \gamma_T(k,\cX) = \max_{A \subset \cX : \abs{A}=t} \frac{1}{2} \ln \abs{I_t + \lambda^{-1}K_A}$.
 \label{thm:regret-bound-tgp-ucb}  
\end{mytheorem}
\vspace*{-1.5mm}
Here, $\gamma_T$ denotes the \textit{maximum information gain} about any $f\sim GP_{\cX}(0, k)$ after $T$ noisy observations obtained by passing $f$ through an i.i.d. Gaussian channel $\cN(0,\lambda)$, and measures the reduction in the uncertainty of $f$ after $T$ noisy observations. It is a property of the kernel $k$ and domain $\cX$, e.g., if $\cX$ is compact and convex, then $\gamma_T=O\left((\ln T)^{d+1}\right)$ for $k_{\text{SE}}$ and  $O\big(T^{\frac{d(d+1)}{2\nu+d(d+1)}}\ln T\big)$ for $k_{\text{Mat\'{e}rn}}$ \cite{srinivas2010gaussian}.

\vspace*{-1.5mm}
\textbf{\textit{Remark 1.}} 
An $R$-sub-Gaussian environment satisfies the moment condition with $\alpha = 1$ and $v = R^2$, so the result implies a sub-linear $\tilde{O}(T^{3/4})$ regret bound for TGP-UCB in sub-Gaussian environments. 


  
\vspace*{-3mm}
\section{Regret lower bound}
\vspace*{-2mm}
Establishing lower bounds under general kernel smoothness structure  is an open problem even when the payoffs are Gaussian. Similar to \citet{scarlett2017lower}, we only focus on the SE and
Mat\'{e}rn kernels.
\vspace*{-2mm}
\begin{mytheorem}[Lower bound on cumulative regret]
Let $\cX=[0,1]^d$ for some $d \in \mathbb{N}$. Fix a kernel $k \in \{k_{\text{SE}}, k_{\text{Mat\'{e}rn}}\}$, $B > 0$, $T \in \mathbb{N}$, $\alpha \in (0,1]$ and $v > 0$. Given any algorithm, there exists a function $f \in \cH_k(\cX)$ with $\norm{f}_{\cH} \le B$, and a reward distribution satisfying $\expect{\abs{y_t}^{1+\alpha}|\cF_{t-1}} \le v$ for all $t \in [T]:=\lbrace 1,2,\ldots,T \rbrace$, such that when the algorithm is run with this $f$ and reward distribution, its regret satisfies
\begin{enumerate}
\vspace*{-2mm}
\item $\mathbb{E}[R_T] = \Omega \left(v^{\frac{1}{1+\alpha}}\left(\ln \left(v^{-\frac{1}{\alpha}}B^{\frac{1+\alpha}{\alpha}} T\right)\right)^{\frac{d\alpha}{1+\alpha}}T^{\frac{1}{1+\alpha}}\right)$ if $k=k_{\text{SE}}$,
\vspace*{-2mm}
\item $\mathbb{E}[R_T]=\Omega \left( v^{\frac{\nu}{\nu(1+\alpha)+d\alpha}} \, B^{\frac{d\alpha}{\nu(1+\alpha)+d\alpha}} \, T^{\frac{\nu+d\alpha}{\nu(1+\alpha)+d\alpha}}\right)$ if $k=k_{\text{Mat\'{e}rn}}$.
\end{enumerate}
\label{thm:lower-bound}
\end{mytheorem}
\vspace*{-3mm}
The proof argument is inspired by that of \citet{scarlett2017lower}, which provides the lower bound of BO under i.i.d. Gaussian noise, but with nontrivial changes to account for heavy tailed observations. 
The proof is based on constructing a finite subset of ``difficult'' functions in $\cH_k(\cX)$. Specifically, we choose $f$ as a uniformly sampled function from a finite set $\lbrace f_1,\ldots,f_M \rbrace$, where each $f_j$ is obtained by shifting a common function $g \in \cH_k(\Real^d)$ by a different amount such that each of these has a unique maximum, and then cropping to $\cX=[0,1]^d$. $g$ takes values in $[-2\Delta,2\Delta]$ with the maximum attained at $x=0$. The function $g$ is constructed properly, and the parameters $\Delta$, $M$ are chosen appropriately based on the kernel $k$, fixed constants $B,T,\alpha,v$ such that any $\Delta$-optimal point for $f_j$ fails to be $\Delta$-optimal point for any other $f_{j'}$ and that $\norm{f_j}_{\cH} \le B$ for all $j \in [M]$. The reward function takes values in $\lbrace sgn\left(f(x)\right)\left(\frac{v}{2\Delta}\right)^{\frac{1}{\alpha}},0\rbrace$, with the former occurring with probability $\left(\frac{2\Delta}{v}\right)^{\frac{1}{\alpha}}\abs{f(x)}$, such that, for every $x \in \cX$, the expected reward is $f(x)$ and $(1+\alpha)$-th raw moment is upper bounded by $v$. Now, if we can lower bound the regret averaged over $j \in [M]$, then there must exist some $f_j$ for which the bound holds. The formal proof is deferred to Appendix C. 

\textit{\textbf{Remark 2.}} Theorem \ref{thm:lower-bound} suggests that (a) TGP-UCB may be
suboptimal, and (b) for the SE kernel, it may be possible to design algorithms recovering $\tilde{O}(\sqrt{T})$ regret bound under finite variances ($\alpha=1$).



\vspace*{-3mm}
\section{An optimal algorithm under heavy tailed rewards}
\vspace*{-2mm}
In view of the gap between the regret bound for TGP-UCB and the fundamental lower bound, it is possible that TGP-UCB (Algorithm \ref{algo:tgp-ucb}) does not completely mitigate the effect of heavy-tailed fluctuations, and perhaps that truncation in a different domain may work better. In fact, for parametric linear bandits (i.e., BO with finite dimensional linear kernels), it has been shown that appropriate truncation in feature space improves regret performance as opposed to truncating raw observations \cite{shao2018almost}, and in this case the feature dimension explicitly appears in the regret bound. %
However, the main challenge in the more general nonparametric setting is that the feature space is infinite dimensional, which would yield a trivial regret upper bound. 
If we can find an approximate feature map $\tilde{\phi}:\cX \ra \Real^m$ in a low-dimensional Euclidean inner product space $\Real^m$ such that $k(x,y)\approx \tilde{\phi}(x)^T\tilde{\phi}(y)$, then we can perform the above feature adaptive truncation effectively as well as keep the error introduced due to approximation in control. Such a kernel approximation can be done efficiently either in a data independent way (Fourier features approximation \citep{rahimi2008random}) or in a data dependent way (Nystr\"{o}m approximation \cite{drineas2005nystrom}) and has been used in the context of BO to reduce the time complexity of GP-UCB \cite{mutny2018efficient,calandriello2019gaussian}. But in this work, the approximations are crucial to obtain optimal theoretical guarantees. We now describe our algorithm Adaptively Truncated Approximate GP-UCB (Algorithm \ref{algo:itgp-ucb}).

\vspace*{-1mm}
\textbf{Adaptively Truncated Approximate GP-UCB (ATA-GP-UCB) algorithm:}
At each round $t$, we select an arm $x_{t}$ which maximizes the approximate (under kernel approximation) GP-UCB score $\tilde{\mu}_{t-1}(x) + \beta_{t}\tilde{\sigma}_{t-1}(x)$, where $\tilde{\mu}_{t-1}(x)$ and $\tilde{\sigma}^2_{t-1}(x)$ denote approximate posterior mean and variance from the previous round, respectively and $\beta_{t}$ is an appropriately
chosen confidence width. Then, we update $\tilde{\mu}_{t}(x)$ and $\tilde{\sigma}^2_{t}(x)$ as follows. First,
we find a feature embedding $\tilde{\phi}_t \in \Real^{m_t}$, of some appropriate dimension $m_t$, which approximates the kernel efficiently. Then, we find the rows $u_1^T,\ldots,u_{m_t}^T$ of the matrix $\tilde{V}_t^{-1/2}\tilde{\Phi}_t ^T$, where $\tilde{\Phi}_t=[\tilde{\phi}_t(x_1),\ldots,\tilde{\phi}_t(x_t)]^T$ and $\tilde{V}_t =\tilde{\Phi}_t ^T \tilde{\Phi}_t + \lambda I_{m_t}$, and use those as the weight vectors for truncating the rewards in each of $m_t$ directions by setting $\hat{r}_{i}=\sum_{\tau =1}^{t}u_{i,\tau}y_\tau \mathds{1}_{\abs{u_{i,\tau}y_\tau} \le b_t}$ for all $i \in [m_t]$, where $b_t$ specifies the truncation level. Then, we find our estimate of $f$ as $\tilde{\theta}_t=\tilde{V}_t^{-1/2}[\hat{r}_{1},\ldots,\hat{r}_{m_t}]^T$. Finally, we approximate the posterior mean as $\tilde{\mu}_t(x)=\tilde{\phi}_t(x)^T \tilde{\theta}_t$ and the posterior variance as
$(i)\:\tilde{\sigma}_t^2(x)=
         \lambda \tilde{\phi_t}(x)^T \tilde{V}_t^{-1}\tilde{\phi_t}(x)$ for the Fourier features approximation, or as
         $(ii)\; \tilde{\sigma}_t^2(x)=k(x,x)-\tilde{\phi_t}(x)^T\tilde{\phi_t}(x)+\lambda \tilde{\phi_t}(x)^T \tilde{V}_t^{-1}\tilde{\phi_t}(x)$ for the Nystr\"{o}m approximation. 
         Now it only remains to describe how to find the feature embeddings $\tilde{\phi}_t$. 
          
\textbf{(a) Quadrature Fourier features (QFF) approximation:} If $k$ is a bounded, continuous, positive definite, stationary kernel satisfying $k(x,x)=1$, then by Bochner's theorem \cite{bochner1959lectures}, $k$ is the Fourier transform of a probability measure $p$, i.e., $k(x,y)=\int_{\Real^d}p(\omega)\cos(\omega^T(x-y))d\omega$.
For the SE kernel, this measure has density $p(\omega) = \big(\frac{l}{\sqrt{2\pi}}\big)^d e^{-\frac{l^2\norm{\omega}_2^2}{2}}$ (abusing notation for measure and density).
\citet{mutny2018efficient} show that for any stationary kernel $k$ on $\Real^d$ whose inverse Fourier transform decomposes product wise, i.e., $p(\omega)=\prod_{j=1}^{d}p_j(\omega_j)$, we can 
use Gauss-Hermite quadrature \cite{hildebrand1987introduction} to approximate it. If $\cX=[0,1]^d$, the SE kernel is approximated as follows. Choose $\bar{m} \in \mathbb{N}$ and $m=\bar{m}^d$, and construct the $2m$-dimensional feature map
\beq
 \tilde{\phi}(x)_i=\begin{cases}
    \sqrt{\nu(\omega_{i})} \cos\left(\frac{\sqrt{2}}{l}\omega_i^Tx\right) & \text{if}\; 1 \le i \le m,\\
    \sqrt{\nu(\omega_{i-m})} \sin\left(\frac{\sqrt{2}}{l}\omega_{i-m}^Tx\right) & \text{if}\; m+1 \le i \le 2m.
  \end{cases}
  \label{eqn:qff-embedding}
\eeq
Here the set $\lbrace \omega_1,\ldots,\omega_m \rbrace= \overbrace{A_{\bar{m}} \times\cdots \times A_{\bar{m}}}^{\text{$d$ times}}$, where $A_{\bar{m}}$ is the set of $\bar{m}$ (real) roots of the $\bar{m}$-th Hermite polynomial $H_{\bar{m}}$, and $\nu(z)=\prod_{j=1}^{d}\frac{2^{\bar{m}-1}\bar{m}!}{\bar{m}^2H_{\bar{m}-1}(z_j)^2}$ for all $z \in \Real^d$. For our purposes, we will have ATA-GP-UCB work with the embedding $\tilde{\phi}_t(x)=\tilde{\phi}(x)$ of dimension $m_t=2m$ for all $t \ge 1$.

\textit{\textbf{Remark 3.}} The seminal work of \citet{rahimi2008random} that develops random Fourier feature (RFF) approximation of 
  any stationary kernel is based on the feature map
$
\tilde{\phi}(x)=\frac{1}{\sqrt{m}}[\cos(\omega_1^Tx),\ldots,\cos(\omega_m^Tx),\sin(\omega_1^Tx),\ldots,\sin(\omega_m^Tx)]^T$, where each $\omega_i$ is sampled independently from $p(\omega)$.
However, RFF embeddings do not appear to be useful for our purpose of achieving sublinear regret (see discussion after Lemma \ref{lem:func-conc-ata-gp-ucb}), so we work with the QFF embedding.

\textbf{(b) Nystr\"{o}m approximation:} Unlike the QFF approximation where the basis functions (cosine and sine) do not depend on the data, the basis functions used by
the Nystr\"{o}m method are data dependent.
For a set of points $\cX_t =\lbrace x_1,\ldots,x_t \rbrace$, the Nystr\"{o}m method \cite{yang2012nystrom} approximates the kernel matrix $K_t$ as follows: First  sample a random number $m_t$ of points from $\cX_t$ to construct a dictionary $\cD_t =\lbrace x_{i_1},\ldots,x_{i_{m_t}}\rbrace ; i_j\in [t]$, according to the following distribution. For each $i \in [t]$, include $x_i$ in $\cD_t$ independently with probability $p_{t,i} = \min \lbrace q \tilde{\sigma}^2_{t-1}(x_i), 1 \rbrace$ for a suitably chosen parameter $q$ (which trades off between the quality and the size of the embedding). Then, compute the (approximate) finite-dimensional feature embedding  $\tilde{\phi}_t(x)=\left(K_{\cD_t}^{1/2}\right)^{\dagger}k_{\cD_t}(x)$, where $K_{\cD_t}=[k(u,v)]_{u,v \in \cD_t}$, $k_{\cD_t}(x)=[k(x_{i_1},x),\ldots,k(x_{i_{m_t}},x)]^T$  and $A^{\dagger}$ denotes the pseudo inverse of any matrix $A$. We call the entire procedure Nystr\"{o}mEmbedding (pseudocode in appendix).
\vspace*{-3mm}         
\begin{algorithm}[H]
     \renewcommand\thealgorithm{2}
     \caption{Adaptively Truncated Approximate GP-UCB (ATA-GP-UCB)}\label{algo:itgp-ucb}
     \begin{algorithmic}
     \STATE \textbf{Input:} Parameters $\lambda > 0$, $\lbrace b_t \rbrace_{t \ge 1}$, $\lbrace \beta_t \rbrace_{t \ge 1}, q$, a kernel approximation (QFF or Nystr\"{o}m)
     \STATE \textbf{Set:} $\tilde{\mu}_0(x)=0$ and $\tilde{\sigma}_0^2(x)=k(x,x)$ for all $x \in \cX$
     \FOR{$t = 1, 2, 3 \ldots$}
     \STATE Play $x_{t} = \argmax_{x\in \cX} \tilde{\mu}_{t-1}(x) + \beta_t\tilde{\sigma}_{t-1}(x)$ and observe payoff $y_{t}$
     \STATE Set $\tilde{\phi}_{t}(x)=\begin{cases}
               \tilde{\phi}(x)& \text{\textit{if QFF approximation}}\\
              \text{Nystr\"{o}mEmbedding} \big(\{ (x_i , \tilde{\sigma}_{t-1}(x_i))\}_{i=1}^{t},q \big) & \text{\textit{if Nystr\"{o}m approximation}}
            \end{cases}$
     \STATE Set $\tilde{\Phi}_t ^T = [\tilde{\phi}_t(x_1),\ldots,\tilde{\phi}_t(x_t)]$ and $\tilde{V}_t = \tilde{\Phi}_t ^T \tilde{\Phi}_t + \lambda I_{m_t}$, where $m_t$ is the dimension of $\tilde{\phi}_t$
     \STATE Find the rows $u_1^T,\ldots,u_{m_t}^T$ of $\tilde{V}_t^{-1/2}\tilde{\Phi}_t^T$ and set $\hat{r}_{i}=\sum_{\tau =1}^{t}u_{i,\tau}y_\tau \mathds{1}_{\abs{u_{i,\tau}y_\tau} \le b_t}$ for all $i \in [m_t]$
     \STATE Set $\tilde{\theta}_t =  \tilde{V}_t^{-1/2}[\hat{r}_{1},\ldots,\hat{r}_{m_t}]^T$ and compute $\tilde{\mu}_t(x)=\tilde{\phi}_t(x)^T \tilde{\theta}_t$
     \STATE Set 
     $ \tilde{\sigma}_t^2(x)=\begin{cases}
         (i)\;\lambda \tilde{\phi_t}(x)^T \tilde{V}_t^{-1}\tilde{\phi_t}(x) & \text{\textit{if QFF approximation}}\\
         (ii)\; k(x,x)-\tilde{\phi_t}(x)^T\tilde{\phi_t}(x)+\lambda \tilde{\phi_t}(x)^T \tilde{V}_t^{-1}\tilde{\phi_t}(x) & \text{\textit{if Nystr\"{o}m approximation}}
       \end{cases}$
      \ENDFOR
     \end{algorithmic}
     \addtocounter{algorithm}{-2}
     \vspace*{-1mm}
     \end{algorithm}
\vspace*{-5mm}
\textit{\textbf{Remark 4.}} It is well known ($\lambda$-ridge leverage score sampling \cite{alaoui2015fast}) that, by sampling points proportional to their posterior variances $\sigma_{t}^2(x)$, one can obtain an accurate embedding $\tilde{\phi}_t(x)$, which in turn gives an accurate approximation $\tilde{\sigma}_t^2(x)$. But, computation of $\sigma_t^2(x)$ in turn requires inverting $K_t$, which takes at most $O(t^3)$ time. So, we make use of the already computed approximations $\tilde{\sigma}_{t-1}^2(x)$ to sample points at round $t$, without significantly compromising on the accuracy of the embeddings \cite{calandriello2019gaussian}. 

\textit{\textbf{Remark 5.}} The choice $(i)$ of $\tilde{\sigma}_t^2(x)$ in Algorithm \ref{algo:itgp-ucb} ensures accurate estimation of the variance of $x$ under the QFF approximation \cite{mutny2018efficient}. But, the same choice leads to severe underestimation of the variance under the Nystr\"{o}m approximation, specially when $x$ is far away from $\cD_t$. The choice $(ii)$ of $\tilde{\sigma}_t^2(x)$ in Algorithm \ref{algo:itgp-ucb} is known as \textit{deterministic training conditional} in the GP literature \cite{quinonero2007approximation} and provably prevents the phenomenon of variance starvation under Nystr\"{o}m approximation \cite{calandriello2019gaussian}.

\textbf{Cumulative regret of ATA-GP-UCB with QFF embeddings:} The following lemma shows that the data adaptive truncation of all the historical rewards and a good approximation of the kernel help us obtain a tighter confidence interval than TGP-UCB.
\vspace*{-1mm}
\begin{mylemma}[Tighter confidence sets with QFF truncation]
For any $\delta \in (0,1]$, ATA-GP-UCB with QFF approximation and parameters $b_t = \left(v/\ln(2mT/\delta)\right)^{\frac{1}{1+\alpha}}t^{\frac{1-\alpha}{2(1+\alpha)}}$ and $\beta_{t+1}=B+4 \sqrt{m/\lambda}\; v^{\frac{1}{1+\alpha}}\left(\ln(2mT/\delta)\right)^{\frac{\alpha}{1+\alpha}}t^{\frac{1-\alpha}{2(1+\alpha)}}$, ensures that with probability at least $1-\delta$, uniformly over all $t \in [T]$ and $x \in \cX$,
\vspace*{-2mm}
\beq
\abs{f(x)-\tilde{\mu}_{t-1}(x)} \le \beta_t\tilde{\sigma}_{t-1}(x) + O(B\epsilon_m^{1/2} t^2),
\label{eqn:func-conc-ata-gp-ucb}
\eeq
where the QFF dimension $m$ is such that 
$\sup_{x,y \in \cX} \abs{k(x,y)-\tilde{\phi}(x)^T\tilde{\phi}(y)} =: \epsilon_m < 1$.
\label{lem:func-conc-ata-gp-ucb}
\end{mylemma}
\vspace*{-2mm}
Here, the scaling $ t^{\frac{1-\alpha}{2(1+\alpha)}}$ of the confidence width $\beta_t$ is much less than the scaling $t^{\frac{1}{2(1+\alpha)}}$ of TGP-UCB, which eventually leads to a tighter confidence interval. However, in order to achive sublinear cumulative regret, we need to ensure that the approximation error $\epsilon_m$ decays at least as fast as $O(1/T^6)$ and feature dimension $m$ grows no faster than $\text{polylog}(T)$. This will ensure that the regret accumulated due to the second term in the RHS of \ref{eqn:func-conc-ata-gp-ucb} is $O(1)$, as well as the contribution from the first term is $\tilde{O}(T^{\frac{1}{1+\alpha}})$, since sum of the approximate posterior standard deviations grows only as $\tilde{O}(\sqrt{mT})$. 
Now, the QFF embedding (\ref{eqn:qff-embedding}) of $k_{\text{SE}}$ can be shown to achieve $\epsilon_m \le d2^{d-1}\frac{1}{\sqrt{2}\bar{m}^{\bar{m}}}\left(\frac{e}{4l^2}\right)^{\bar{m}}=O\left(\frac{d2^{d-1}}{(\bar{m}l^2)^{\bar{m}}}\right)$, where $m=\bar{m}^d$ \cite{mutny2018efficient}. The decay is exponential when $\bar{m} > 1/l^2$ and $d=O(1)$\footnote{For most BO applications, the effective dimensionality of the problem is low, e.g., additive models \cite{kandasamy2015high,rolland2018high}.}. Now, for $\bar{m} \ge 2\log_{4/e}(T^3)$, we have $\epsilon_m^{1/2}T^3=O(1)$ and $m = O((\ln T)^d)$, which gives a sublinear regret bound \footnote{ Under RFF approximation $\epsilon_m=\tilde{O}(\sqrt{1/m})$  \cite{sriperumbudur2015optimal}. Hence, ATA-GP-UCB does not achieve sublinear regret.}. The following theorem states this formally, with a full proof deferred to Appendix D.2.

\vspace*{-1mm}
\begin{mytheorem}[Regret bound for ATA-GP-UCB with QFF embedding]
Fix any $\delta \in (0,1]$. Then, under the same hypothesis of Theorem \ref{thm:regret-bound-tgp-ucb}, for $\cX=[0,1]^d$ and $k=k_{\text{SE}}$, ATA-GP-UCB under QFF approximation, with parameters $b_t$ and $\beta_t$ set as in Lemma \ref{lem:func-conc-ata-gp-ucb}, and with the embedding $\tilde{\phi}$ from \ref{eqn:qff-embedding} such that $\bar{m} > 1/l^2$ and $\bar{m} \ge 2\log_{4/e}(T^3)$,  enjoys, with probability at least $1-\delta$, the regret bound
\vspace*{-1mm}
\beqn
R_T = O\left(B\sqrt{T(\ln T)^{d+1}}+v^{\frac{1}{1+\alpha}}\left(\ln\left(\frac{T(\ln T)^d}{\delta}\right)\right)^{\frac{\alpha}{1+\alpha}}\sqrt{\ln T}\left(\ln T\right)^{d}T^{\frac{1}{1+\alpha}}\right).
\eeqn
\label{thm:regret-bound-qff}
\end{mytheorem}
\vspace*{-5mm}
\textit{\textbf{Remark 6.}} When the variance of the rewards is finite (i.e., $\alpha = 1$), the cumulative regret for ATA-GP-UCB under QFF approximation of the SE kernel is
$O((\ln T)^{d+1}\sqrt{T})$, which now recovers the state-of-the-art regret bound of GP-UCB under sub-Gaussian rewards \cite[Corollary 2]{mutny2018efficient} unlike the earlier TGP-UCB. It is worth pointing out that the bound in Theorem \ref{thm:regret-bound-qff} is only for the SE kernel defined on $\cX = [0,1]^d$, and designing a no-regret BO strategy under the QFF approximation of any other stationary kernel still remains a open question even when the rewards are sub-Gaussian \cite{mutny2018efficient}.


\textbf{Cumulative regret of ATA-GP-UCB with Nystr\"{o}m embeddings:} Now, we will show that ATA-GP-UCB under Nystr\"{o}m approximation achives optimal regret for any stationary kernel defined on $\cX \subset \Real^d$ without any restriction on $d$. Similar to Lemma \ref{lem:func-conc-ata-gp-ucb}, ATA-GP-UCB under Nystr\"{o}m approximation also maintains tighter confidence sets than TGP-UCB.
As before, the confidence sets are useful only if the dimension of the embeddings $m_t$ grows no faster than $\text{polylog}(t)$. Not only that, we also need to ensure that the approximate posterior variances are only a constant factor away from the exact ones. Then, since sum of the posterior standard deviations grows only as $O(\sqrt{T\gamma_T})$, we can achieve the optimal $\tilde{O}(T^\frac{1}{1+\alpha})$ regret scaling.
Now for any $\epsilon \in (0,1)$, setting $q=6\frac{1+\epsilon}{1-\epsilon}\ln(2T/\delta)/\epsilon^2$, the Nystr\"{o}m embeddings $\tilde{\phi}_t$ can be shown to achieve $m_t \le 6 \frac{1+\epsilon}{1-\epsilon}\left(1+\frac{1}{\lambda}\right) \; q \gamma_t$ and $\frac{1-\epsilon}{1+\epsilon} \sigma_t^2(x) \le \tilde{\sigma}_t^2(x) \le		 \frac{1+\epsilon}{1-\epsilon}\sigma_t^2(x)$ with probability at least $1-\delta$ \cite{calandriello2019gaussian}, which helps us to achieve an optimal regret bound. The following theorem states this formally, with a full proof deferred to Appendix D.3.
\vspace*{-1mm}
\begin{mytheorem}[Regret bound for ATA-GP-UCB with Nystr\"{o}m embedding]
Fix any $\delta \in (0,1]$, $\epsilon \in (0,1)$ and set $\rho=\frac{1+\epsilon}{1-\epsilon}$. Then, under the same hypothesis of Theorem \ref{thm:regret-bound-tgp-ucb}, ATA-GP-UCB under Nystr\"{o}m approximation, and with parameters $q=6\rho\ln(4T/\delta)/\epsilon^2$, $b_t = \left(v/\ln(4m_tT/\delta)\right)^{\frac{1}{1+\alpha}}t^{\frac{1-\alpha}{2(1+\alpha)}}$ and $\beta_{t+1}=B(1+\frac{1}{\sqrt{1-\epsilon}})+4 \sqrt{m_t/\lambda}\; v^{\frac{1}{1+\alpha}}\left(\ln(4m_tT/\delta)\right)^{\frac{\alpha}{1+\alpha}}t^{\frac{1-\alpha}{2(1+\alpha)}}$, enjoys, with probability at least $1-\delta$, the regret bound
\vspace*{-1mm}
\beqn
R_T = O\left(\rho B\left(1+\frac{1}{\sqrt{1-\epsilon}}\right)\sqrt{T \gamma_T} + \frac{\rho^2}{\epsilon} \; v^{\frac{1}{1+\alpha}}\left(\ln\left(\frac{\gamma_T \ln(T/\delta)T}{\delta}\right)\right)^{\frac{\alpha}{1+\alpha}}\sqrt{\ln(T/\delta)}\gamma_TT^{\frac{1}{1+\alpha}}\right).
\eeqn
\label{thm:regret-bound-nystrom}
\end{mytheorem}
\vspace*{-4mm}
\textit{\textbf{Remark 7.}} Theorem \ref{thm:regret-bound-qff} and \ref{thm:regret-bound-nystrom} imply that ATA-GP-UCB achieves $\tilde{O}\big(v^{\frac{1}{1+\alpha}}(\ln T)^{d}T^{\frac{1}{1+\alpha}}\big)$ regret bound for $k_{\text{SE}}$, which matches the lower bound (Theorem \ref{thm:lower-bound}) upto a factor of $\frac{\alpha}{1+\alpha}$ in the exponent of $\ln T$, as well as a few extra $\ln T$ factors hidden in the notation $\tilde{O}$. For the Mat\'{e}rn kernel, the bound is  $\tilde{O}\big(T^{\frac{1}{1+\alpha}\frac{2\nu +(2+\alpha)d(d+1)}{2\nu + d(d+1)}}\big)$, which is sublinear only when $\frac{d(d+1)}{2\nu + d(d+1)} < \frac{\alpha}{1+\alpha}$, and the gap from the lower bound is more significant in this case. It is worth mentioning that a similar gap is present even for the (easier) setting of sub-Gaussian rewards \cite{scarlett2017lower} and there might exist better algorithms which can bridge this gap. When the variance of the rewards is finite (i.e., $\alpha = 1$), the cumulative regret for ATA-GP-UCB under Nystr\"{o}m approximation is
$\tilde{O}(\gamma_T \sqrt{T})$, which recovers the state-of-the-art regret bound under sub-Gaussian rewards \cite[Thm. 2]{calandriello2019gaussian}. For the linear bandit setting, i.e. when the feature map $\tilde{\phi}_t(x)=x$ itself, substituting $\gamma_T=O(d \ln T)$, we find that the regret upper bound in Theorem \ref{thm:regret-bound-nystrom} recovers the (optimal) regret bound of \cite[Thm. 3]{shao2018almost} up to a logarithmic factor.

\textbf{Computational complexity of ATA-GP-UCB:}
(a) Under the (data-dependent) Nystr\"{o}m approximation, constructing the dictionary $D_t$ takes $O(t)$ time at each step $t$. Then, we compute the embeddings $\tilde{\phi}_t(x)$ for all arms in $O(m_t^3+m_t^2\abs{\cX})$ time, where $\abs{\cX}$ is the cardinality of $\cX$. Now, construction of $\tilde{V}_t$ takes $O(m_t^2t)$ time, since we need to rebuild it from the scratch. Then, $\tilde{V}_t^{-1/2}$ is computed in $O(m_t^3)$ time. We can now compute $\tilde{\mu}_t(x)$ and $\tilde{\sigma}_t^2(x)$ for all arms in $O(m_t^2t+m_t \abs{\cX})$ and $O(m_t^2 \abs{\cX})$ time, respectively, using already computed $\tilde{\phi}_t(x)$ and $\tilde{V}_t^{-1/2}$. Thus per-step time complexity is $O\left(m_t^2 (t+ \abs{\cX})\right)$, since $m_t \le t$. As we only need to store $\tilde{V}_t^{-1/2}$ and $\tilde{\phi}_t(x) $ for all $x$, per-step space complexity is $O(m_t(m_t+\abs{\cX}))$. Since $m_t=\tilde{O}(\gamma_t)$, the total time and space requirements are $\tilde{O}(T^2+T\abs{\cX})$ and $\tilde{O}(T\abs{\cX})$, respectively, whenever $\gamma_T \ll T$. (b) Under (data-independent) QFF approximation, the per-step complexities are $O(m^3+m^2(t+\abs{\cX})$ and $O(m(m+\abs{\cX}))$, respectively. But, since $m=O((\ln T)^d)$, the total time and space required are also of the same order as in (a).

\vspace{-3mm}
\section{Experiments}
\label{sec:experiments} 
\vspace{-3mm}
We numerically compare the performance of TGP-UCB (Algorithm \ref{algo:tgp-ucb}), ATA-GP-UCB with QFF (ATA-GP-UCB-QFF) and Nystr\"{o}m (ATA-GP-UCB-Nystr\"{o}m) approximations (Algorithm \ref{algo:itgp-ucb}) on both synthetic and real-world heavy-tailed environments. The confidence width $\beta_t$ and truncation level $b_t$ of our algorithms, and the trade-off parameter $q$ used in Nystr\"{o}m approximation are set order-wise similar to those recommended by theory (Theorems \ref{thm:regret-bound-tgp-ucb}, \ref{thm:regret-bound-qff} and \ref{thm:regret-bound-nystrom}). We use $\lambda=1$ in all algorithms and $\epsilon=0.1$ in ATA-GP-UCB-Nystr\"{o}m. We plot the mean and standard deviation (under independent trials) of the time-average regret $R_T/T$ in Figure \ref{fig:plot}. We use the following datasets.
\begin{figure}[t!]
\vskip -5mm
\centering
\subfigure[$k_{\text{SE}},f \in $ RKHS, Student's-$t$]{\includegraphics[height=1.15in,width=1.7in]{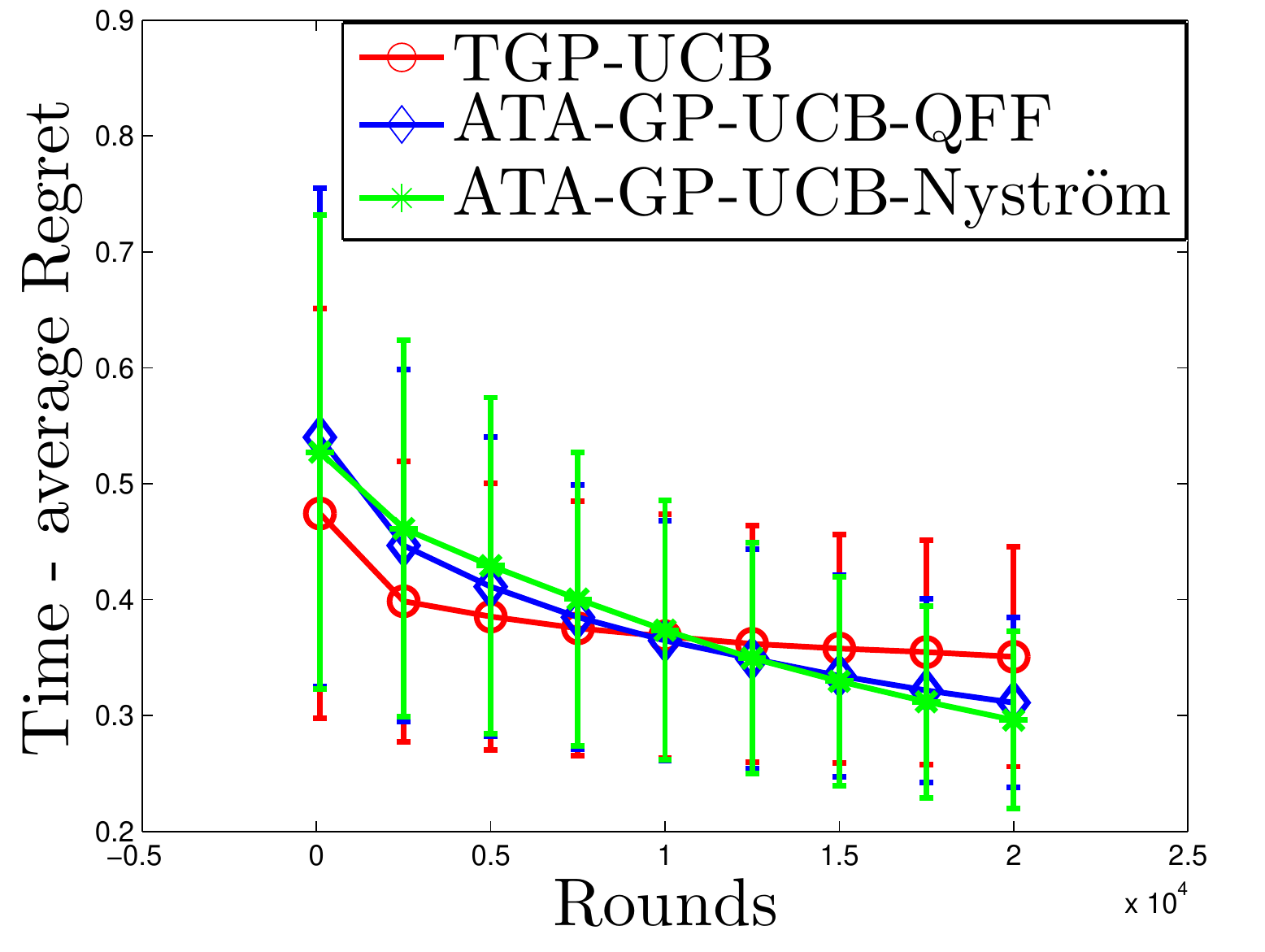}}
\subfigure[$k_{\text{SE}},f \in $ RKHS, Pareto]{\includegraphics[height=1.15in,width=1.7in]{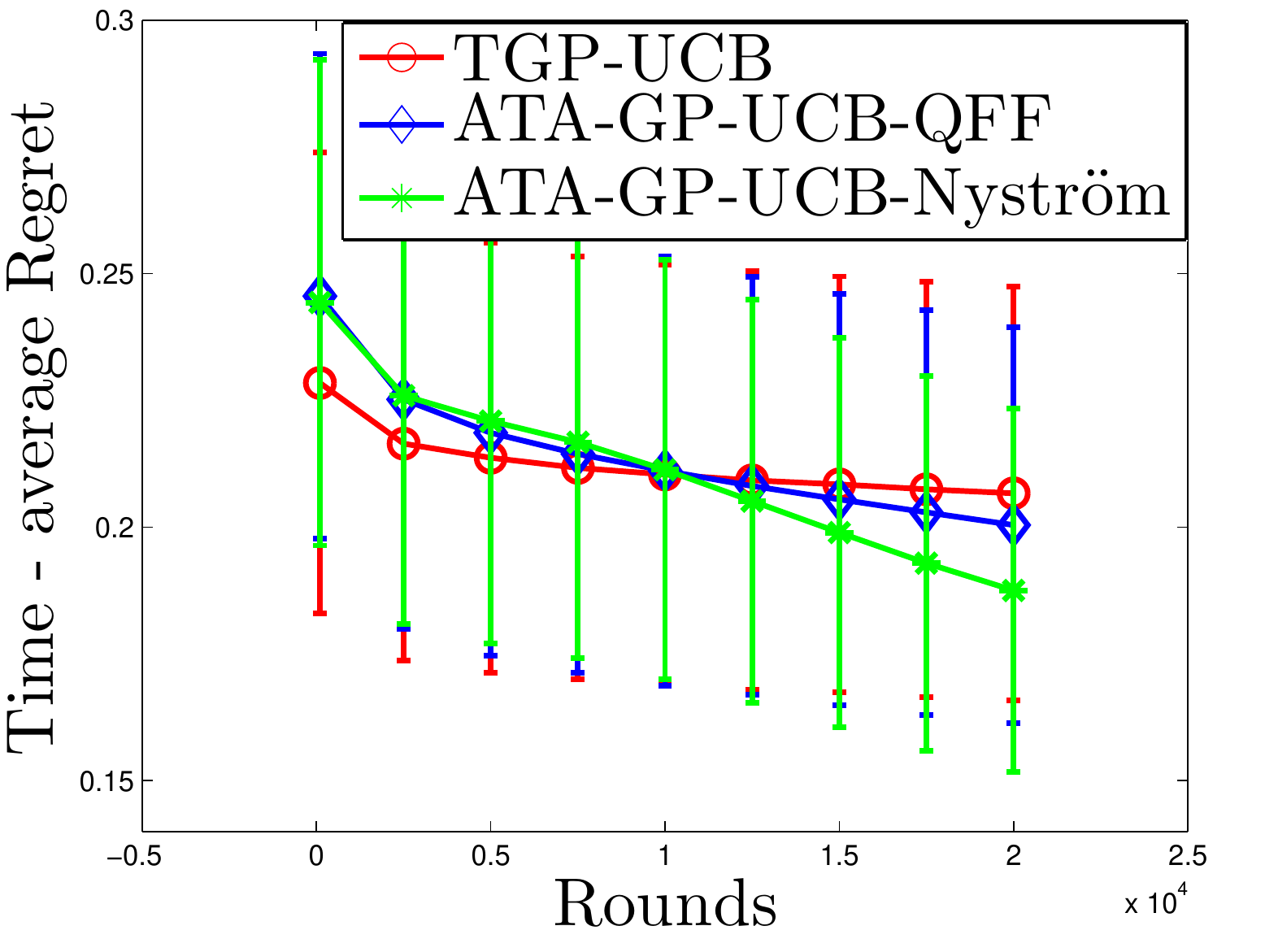}}
\subfigure[$k_{\text{Mat\'{e}rn}},f \in $ RKHS, Student's-$t$]{\includegraphics[height=1.15in,width=1.7in]{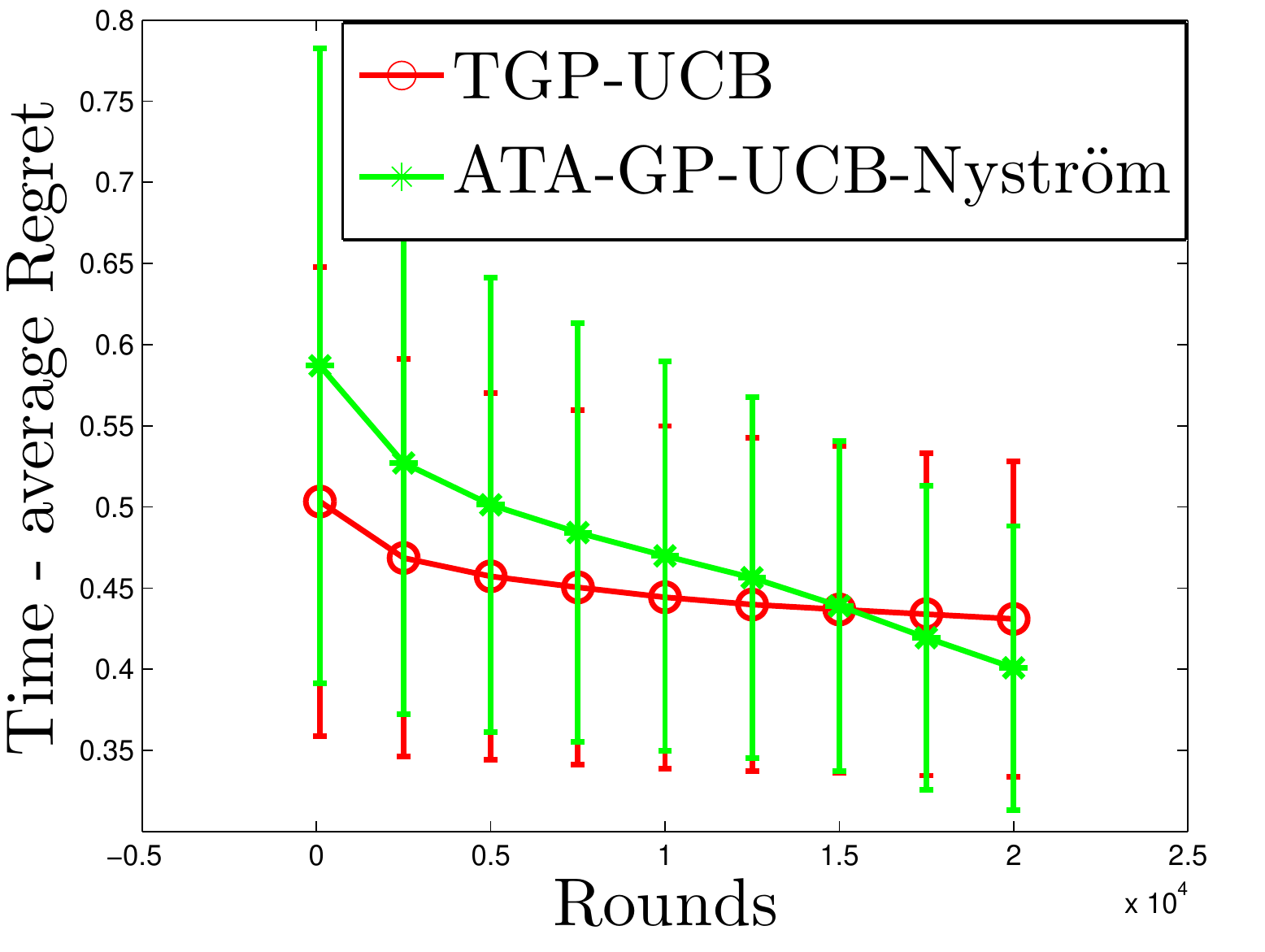}}
\vskip -3mm
\subfigure[Stock market data]{\includegraphics[height=1.15in,width=1.7in]{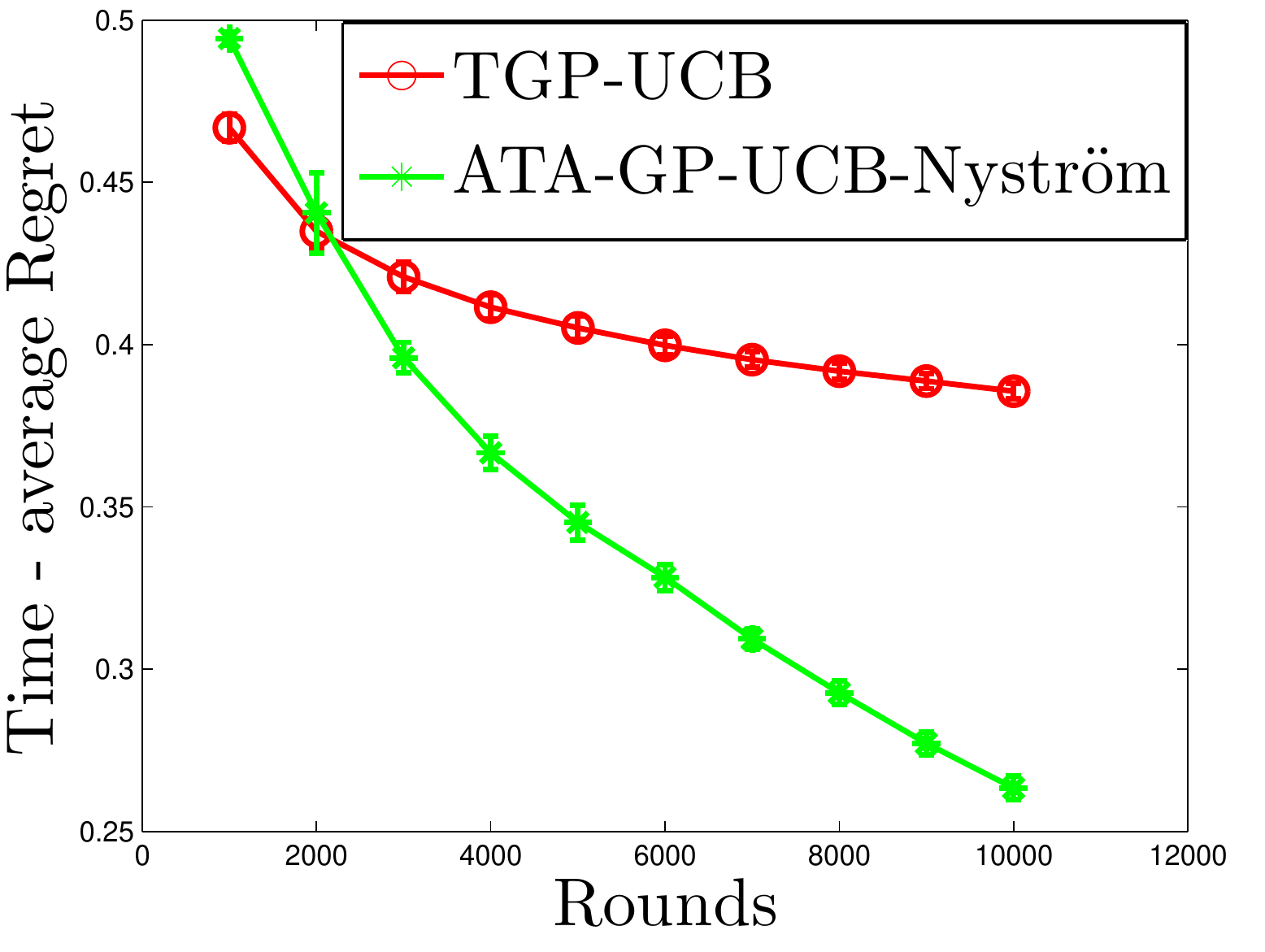}}
\subfigure[Light sensor data]{\includegraphics[height=1.15in,width=1.7in]{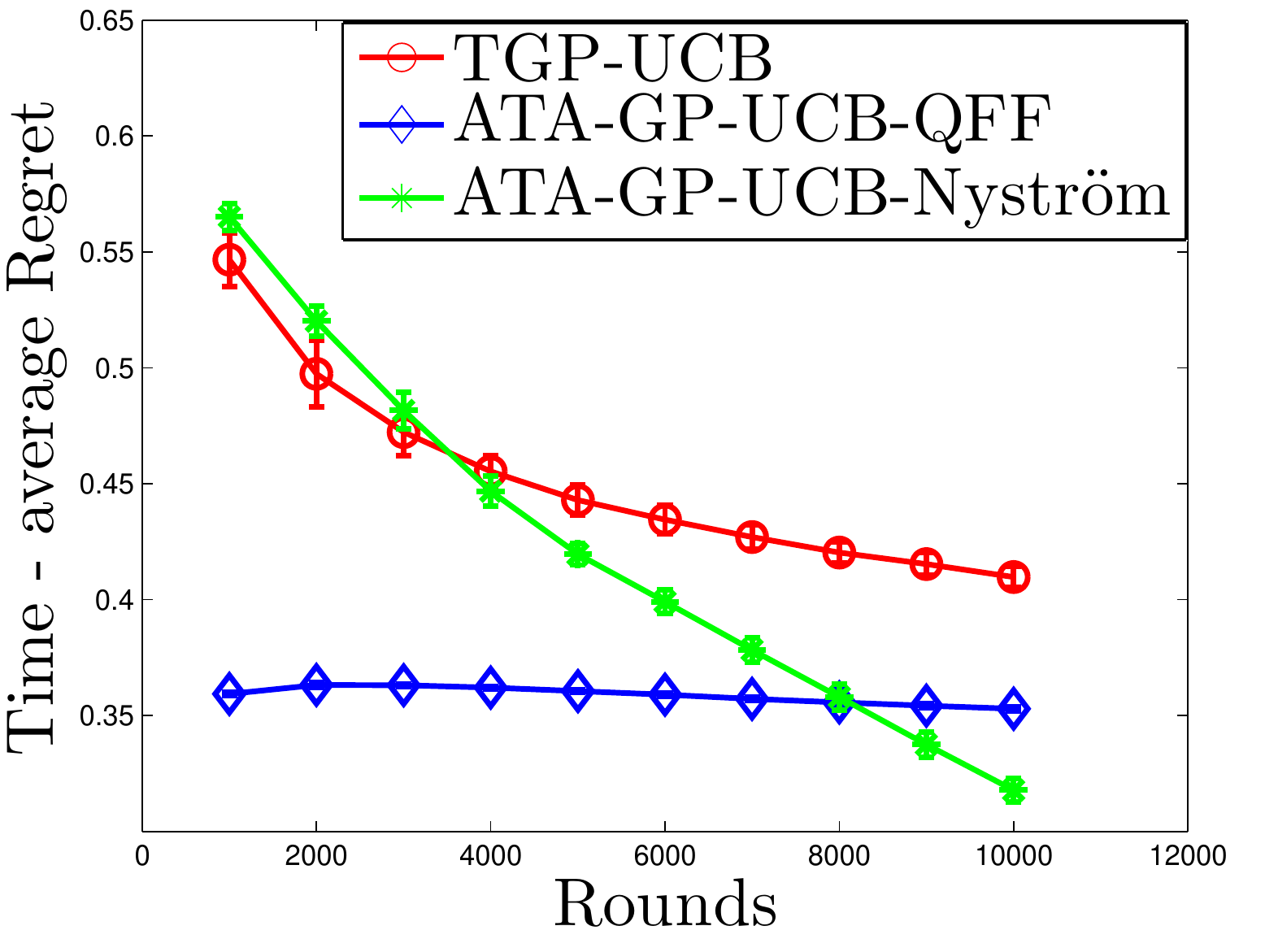}}
\subfigure[Effect of  truncation on GP-UCB]{\includegraphics[height=1.15in,width=1.75in]{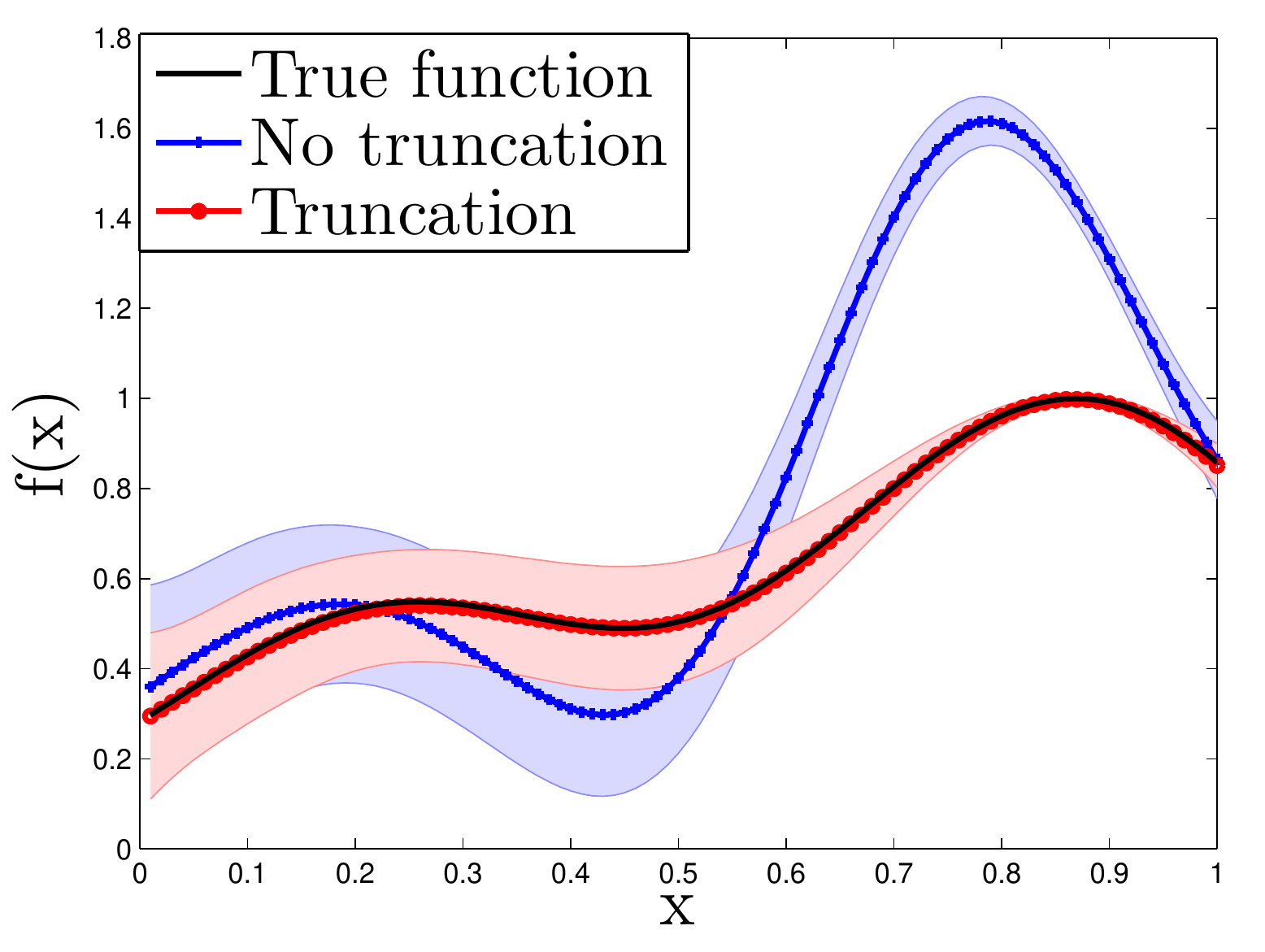}}
\vskip -3mm
\caption{{\footnotesize (a)-(e) Time-average regret ($R_T/T$) for TGP-UCB, ATA-GP-UCB with QFF approximation (ATA-GP-UCB-QFF) and Nystr\"{o}m approximation (ATA-GP-UCB-Nystr\"{o}m) on heavy-tailed data. (f) Confidence sets ($\mu_t\pm \sigma_t$) formed by GP-UCB with and without truncation under heavy fluctuations.
}}
\label{fig:plot} 
\vskip -5mm
\end{figure}

\vspace*{-1.5mm}
\textbf{1. Synthetic data:} We generate the objective
function $f \in \cH_k(\cX)$ with $\cX$ set to be a discretization of $[0,1]$ into $100$ evenly spaced points. Each $f=\sum_{i=1}^{p}a_ik(\cdot,x_i)$
was generated using an SE kernel with $l=0.2$ and by uniformly sampling $a_i \in [-1,1]$ and support points $x_i \in \cX$ with $p=100$. We set $B=\max_{x \in \cX} \abs{f(x)}$. To generate the rewards, first we consider $y(x)=f(x)+\eta$, where the noise $\eta$ are samples from the Student's $t$-distribution with $3$ degrees of freedom (Figure \ref{fig:plot} a). Here, the variance is bounded ($\alpha=1$) and hence $v=B^2+3$. Next, we generate the rewards as samples from the Pareto distribution with shape parameter $2$ and scale parameter $f(x)/2$. $f$ is generated similarly, except that here we sample $a_i$'s uniformly from $[0,1]$. Then, we set $B$ as before leading to the bound of $(1+\alpha)$-th raw moments $v=\frac{B^{1+\alpha}}{2^{\alpha}(1-\alpha)}$. We plot the results for $\alpha=0.9$ (Figure \ref{fig:plot} b). We use $m=32$ features (in consistence with Theorem \ref{thm:regret-bound-qff}) for ATA-GP-UCB-QFF in these experiments. Next, we generate $f$ using the Mat\'{e}rn kernel with $l=0.2$ and $\nu=2.5$, and consider the same Student's-$t$ distribution as earlier to generate rewards. As we do not have the theory of ATA-GP-UCB-QFF for the Mat\'{e}rn kernel yet, we exclude evaluating it here (Figure \ref{fig:plot} c). We perform $20$ trials for $2 \times 10^4$ rounds and for each trial we evaluate on a different $f$ (which explains the high error bars). 

\vspace*{-1.5mm}
\textbf{2. Stock market data:} We
consider a representative application of identifying the most profitable stock in a given pool of stocks. This is motivated by the practical scenario that
an investor would like to invest a fixed budget of money in a stock and get as much return as possible. We took the adjusted closing price of $29$ stocks from January 4th, 2016 to April 10th, 2019 ({\small \url{https://www.quandl.com/data/EOD-End-of-Day-US-Stock-Prices}}).  We conduct Kolmogrov-Smirnov (KS) test to find out that the null hypothesis
of stock prices following a Gaussian distribution is rejected against the favor of a heavy-tailed distribution. We take the empirical mean of stock prices as our objective function $f$ and empirical covariance of the normalized stock prices as our kernel function $k$ (since stock behaviors are mostly correlated with one another). We consider $\alpha=1$ and set $v$ as the empirical average of the squared prices. Since the kernel is data dependent, we cannot run ATA-GP-UCB-QFF here. We average over $10$ independent trials of the algorithms (Figure \ref{fig:plot} d).

\vspace*{-1.5mm}
\textbf{3. Light sensor data:} We take light sensor data collected in the CMU Intelligent Workplace in Nov 2005 containing locations of $41$ sensors, $601$ train samples and $192$ test samples ({\small\url{http://www.cs.cmu.edu/~guestrin/Class/10708-F08/projects}}) in the context of learning the maximum average reading of the sensors. For each sensor, we find that the KS test on its readings rejects the Gaussian against the favor of a heavy-tailed distribution. We take the empirical average of the test samples as our objective $f$ and empirical covariance of the normalized train samples as our kernel $k$. We consider $\alpha=1$, set $v$ as the empirical mean of the squared readings and $B$ as the maximum of the average readings. For ATA-GP-UCB-QFF, we fit a SE kernel with $l^2=0.1$ on the given sensor locations and approximate it with $m=16^2=256$ features (Figure \ref{fig:plot} e).


\vspace*{-1.5mm}
\textbf{Observations:} We find that ATA-GP-UCB outperforms TGP-UCB uniformly over all experiments, which is consistent with our theoretical results. We also see that the performance of ATA-GP-UCB under the Nystr\"{o}m approximation is no worse than that under the QFF approximation. Not only that, the scope of the latter is limited due to its dependence on the analytical form of the kernel, whereas the former is data-adaptive and hence, well suited for practical purposes.

\vspace*{-1.5mm}
\textbf{Effect of truncation:} For heavy-tailed rewards, the sub-Gaussian constant $R=\infty$. Hence, we exclude evaluating GP-UCB in the above experiments. Now, we demonstrate the effect of truncation on GP-UCB in the following experiment. First, we generate a function $f \in \cH_k(\cX)$ and normalize it between $[0,1]$. Then, we simulate rewards as $y(x)=f(x)+\eta$, where $\eta$ takes values in $\lbrace -10,10\rbrace$, uniformly, for any single random point in $\cX$, and is zero everywhere else. We run GP-UCB with $\beta_t=\ln t$ and see that the posterior mean after $T=10^4$ rounds is not a good estimate of $f$. However, by truncating reward samples which exceeds $t^{1/4}$ (truncation threshold in TGP-UCB when $\alpha=1$) at round $t$, we get an (almost) accurate estimator of $f$. Not only that, the confidence interval around this estimator contains $f$ at every point in $\cX$, which in turn ensures good performance. We plot the respective confidence sets averaged over $50$ such randomizations of noise (Figure \ref{fig:plot} f).

\vspace*{-1.5mm}
\textbf{Conclusion:} To the best of our knowledge, this is the first work to formulate and solve BO optimally under heavy-tailed observations. We have demonstrated the failure of existing methods and developed optimal algorithms using kernel approximation techniques, which are easy to implement and perform well in practice, with rigorous theoretical guarantees. One can also consider building and studying a median of means-style estimator \cite{bubeck2013bandits} in the feature space and hope to develop an optimal algorithm.

\bibliographystyle{plainnat}
\bibliography{main}

\begin{thebibliography}{42}
\providecommand{\natexlab}[1]{#1}
\providecommand{\url}[1]{\texttt{#1}}
\expandafter\ifx\csname urlstyle\endcsname\relax
  \providecommand{\doi}[1]{doi: #1}\else
  \providecommand{\doi}{doi: \begingroup \urlstyle{rm}\Url}\fi

\bibitem[Alaoui and Mahoney(2015)]{alaoui2015fast}
Ahmed Alaoui and Michael~W Mahoney.
\newblock Fast randomized kernel ridge regression with statistical guarantees.
\newblock In \emph{Advances in Neural Information Processing Systems}, pages
  775--783, 2015.

\bibitem[Azar et~al.(2014)Azar, Lazaric, and Brunskill]{azar2014online}
Mohammad~Gheshlaghi Azar, Alessandro Lazaric, and Emma Brunskill.
\newblock Online stochastic optimization under correlated bandit feedback.
\newblock In \emph{ICML}, pages 1557--1565, 2014.

\bibitem[Bergstra and Bengio(2012)]{BerBen12:randomHypParamTun}
James Bergstra and Yoshua Bengio.
\newblock Random search for hyper-parameter optimization.
\newblock \emph{J. Mach. Learn. Res.}, 13:\penalty0 281--305, February 2012.

\bibitem[Bochner(1959)]{bochner1959lectures}
Salomon Bochner.
\newblock \emph{Lectures on Fourier integrals}.
\newblock Princeton University Press, 1959.

\bibitem[Bogunovic et~al.(2018)Bogunovic, Scarlett, Jegelka, and
  Cevher]{bogunovic2018adversarially}
Ilija Bogunovic, Jonathan Scarlett, Stefanie Jegelka, and Volkan Cevher.
\newblock Adversarially robust optimization with gaussian processes.
\newblock In \emph{Advances in Neural Information Processing Systems}, pages
  5760--5770, 2018.

\bibitem[Bubeck et~al.(2011)Bubeck, Munos, Stoltz, and
  Szepesv{\'a}ri]{bubeck2011x}
S{\'e}bastien Bubeck, R{\'e}mi Munos, Gilles Stoltz, and Csaba Szepesv{\'a}ri.
\newblock X-armed bandits.
\newblock \emph{Journal of Machine Learning Research}, 12\penalty0
  (May):\penalty0 1655--1695, 2011.

\bibitem[Bubeck et~al.(2013)Bubeck, Cesa-Bianchi, and
  Lugosi]{bubeck2013bandits}
S{\'e}bastien Bubeck, Nicolo Cesa-Bianchi, and G{\'a}bor Lugosi.
\newblock Bandits with heavy tail.
\newblock \emph{IEEE Transactions on Information Theory}, 59\penalty0
  (11):\penalty0 7711--7717, 2013.

\bibitem[Calandriello et~al.(2019)Calandriello, Carratino, Lazaric, Valko, and
  Rosasco]{calandriello2019gaussian}
Daniele Calandriello, Luigi Carratino, Alessandro Lazaric, Michal Valko, and
  Lorenzo Rosasco.
\newblock Gaussian process optimization with adaptive sketching: Scalable and
  no regret.
\newblock \emph{In Conference on Learning Theory}, 2019.

\bibitem[Carpentier and Valko(2014)]{carpentier2014extreme}
Alexandra Carpentier and Michal Valko.
\newblock Extreme bandits.
\newblock In \emph{Advances in Neural Information Processing Systems}, pages
  1089--1097, 2014.

\bibitem[Chowdhury and Gopalan(2017)]{chowdhury2017kernelized}
Sayak~Ray Chowdhury and Aditya Gopalan.
\newblock On kernelized multi-armed bandits.
\newblock In \emph{Proceedings of the 34th International Conference on Machine
  Learning-Volume 70}, pages 844--853. JMLR. org, 2017.

\bibitem[Drineas and Mahoney(2005)]{drineas2005nystrom}
Petros Drineas and Michael~W Mahoney.
\newblock On the nystr{\"o}m method for approximating a gram matrix for
  improved kernel-based learning.
\newblock \emph{journal of machine learning research}, 6\penalty0
  (Dec):\penalty0 2153--2175, 2005.

\bibitem[Durand et~al.(2018)Durand, Maillard, and Pineau]{durand2018streaming}
Audrey Durand, Odalric-Ambrym Maillard, and Joelle Pineau.
\newblock Streaming kernel regression with provably adaptive mean, variance,
  and regularization.
\newblock \emph{The Journal of Machine Learning Research}, 19\penalty0
  (1):\penalty0 650--683, 2018.

\bibitem[Garnett et~al.(2010)Garnett, Osborne, and
  Roberts]{GarOsbRob10:BOsensor}
R.~Garnett, M.~A. Osborne, and S.~J. Roberts.
\newblock Bayesian optimization for sensor set selection.
\newblock In \emph{Proceedings of the 9th ACM/IEEE International Conference on
  Information Processing in Sensor Networks}, IPSN '10, pages 209--219, New
  York, NY, USA, 2010. ACM.

\bibitem[Gonzalez et~al.(2015)Gonzalez, Longworth, James, and
  Lawrence]{gonzalez2015bayesian}
Javier Gonzalez, Joseph Longworth, David~C James, and Neil~D Lawrence.
\newblock Bayesian optimization for synthetic gene design.
\newblock \emph{arXiv preprint arXiv:1505.01627}, 2015.

\bibitem[Hazan et~al.(2007)Hazan, Agarwal, and Kale]{hazan2007logarithmic}
Elad Hazan, Amit Agarwal, and Satyen Kale.
\newblock Logarithmic regret algorithms for online convex optimization.
\newblock \emph{Machine Learning}, 69\penalty0 (2-3):\penalty0 169--192, 2007.

\bibitem[Hern{\'a}ndez-Lobato et~al.(2014)Hern{\'a}ndez-Lobato, Hoffman, and
  Ghahramani]{hernandez2014predictive}
Jos{\'e}~Miguel Hern{\'a}ndez-Lobato, Matthew~W Hoffman, and Zoubin Ghahramani.
\newblock Predictive entropy search for efficient global optimization of
  black-box functions.
\newblock In \emph{Advances in neural information processing systems}, pages
  918--926, 2014.

\bibitem[Hildebrand(1987)]{hildebrand1987introduction}
Francis~Begnaud Hildebrand.
\newblock \emph{Introduction to numerical analysis}.
\newblock Courier Corporation, 1987.

\bibitem[Hsu and Sabato(2014)]{hsu2014heavy}
Daniel Hsu and Sivan Sabato.
\newblock Heavy-tailed regression with a generalized median-of-means.
\newblock In \emph{International Conference on Machine Learning}, pages 37--45,
  2014.

\bibitem[Jagannathan et~al.(2014)Jagannathan, Markakis, Modiano, and
  Tsitsiklis]{JagMarModTsi14:heavytailsched}
Krishna~P. Jagannathan, Mihalis~G. Markakis, Eytan Modiano, and John~N.
  Tsitsiklis.
\newblock Throughput optimal scheduling over time-varying channels in the
  presence of heavy-tailed traffic.
\newblock \emph{{IEEE} Trans. Information Theory}, 60\penalty0 (5):\penalty0
  2896--2909, 2014.
\newblock \doi{10.1109/TIT.2014.2311125}.
\newblock URL \url{https://doi.org/10.1109/TIT.2014.2311125}.

\bibitem[Kandasamy et~al.(2015)Kandasamy, Schneider, and
  P{\'o}czos]{kandasamy2015high}
Kirthevasan Kandasamy, Jeff Schneider, and Barnab{\'a}s P{\'o}czos.
\newblock High dimensional bayesian optimisation and bandits via additive
  models.
\newblock In \emph{International Conference on Machine Learning}, pages
  295--304, 2015.

\bibitem[Kleinberg et~al.(2008)Kleinberg, Slivkins, and
  Upfal]{kleinberg2008multi}
Robert Kleinberg, Aleksandrs Slivkins, and Eli Upfal.
\newblock Multi-armed bandits in metric spaces.
\newblock In \emph{Proceedings of the fortieth annual ACM symposium on Theory
  of computing}, pages 681--690. ACM, 2008.

\bibitem[Lattimore(2017)]{lattimore2017scale}
Tor Lattimore.
\newblock A scale free algorithm for stochastic bandits with bounded kurtosis.
\newblock In \emph{Advances in Neural Information Processing Systems}, pages
  1584--1593, 2017.

\bibitem[Medina and Yang(2016)]{medina2016no}
Andres~Munoz Medina and Scott Yang.
\newblock No-regret algorithms for heavy-tailed linear bandits.
\newblock In \emph{International Conference on Machine Learning}, pages
  1642--1650, 2016.

\bibitem[Mo{\v{c}}kus(1975)]{movckus1975bayesian}
Jonas Mo{\v{c}}kus.
\newblock On bayesian methods for seeking the extremum.
\newblock In \emph{Optimization Techniques IFIP Technical Conference}, pages
  400--404. Springer, 1975.

\bibitem[Mutny and Krause(2018)]{mutny2018efficient}
Mojmir Mutny and Andreas Krause.
\newblock Efficient high dimensional bayesian optimization with additivity and
  quadrature fourier features.
\newblock In \emph{Advances in Neural Information Processing Systems}, pages
  9005--9016, 2018.

\bibitem[Quinonero-Candela et~al.(2007)Quinonero-Candela, Rasmussen, and
  Williams]{quinonero2007approximation}
Joaquin Quinonero-Candela, Carl~Edward Rasmussen, and Christopher~KI Williams.
\newblock Approximation methods for gaussian process regression.
\newblock \emph{Large-scale kernel machines}, pages 203--224, 2007.

\bibitem[Rahimi and Recht(2008)]{rahimi2008random}
Ali Rahimi and Benjamin Recht.
\newblock Random features for large-scale kernel machines.
\newblock In \emph{Advances in neural information processing systems}, pages
  1177--1184, 2008.

\bibitem[Resnick(2007)]{resnick2007heavy}
Sidney~I Resnick.
\newblock \emph{Heavy-tail phenomena: probabilistic and statistical modeling}.
\newblock Springer Science \& Business Media, 2007.

\bibitem[Rolland et~al.(2018)Rolland, Scarlett, Bogunovic, and
  Cevher]{rolland2018high}
Paul Rolland, Jonathan Scarlett, Ilija Bogunovic, and Volkan Cevher.
\newblock High-dimensional bayesian optimization via additive models with
  overlapping groups.
\newblock \emph{arXiv preprint arXiv:1802.07028}, 2018.

\bibitem[Scarlett et~al.(2017)Scarlett, Bogunovic, and
  Cevher]{scarlett2017lower}
Jonathan Scarlett, Ilija Bogunovic, and Volkan Cevher.
\newblock Lower bounds on regret for noisy gaussian process bandit
  optimization.
\newblock In \emph{Conference on Learning Theory}, pages 1723--1742, 2017.

\bibitem[Seldin et~al.(2012)Seldin, Laviolette, Cesa-Bianchi, Shawe-Taylor, and
  Auer]{seldin2012pac}
Yevgeny Seldin, Fran{\c{c}}ois Laviolette, Nicolo Cesa-Bianchi, John
  Shawe-Taylor, and Peter Auer.
\newblock Pac-bayesian inequalities for martingales.
\newblock \emph{IEEE Transactions on Information Theory}, 58\penalty0
  (12):\penalty0 7086--7093, 2012.

\bibitem[Sen et~al.(2019)Sen, Kandasamy, and Shakkottai]{sen2019noisy}
Rajat Sen, Kirthevasan Kandasamy, and Sanjay Shakkottai.
\newblock Noisy blackbox optimization using multi-fidelity queries: A tree
  search approach.
\newblock In \emph{The 22nd International Conference on Artificial Intelligence
  and Statistics}, pages 2096--2105, 2019.

\bibitem[Shao et~al.(2018)Shao, Yu, King, and Lyu]{shao2018almost}
Han Shao, Xiaotian Yu, Irwin King, and Michael~R Lyu.
\newblock Almost optimal algorithms for linear stochastic bandits with
  heavy-tailed payoffs.
\newblock In \emph{Advances in Neural Information Processing Systems}, pages
  8420--8429, 2018.

\bibitem[Srinivas et~al.(2010)Srinivas, Krause, Kakade, and
  Seeger]{srinivas2010gaussian}
Niranjan Srinivas, Andreas Krause, Sham Kakade, and Matthias Seeger.
\newblock Gaussian process optimization in the bandit setting: no regret and
  experimental design.
\newblock In \emph{Proceedings of the 27th International Conference on
  International Conference on Machine Learning}, pages 1015--1022. Omnipress,
  2010.

\bibitem[Sriperumbudur and Szab{\'o}(2015)]{sriperumbudur2015optimal}
Bharath Sriperumbudur and Zolt{\'a}n Szab{\'o}.
\newblock Optimal rates for random fourier features.
\newblock In \emph{Advances in Neural Information Processing Systems}, pages
  1144--1152, 2015.

\bibitem[Strogatz(2001)]{strogatz2001exploring}
Steven~H Strogatz.
\newblock Exploring complex networks.
\newblock \emph{nature}, 410\penalty0 (6825):\penalty0 268, 2001.

\bibitem[Vakili et~al.(2013)Vakili, Liu, and Zhao]{vakili2013deterministic}
Sattar Vakili, Keqin Liu, and Qing Zhao.
\newblock Deterministic sequencing of exploration and exploitation for
  multi-armed bandit problems.
\newblock \emph{IEEE Journal of Selected Topics in Signal Processing},
  7\penalty0 (5):\penalty0 759--767, 2013.

\bibitem[Wang and Jegelka(2017)]{wang2017max}
Zi~Wang and Stefanie Jegelka.
\newblock Max-value entropy search for efficient bayesian optimization.
\newblock In \emph{Proceedings of the 34th International Conference on Machine
  Learning-Volume 70}, pages 3627--3635. JMLR. org, 2017.

\bibitem[Wang et~al.(2016)Wang, Zhou, and Jegelka]{wang2016optimization}
Zi~Wang, Bolei Zhou, and Stefanie Jegelka.
\newblock Optimization as estimation with gaussian processes in bandit
  settings.
\newblock In \emph{Artificial Intelligence and Statistics}, pages 1022--1031,
  2016.

\bibitem[Wang et~al.(2017)Wang, Gehring, Kohli, and Jegelka]{wang2017batched}
Zi~Wang, Clement Gehring, Pushmeet Kohli, and Stefanie Jegelka.
\newblock Batched large-scale bayesian optimization in high-dimensional spaces.
\newblock \emph{arXiv preprint arXiv:1706.01445}, 2017.

\bibitem[Yang et~al.(2012)Yang, Li, Mahdavi, Jin, and Zhou]{yang2012nystrom}
Tianbao Yang, Yu-Feng Li, Mehrdad Mahdavi, Rong Jin, and Zhi-Hua Zhou.
\newblock Nystr{\"o}m method vs random fourier features: A theoretical and
  empirical comparison.
\newblock In \emph{Advances in neural information processing systems}, pages
  476--484, 2012.

\bibitem[Yu et~al.(2018)Yu, Shao, Lyu, and King]{yu2018pure}
Xiaotian Yu, Han Shao, Michael~R Lyu, and Irwin King.
\newblock Pure exploration of multi-armed bandits with heavy-tailed payoffs.
\newblock In \emph{Proceedings of the Thirty-Fourth Conference on Uncertainty
  in Artificial Intelligence}, pages 937--946, 2018.

\end{thebibliography}


\begin{appendix}
\begin{center}
\huge{Appendix}
\end{center}
\section{Preliminaries}
\label{appendix:prelim}
First, we review some useful matrix identities.
\begin{mylemma}\cite[Lemma 12]{hazan2007logarithmic}
Let $A \succeq B \succ 0$ be positive definite matrices. Then $A^{-1}\bullet(A-B)=\ln \frac{\abs{A}}{\abs{B}}$, where $X\bullet Y \bydef \sum_{i=1}^{n}\sum_{j=1}^{n}X_{i,j}Y_{i,j}$ for any two matrices $X,Y \in \Real^{n\times n}$.
\label{lem:useful-1}
\end{mylemma}
\begin{mylemma}
   \label{lem:dim-change}
   For any linear operator $A:\cH_k(\cX) \ra \Real^t$ and its adjoint $A^T:\Real^t \ra \cH_k(\cX)$, and for any $\lambda > 0$,
   \beq
   (A^T A + \lambda I_\cH)^{-1}A^T=A^T(AA^T+\lambda I_t)^{-1},
   \label{eqn:dim-change-1}
   \eeq
   and
   \beq
   I_\cH - A^T(AA^T+\lambda I_t)A = \lambda (A^T A + \lambda I_\cH)^{-1}.
   \label{eqn:dim-change-2}
   \eeq
   \end{mylemma}
   \begin{proof}
   The proofs follow from the fact that $(A^T A + \lambda I_\cH)A^T=A^T(AA^T+\lambda I_t)$ for any $\lambda > 0$.
   \end{proof}
Next, we review some relevant definitions and results, which will be useful in the analysis of our algorithms. We first begin with the definition of \textit{Maximum Information Gain}, first appeared in \cite{srinivas2010gaussian}, which basically measures the reduction in uncertainty about the unknown function after some noisy observations (rewards).

For a function $f:\cX \ra \Real$ and any subset $A \subset \cX$ of its domain, we use $f_A := [f(x)]_{x\in A}$ to denote its restriction to $A$, i.e., a vector containing $f$'s evaluations at each point in $A$ (under an implicitly understood bijection from coordinates of the vector to points in $A$). In case $f$ is a random function, $f_A$ will be understood to be a random vector. For jointly distributed random variables $X, Y$, $I(X;Y)$ denotes the Shannon mutual information between them.

\begin{mydefinition}[Maximum Information Gain (MIG)] 
\label{def:mig}
Let $f:\cX\ra \Real$ be a (possibly random) real-valued function defined on a domain $\cX$, and $t$ a positive integer. For each subset $\cX \subset D$, let $Y_A$ denote a noisy version of $f_A$ obtained by passing $f_A$ through a channel $\prob{Y_A | f_A}$. The \textit{Maximum Information Gain (MIG)} about $f$ after $t$ noisy observations is defined as
\beqn
\gamma_t \bydef \max_{A \subset \cX : \abs{A}=t} I(f_A;Y_A).
\eeqn
(We omit mentioning explicitly the dependence on the channels for ease of notation.)
\end{mydefinition}



Let $k: \cX \times \cX \ra \Real$ be a symmetric positive semi-definite kernel and for any $A \subset \cX$, let $K_A$ denotes the induced kernel matrix.
\begin{mylemma}[MIG under GP prior and additive Gaussian noise \cite{srinivas2010gaussian}]
\label{lem:info-gain}
Let $f \sim GP_{\cX}(0,k)$ be a sample from a Gaussian process over $\cX$ and $Y_A$ denote a noisy version of $f_A$ obtained by passing $f_A$ through a channel that adds iid $\cN(0,\lambda)$ noise to each element of $f_A$. Then, 
\beqn
\gamma_t  \equiv \gamma_t(k,\cX) =  \max_{A \subset \cX : \abs{A}=t} \frac{1}{2} \ln \abs{I + \lambda^{-1}K_A}.
\eeqn
\end{mylemma}
\citet{srinivas2010gaussian} proved upper bounds over $\gamma_t$ for commonly used kernels. 
The bounds are given in Lemma \ref{lem:info-gain-bound}.
\begin{mylemma}[MIG for common kernels \cite{srinivas2010gaussian}]
\label{lem:info-gain-bound}
Let $\cX$ be a compact and convex subset of $\Real^d$ and the kernel $k$ satisfies $k(x,x') \le 1$ for all $x,x' \in \cX$. Then for
\begin{itemize}
\item Linear kernel: $\gamma_t=O(d\ln t)$.
\item Squared Exponential kernel: $\gamma_t=O\left((\ln t)^{d+1}\right)$.
\item Mat$\acute{e}$rn kernel: $\gamma_t=O\left(t^{\frac{d(d+1)}{2\nu+d(d+1)}}\ln t\right)$.
\end{itemize}
\end{mylemma}

Note that, MIG depends only \textit{sublinearly} on the number of observations $t$ for all these kernels and it will serve as a key instrument to obtain our regret bounds by virtue of Lemma \ref{lem:info-gain} and \ref{lem:pred-var}.

Now, observe that
any kernel function $k : \cX \times \cX \ra \Real, \cX \subset \Real^d$ is associated with a non-linear feature map $\phi: \cX \ra \cH_k(\cX)$ such that $k(x,y)=\inner{\phi(x)}{\phi(y)}_{\cH}$, where $\inner{\cdot}{\cdot}_{\cH}$ denotes the inner product in the RKHS $\cH_k(\cX)$ and $\norm{\cdot}_{\cH}$ denotes the corresponding norm. Observe that for any $h \in \cH_k(\cX)$, $h(x)=\inner{h}{\phi(x)}_\cH$ by the reproducing property. For a set $\lbrace x_1,\ldots,x_t \rbrace \subset \cX$ define the operator $\Phi_t : \cH_k(\cX) \ra \Real^t$ such that for any $h \in \cH_k(\cX)$, $\Phi_t h = [\inner{\phi(x_1)}{h}_\cH,\ldots,\inner{\phi(x_t)}{h}_\cH]^T$, and denote its adjoint by $\Phi_t^T : \Real^t \ra \cH_k(\cX)$. By reproducing property $\phi_t h = [h(x_1),\ldots,h(x_t)]^T$. For any $\lambda > 0$, define  $V_t=\Phi_t^T \Phi_t + \lambda I_{\cH}$, where $I_{\cH} : \cH_k(\cX) \ra \cH_k(\cX)$ denotes the identity operator. For a positive definite operator $V:\cH_k(\cX) \ra \cH_k(\cX)$, define the inner product $\inner{\cdot}{\cdot}_V \bydef \inner{\cdot}{V\cdot}_\cH$
      with corresponding norm $\norm{\cdot}_V$. Observe that, under this definition, the posterior variance $\sigma_t^2(x) = \lambda\norm{\phi(x)}^2_{V_t^{-1}}$.
\begin{mylemma}[Sum of predictive variances and MIG]
If $k(x,x) \le 1$ for all $x \in \cX$, then
\beqn
\label{eqn:info-gain-two}
\sum_{s=1}^{t}\sigma^2_{s-1}(x_s) \le 2\left(1+\lambda\right)\gamma_t.
\eeqn
\label{lem:pred-var}
\end{mylemma}

\begin{proof}
Observe that $V_t=V_{t-1}+\phi(x_t)\phi(x_t)^T$. Therefore, by Sherman–Morrison-Woodbury matrix identity, we have $V_t^{-1}=V_{t-1}^{-1}-\frac{V_{t-1}^{-1}\phi(x_t)\phi(x_t)^TV_{t-1}^{-1}}{1+\phi(x_t)^TV_{t-1}^{-1}\phi(x_t)}$. This, in turn, implies that
\beqn
\norm{\phi(x)}^2_{V_t^{-1}}=\norm{\phi(x)}^2_{V_{t-1}^{-1}}-\frac{\inner{\phi(x)}{\phi(x_t)}^2_{V_{t-1}^{-1}}}{1+\norm{\phi(x_t)}^2_{V_{t-1}^{-1}}}\stackrel{(a)}{\ge}\norm{\phi(x)}^2_{V_{t-1}^{-1}}\left(1- \frac{\norm{\phi(x_t)}^2_{V_{t-1}^{-1}}}{1+\norm{\phi(x_t)}^2_{V_{t-1}^{-1}}}\right)=\frac{\norm{\phi(x)}^2_{V_{t-1}^{-1}}}{1+\norm{\phi(x_t)}^2_{V_{t-1}^{-1}}}
\eeqn
where $(a)$ follows from Cauchy-Schwartz inequality. Since $V_{t-1} \succeq \lambda I_{\cH}$, we have $\norm{\phi(x_t)}^2_{V_{t-1}^{-1}} \le \frac{1}{\lambda}\norm{\phi(x_t)}^2_\cH = \frac{1}{\lambda}k(x_t,x_t) \le \frac{1}{\lambda}$. This implies that $\norm{\phi(x)}^2_{V_{t-1}^{-1}} \le (1+\frac{1}{\lambda})\norm{\phi(x)}^2_{V_{t}^{-1}}$ and therefore 
\beq
\sigma_{t-1}^2(x)\le \left(1+\frac{1}{\lambda}\right)\sigma_{t}^2(x) \; \text{for all}\; x \in \cX. 
\label{eqn:comb-sd-1}
\eeq
Observe that $\phi(x_t)^TV_t^{-1}\phi(x_t)=V_t^{-1} \bullet \phi(x_t)\phi(x_t)^T = V_t^{-1} \bullet (V_t-V_{t-1})$ since for any $a \in \Real^n$ and $B \in \Real^{n \times n}$, $a^TBa=B\bullet aa^T$. Then from Lemma \ref{lem:useful-1}, we have $\frac{1}{\lambda}\sigma_t^2(x_t) = \ln \frac{|V_t|}{|V_{t-1}|}$ and thus, in turn,
\beq
\frac{1}{\lambda}\sum_{s=1}^{t}\sigma_s^2(x_s) \le \ln \frac{\abs{V_t}}{\abs{V_0}} = \ln \abs{\lambda^{-1}\Phi_t^T\Phi_t+I_\cH} = \ln \abs{\lambda^{-1}\Phi_t\Phi_t^T+I_t} = \ln \abs{\lambda^{-1}K_t+I_t}.
\label{eqn:comb-sd-2}
\eeq
Combining \ref{eqn:comb-sd-1} and \ref{eqn:comb-sd-2}, we get
\beqn
\sum_{s=1}^{t}\sigma^2_{s-1}(x_s) \le \left(1+\frac{1}{\lambda}\right)\sum_{s=1}^{t}\sigma^2_{s}(x_s) \le (1+\lambda)\ln \abs{\lambda^{-1}K_t+I_t}.
\eeqn
Now the result follows from Lemma \ref{lem:info-gain}.
\end{proof}

 \section{Analysis of TGP-UCB}
   
The following lemma states a self-normalized concentration inequality for RKHS-valued martingales.
   
   \begin{mylemma}[RKHS-valued martingale control \cite{durand2018streaming}]
         \label{lem:sub-Gaussian-noise-bound}
         Let $\lbrace z_t \rbrace_{t \ge 1}$ be an $\Real^d$-valued discrete time stochastic processes such that $z_t$ is predictable with respect to a filtration $\lbrace \cG_t\rbrace_{t \ge 0}$, i.e., $z_t$ is $\cG_{t-1}$-measurable for all $t \ge 1$. Let $\lbrace w_t \rbrace_{t \ge 1}$ be a real-valued stochastic process such that for all $t \ge 1$, $w_t$ is (a) $\cG_t$-measurable, and (b) $R$-sub-Gaussian conditionally on $\cG_{t-1}$ for some $R > 0$.
         Then, for any $\delta \in (0,1]$, with probability at least $1-\delta$, uniformly over all $t \ge 1$,
         \beqn
         \norm{\sum_{\tau=1}^{t}w_\tau \phi(z_\tau) }_{Z_t^{-1}}\le R\sqrt{2\left(\frac{1}{2}\ln \frac{\abs{Z_t}}{\abs{Z}} +\ln(1/\delta)\right)}.
         \eeqn
         where $Z_t=Z+\sum_{\tau=1}^{t}\phi(z_\tau)\phi(z_\tau)^T$ and $Z:\cH_k(\Real^d)\ra \cH_k(\Real^d)$ is a positive definite operator.
         \label{lem:hilbert-mart-control}
         \end{mylemma}
         Observe that $\sum_{\tau=1}^{t}w_\tau \phi(z_\tau)$ is $\cG_t$-measurable and $\expect{\sum_{\tau=1}^{t}w_\tau \phi(z_\tau)|\cG_{t-1}}=\sum_{\tau=1}^{t-1}w_\tau \phi(z_\tau)$. The process $\left(\sum_{\tau=1}^{t}w_\tau \phi(z_\tau)\right)_{t \ge 1}$ is thus a martingale with respect to the filtration $(\cG_t)_{t \ge 0}$ with values in the RKHS $\cH_k(\cX)$, whose deviation is measured by the norm weighted by $Z_t^{-1}$, which is derived from the process itself. Hence, the name self-normalized concentration inequality.
Now, we will show that $f$ lies in the confidence sets constructed by TGP-UCB with high probability.
\begin{mylemma}[Confidence sets of TGP-UCB contains $f$]
Let $f \in \cH_k(\cX)$, $\norm{f}_\cH \le B$ and $k(x,x) \le 1$ for all $x \in \cX$. Let $\expect{\abs{y_t}^{1+\alpha}|\cF_{t-1}} \le v < \infty$ for some $\alpha \in (0,1]$ and for all $t \ge 1$. Then, for any $\delta \in (0,1]$, TGP-UCB, with $b_t=v^{\frac{1}{1+\alpha}}t^{\frac{1}{2(1+\alpha)}}$ and $\beta_{t+1} = B + \frac{3}{\sqrt{\lambda}} \;v^{\frac{1}{1+\alpha}}t^{\frac{1}{2(1+\alpha)}}\sqrt{\ln \abs{I_t+\lambda^{-1}K_t}+2\ln(1/\delta)}$, ensures, with probability at least $1-\delta$, uniformly over all $x \in \cX$ and $t \ge 1$, that
\beqn
\abs{f(x)-\hat{\mu}_{t-1}(x)} \le \beta_t \sigma_{t-1}(x).
\eeqn
\label{lem:func-conc-tgp-ucb}
\end{mylemma}
\begin{proof}
First, we define $\alpha_t(x)=k_t(x)^T(K_t+\lambda I_t)^{-1}f_t$, where $f_t=[f(x_1),\ldots,f(x_t)]^T$ is a vector containing $f$'s evaluations up to round $t$. By reproducing property, $\alpha_t(x)=\inner{\phi(x)}{ \Phi_t^T(\Phi_t\Phi_t^T+\lambda I_t)^{-1}\Phi_tf}_\cH$. Then, we have
      \beqn
      f(x)-\alpha_t(x)=\inner{\phi(x)}{\left(I_\cH - \Phi_t^T(\Phi_t\Phi_t^T+\lambda I_t)^{-1}\Phi_t\right) f}_\cH \stackrel{(a)}{=}\lambda \inner{\phi(x)}{f}_{V_t^{-1}} = \lambda \inner{V_t^{-1/2} \phi(x)}{V_t^{-1/2}f}_\cH,
      \eeqn
      where $(a)$ follows from \ref{eqn:dim-change-2}. By Cauchy-Schwartz inequality, we have for any $x \in \cX$
      \beqa
      \abs{f(x)-\alpha_t(x)} & \le & \lambda \norm{V_t^{-1/2}\phi(x)}_\cH \norm{V_t^{-1/2}f}_\cH \nonumber\\
      & \stackrel{(a)}{\le} & \lambda^{1/2}\norm{\phi(x)}_{V_t^{-1}} \norm{f}_\cH 
      \stackrel{(b)}{\le}  B\; \sigma_t(x). \label{eqn:func-error} 
      \eeqa
      Here in $(a)$ we have used the fact that $V_t^{-1} \preceq \lambda^{-1}I_\cH$, and hence, $\norm{V_t^{-1/2}f}_\cH \le \lambda^{-1/2}\norm{f}_{\cH}$. $(b)$ follows from $\norm{f}_\cH \le B$. Now,
      let $\hat{\eta}_t = \hat{y}_t-f(x_t),t=1,2,\ldots$ denotes the truncated noise and $\hat{N}_t=[\hat{\eta}_1,\ldots,\hat{\eta}_t]^T$ denotes the vector formed by the first $t$ of those. This implies $\hat{\mu}_t(x)=\alpha_t(x)+k_t(x)^T(K_t+\lambda I_t)^{-1}\hat{N}_t$. Thus
   \beqn
   k_t(x)^T(K_t+\lambda I_t)^{-1}\hat{N}_t = \inner{\phi(x)}{\Phi_t^T(\Phi_t\Phi_t^T+\lambda I_t)^{-1}\hat{N}_t}_\cH \stackrel{(a)}{=}\inner{\phi(x)}{\Phi_t^T\hat{N}_t}_{V_t^{-1}},
   \eeqn
   where $(a)$ uses equation \ref{eqn:dim-change-1}. By Cauchy-Schwartz inequality, we have for any $x \in \cX$
   \beq
   \abs{k_t(x)^T(K_t+\lambda I_t)^{-1}\hat{N}_t} \le \norm{\phi(x)}_{V_t^{-1}}\norm{\Phi_t^T\hat{N}_t}_{V_t^{-1}} = \lambda^{-1/2}\norm{\Phi_t^T\hat{N}_t}_{V_t^{-1}}\sigma_t(x).
   \label{eqn:trunc-noise-bound}
   \eeq
   Now, by triangle inequality, we have
   \beqn
   \abs{f(x)-\hat{\mu}_t(x)} \le  \abs{f(x)-\alpha_t(x)} + \abs{k_t(x)^T(K_t+\lambda I_t)^{-1}\hat{N}_t}.
   \eeqn
   Hence from equation \ref{eqn:func-error} and \ref{eqn:trunc-noise-bound}, we get
   \beq
   \abs{f(x)-\hat{\mu}_t(x)} \le \left(B+\lambda^{-1/2}\norm{\Phi_t^T\hat{N}_t}_{V_t^{-1}}\right)\sigma_t(x).
   \label{eqn:combine-1}
   \eeq
    Now, we define $\xi_t=\hat{\eta}_t-\expect{\hat{\eta}_t | \cF_{t-1}}$. Then, we have
   \beq
   \Phi_t^T\hat{N}_t=\sum_{\tau = 1}^{t}\hat{\eta}_\tau\phi(x_\tau) = \sum_{\tau = 1}^{t}\xi_\tau\phi(x_\tau) + \sum_{\tau = 1}^{t}\expect{\hat{\eta}_\tau | \cF_{\tau-1}}\phi(x_\tau).
   \label{eqn:combine-2} 
   \eeq
     Observe that $\xi_t = \hat{y}_t-\expect{\hat{y}_t|\cF_{t-1}}$, and hence $\abs{\xi_t} \le 2b_t$. This implies that $\xi_t$ is zero-mean $2b_t$-sub-Gaussian random variable conditioned on $\cF_{t-1}$. Further, observe that $\xi_t$ is $\cF_t$- measurable and $x_t$ is $\cF_{t-1}$- measurable. Hence,
    Lemma \ref{lem:hilbert-mart-control} implies that for any $\delta \in (0,1]$, with probability at least $1-\delta$, for all $t \in \mathbb{N}$:
         \beqa
         \norm{\sum_{\tau = 1}^{t}\xi_\tau\phi(x_\tau)}_{V_t^{-1}} &\le & 2b_t\sqrt{2\left(\frac{1}{2}\ln \abs{I_\cH+\lambda^{-1}\Phi_t^T\Phi_t}+\ln(1/\delta)\right)}\nonumber \\ &=& 2b_t \sqrt{2\left(\frac{1}{2}\ln \abs{I_t+\lambda^{-1}K_t}+\ln(1/\delta)\right)}
         \label{eqn:combine-4}
         \eeqa
 Now for any $a \in \Real^t$,
   \beqn
   \norm{\sum_{\tau = 1}^{t}a_\tau\phi(x_\tau)}^2_{V_t^{-1}} = \norm{\Phi_t^Ta}^2_{V_t^{-1}} = a^T\Phi_t(\Phi_t^T\Phi_t+\lambda I_\cH)^{-1}\Phi_t^Ta \stackrel{(a)}{=} a^T\Phi_t\Phi_t^T(\Phi_t\Phi_t^T+\lambda I_t)^{-1}a \stackrel{(b)}{\le} \norm{a}_2^2,
   \eeqn
   where $(a)$ follows from \ref{eqn:dim-change-1} and $(b)$ follows from the fact that $\Phi_t\Phi_t^T(\Phi_t\Phi_t^T+\lambda I_t)^{-1} \preceq I_t$. Therefore $\norm{\sum_{\tau = 1}^{t}\expect{\hat{\eta}_\tau | \cF_{\tau-1}}\phi(x_\tau)}_{V_t^{-1}}^2 \le \sum_{\tau = 1}^{t}\expect{\hat{\eta}_\tau | \cF_{\tau-1}}^2$. Further, observe that $\expect{\hat{\eta}_t|\cF_{t-1}}=\expect{y_t\mathds{1}_{\abs{y_t}\le b_t}|\cF_{t-1}}-f(x_t)=-\expect{y_t\mathds{1}_{\abs{y_t}> b_t}|\cF_{t-1}}$. This implies
   \beqn
    \norm{\sum_{\tau = 1}^{t}\expect{\hat{\eta}_\tau | \cF_{\tau-1}}\phi(x_\tau)}_{V_t^{-1}}^2 \le \sum_{\tau=1}^{t}\expect{y_\tau\mathds{1}_{\abs{y_\tau}> b_\tau}|\cF_{\tau-1}}^2
    \le \sum_{\tau=1}^{t}\frac{1}{b_\tau^{2\alpha}}\expect{\abs{y_\tau}^{1+\alpha}|\cF_{\tau-1}}^2 \le v^2\sum_{\tau=1}^{t}\frac{1}{b_\tau^{2\alpha}}.
   \eeqn
   Now setting $b_t = v^{\frac{1}{1+\alpha}}t^{\frac{1}{2(1+\alpha)}}$, we get
   \beq
   \norm{\sum_{\tau = 1}^{t}\expect{\hat{\eta}_\tau | \cF_{\tau-1}}\phi(x_\tau)}_{V_t^{-1}} \le v^{\frac{1}{1+\alpha}}\sqrt{\sum_{\tau = 1}^{t}\tau^{-\frac{\alpha}{1+\alpha}}} \le v^{\frac{1}{1+\alpha}} \sqrt{\int_{0}^{t}\tau^{-\frac{\alpha}{1+\alpha}}d\tau} \le \sqrt{2} v^{\frac{1}{1+\alpha}}t^{\frac{1}{2(1+\alpha)}}.
   \label{eqn:combine-3}
   \eeq
   Combining \ref{eqn:combine-1},\ref{eqn:combine-2}, \ref{eqn:combine-4} and \ref{eqn:combine-3}, we have that for any $\delta \in (0,1]$, with probability at least $1-\delta$, uniformly over all $t \ge 1$ and $x \in \cX$:
      \beqa
      \abs{f(x)-\hat{\mu}_t(x)} &\le & \left(B + \sqrt{2/\lambda}\;v^{\frac{1}{1+\alpha}}t^{\frac{1}{2(1+\alpha)}}\big(1+2\sqrt{\frac{1}{2}\ln \abs{I_t+\lambda^{-1}K_t}+\ln(1/\delta)}\right)\sigma_t(x)\nonumber \\
      &\le & \left(B + 3\sqrt{2/\lambda}\;v^{\frac{1}{1+\alpha}}t^{\frac{1}{2(1+\alpha)}}\sqrt{\frac{1}{2}\ln \abs{I_t+\lambda^{-1}K_t}+\ln(1/\delta)}\right)\sigma_t(x).
      \label{eqn:true-function-bound}
      \eeqa
      Further observe that $\abs{f(x)-\hat{\mu}_0(x)}=\abs{f(x)} = \abs{\inner{f}{k(x,\cdot)}_\cH} \le \norm{f}_\cH k^{1/2}(x,x) \le B \sigma_0(x)$.
      Now the result follows by setting $\beta_{t+1} = B + \frac{3}{\sqrt{\lambda}} \;v^{\frac{1}{1+\alpha}}t^{\frac{1}{2(1+\alpha)}}\sqrt{\ln \abs{I_t+\lambda^{-1}K_t}+2\ln(1/\delta)}$,   for all $t \ge 0$.
      \end{proof}

Now, we will prove Theorem 
         \ref{thm:regret-bound-tgp-ucb}. For for any $\delta \in (0,1]$, we have, with probability at least $1-\delta$, uniformly over all $t \ge 1$, the instantaneous regret of TGP-UCB (Algorithm \ref{algo:tgp-ucb}) is
	\beqan
	r_t &=& f(x^\star)-f(x_t)\\
	& \stackrel{(a)}{\le} & \hat{\mu}_{t-1}(x^\star) +  \beta_{t}\sigma_{t-1}(x^\star)- f(x_t)\\
	& \stackrel{(b)}{\le} & \hat{\mu}_{t-1}(x_t) +  \beta_{t}\sigma_{t-1}(x_t) -f(x_t)\\
	& \stackrel{(c)}{\le} & 2\beta_{t}\sigma_{t-1}(x_t).
	\eeqan
	Here $(a)$ and $(c)$ follow from \ref{eqn:true-function-bound}, and $(b)$ is due to the choice of TGP-UCB(Algorithm \ref{algo:tgp-ucb}). Since from Lemma \ref{lem:info-gain}, $\ln \abs{I_t+\lambda^{-1}K_t} \le \gamma_t$, we have $\beta_t \le B + 3\sqrt{2/\lambda}\;v^{\frac{1}{1+\alpha}}t^{\frac{1}{2(1+\alpha)}}\sqrt{\gamma_t+\ln(1/\delta)}$, which is an increasing sequence $t$. Further, see that $\sum_{t=1}^{T}\sigma_{t-1}(x_t) \stackrel{(a)}{\le} \sqrt{T\sum_{t=1}^{T}\sigma^2_{t-1}(x_t)} \stackrel{(b)}{\le} \sqrt{2(1+\lambda)\gamma_T T}$, where $(a)$ is due to Cauchy-Schwartz inequality and $(b)$ is due to Lemma \ref{lem:pred-var}. Hence, for any $\delta \in (0,1]$, with probability at least $1-\delta$, the cumulative regret of TGP-UCB after $T$ rounds is
	\beqan
	R_T = O\left(B\sqrt{T\gamma_T}+v^{\frac{1}{1+\alpha}}\sqrt{\gamma_T(\gamma_T+\ln(1/\delta))} T^{\frac{2+\alpha}{2(1+\alpha)}}\right).
	\eeqan

\section{Regret lower bound: proof of Theorem \ref{thm:lower-bound}} 
\label{appendix:lower-nound}
Our analysis builds heavily on that of the optimization setting with $f \in \cH_k(\cX)$ and with Gaussian noise studied in \cite{scarlett2017lower}, but with important differences. Roughly speaking, we use the same construction of $f$ as in \cite{scarlett2017lower}, but we construct the rewards differently to capture the heavy-tailed scenario. We now proceed with the formal proof.
\subsection{Construction of the ground-truth function}
\label{subsec:function-class}
\begin{itemize}
\item Let $g(x)$ be a function on $\Real^d$ with the following properties:
\begin{enumerate}
\item The RKHS norm of $g$ is bounded: $\norm{g}_\cH \le B$.
\item $\abs{g(x)}\le 2\Delta$ with a maximum value of $2\Delta$ at $x=0$ and $g(x) < \Delta$ when $\norm{x}_\infty > w$ for some $w > 0$ and $\Delta > 0$, to be chosen later.
\end{enumerate}
\item Letting $g(x)$ be such a function, we construct $M$ functions $f_1,\ldots,f_M$ first by shifting $g$ such that each $f_j$ has its maximum at a unique point in a uniform grid, and then by restricting them to the domain $\cX=[0,1]^d$. Using a step size $w$ in each dimension, one can construct a grid of size $M=\lfloor \left(\frac{1}{w}\right)^d \rfloor$ of the domain $\cX$, and hence $M$ such functions $f_j$. In this process we ensure that any $\Delta$-optimal point for $f_j$ fails to be $\Delta$-optimal point for any other $f_{j'}$.
\item Finally, we choose $f$ as a uniformly sampled function from the set $\lbrace f_1,\ldots,f_M \rbrace$.
\end{itemize}
It remains to choose $g$, $w$, and $\Delta$ so that the above properties are satisfied. 
\begin{itemize}
\item For some absolute constant $\zeta > 0$ we choose $g(x)=\frac{2\Delta}{h(0)}h(\frac{x\zeta}{w})$, where $h$ is the inverse Fourier transform of the \textit{multi-dimensional bump function}: $H(\omega)=e^{-\frac{1}{1-\norm{\omega}_2^2}}\mathds{1}_{\lbrace \norm{\omega}_2^2 \le 1 \rbrace}$.
Note that since $H$ is real and symmetric, the maximum of $h$ is attained at $x=0$, and hence the maximum of $g$ is $g(0)=2\Delta$, as desired. Further, since $H$ has finite energy, $h(x) \ra 0$ as $\norm{x}_2 \ra \infty$. Hence, there exists an absolute constant $\zeta$ such that $h(x) < \frac{1}{2}h(0)$ when $\norm{x}_\infty > \zeta$, and thus $g(x) < \Delta$ for $\norm{x}_\infty > w$, as desired. 
\item It now remains to choose $w$ and $\Delta$ to ensure that $\norm{g}_\cH \le B$, for a given $B$. Note that, while a smaller $\Delta$ ensures a low RKHS norm, a smaller $w$ increases it. Hence, as long as $\Delta$ is very small, we can afford to take $w << 1$, so that there is no risk of having $M=0$. For $\frac{\Delta}{B} << 1$, it is shown in \cite{scarlett2017lower} that the condition $\norm{g}_\cH \le B$ can be achieved with $w=\frac{\zeta \pi l}{\sqrt{\ln \frac{B(2\pi l^2)^{d/4}h(0)}{2\Delta}}}$ for the SE kernel, and with $w=\tau \left(\frac{2\Delta (8\pi^2)^{(\nu+d/2)/2}}{Bc^{-1/2}h(0)}\right)^{1/\nu}$ for the Mat\'{e}rn kernel. We consider $\Delta$ as arbitrary for now, but later this will be chosen to ensure that $\frac{\Delta}{B}$ is sufficiently small.
\item From the choice of $w$, we see that $M = \Theta\left((\ln \frac{B}{\Delta})^d\right)$ for the SE kernel, and $M=\Theta\left((\frac{B}{\Delta})^{\frac{d}{\nu}}\right)$ for the Mat\'{e}rn kernel. Note that the assumption of sufficiently small $\frac{\Delta}{B}$ in ensures that $M >> 1$, i.e. there are enough number of functions to sample from.

\end{itemize}

\subsection{Construction of the reward distribution}
\label{subsec:reward-dist} 
For any given $\alpha \in (0,1]$, $v > 0$ and $x \in [0,1]^d$, we define the reward distribution as 
\beq
 y(x) =\begin{cases}
    sgn\left(f(x)\right)\left(\frac{v}{2\Delta}\right)^{\frac{1}{\alpha}} & \text{with probability}\; \left(\frac{2\Delta}{v}\right)^{\frac{1}{\alpha}}\abs{f(x)},\\
    0 & \text{otherwise}.
  \end{cases}
  \label{eqn:reward-dist}
\eeq
Note that \ref{eqn:reward-dist} is a valid probability distribution as long as $\Delta \le \frac{1}{2}v^{\frac{1}{1+\alpha}}$.
Then, $\expect{y(x)}=\abs{f(x)}sgn((f(x))=f(x)$ and
$\expect{\abs{y(x)}^{1+\alpha}}=\left(\frac{v}{2\Delta}\right)^{\frac{1+\alpha}{\alpha}}\left(\frac{2\Delta}{v}\right)^{\frac{1}{\alpha}}\abs{f(x)}=\frac{v\abs{f(x)}}{2\Delta} \le v$ for any $\alpha \in (0,1]$. Thus, we ensure that the $(1+\alpha)$-th absolute moment of the rewards are upper bounded by $v$. 

\subsection{Preliminary notations and lemmas}
\label{subsec:prelim}
Now, we introduce the following notations, also used in \cite{scarlett2017lower}:
\begin{itemize}
\item $y_m$ denote the reward function when the underlying ground truth is $f_m$ for $m=1,\ldots,M$. $f_0$ denotes the function which is zero everywhere, and $y_0$ the corresponding reward function. $P_m(Y_T)$ (resp. $P_0(Y_T)$) denotes the probability density function of the reward sequence $Y_T=\lbrace y_1,\ldots,y_T \rbrace$ when the underlying function is $f_m$ (resp. $f_0$). $P_m(y|x)$ (resp. $P_0(y|x)$) denotes the conditional density of the reward $y$ given the selected point $x$ when the underlying function is $f_m$ (resp. $f_0$).
\item $\mathbb{E}_m$ (resp. $\mathbb{E}_0$) and $\mathbb{P}_m$ (resp. $\mathbb{P}_0$) denote expectations and probabilities (with respect to the noisy rewards) when the underlying function is $f_m$ (resp. $f_0$). $\mathbb{E}[\cdot]=\frac{1}{M}\sum_{m=1}^{M}\mathbb{E}_m[\cdot]$ (resp. $\mathbb{P}_m[\cdot]$) denote the expectation (resp. probability) with respect to the noisy rewards and $f$ drawn uniformly from $\lbrace f_1,\ldots,f_M \rbrace$.
\item $\lbrace \cR_m \rbrace_{m=1}^M$ denote a partition of $\cX$ into $M$ regions such that each $f_m,m=1,\ldots,M$ has its maximum at the center of $\cR_m$. $v_m^j=\max_{x \in \cR_j}\abs{f_m(x)}$ denotes the maximum absolute value of $f_m$ in the region $\cR_j$ and $D_m^j=\max_{x \in \cR_j}D_{\text{KL}}\left(P_0(\cdot|x)||P_m(\cdot|x)\right)$ denotes the maximum KL divergence between $P_0(\cdot|x)$ and $P_m(\cdot|x)$ within $\cR_j$. $N_j=\sum_{t=1}^{T}\mathds{1}_{\lbrace x_t \in \cR_j\rbrace}$ denotes the number of points within $\cR_j$ that are selected up to time $T$.
\end{itemize}
Next, we present some useful lemmas from \cite{scarlett2017lower}. 
\begin{mylemma}\cite[Lemma 3]{scarlett2017lower}
Under the preceding definitions, we have $\mathbb{E}_m[N_j] \le \mathbb{E}_0[N_j] + T \sqrt{D_{KL}(P_0||P_m)}$ for all $m=1,\ldots,M$ and $j=1,\ldots,M$.
\label{lem:change-of-measure}
\end{mylemma}
\begin{mylemma}\cite[Lemma 4]{scarlett2017lower}
Under the preceding definitions, we have $D_{\text{KL}}(P_0||P_m) \le \sum_{j=1}^{M}\mathbb{E}_0[N_j]D^j_m$ for all $m=1,\ldots,M$.
\end{mylemma}

\begin{mylemma}\cite[Lemma 5]{scarlett2017lower} The functions $f_m$ constructed in Section \ref{subsec:function-class} are such that the quantities $v_m^j$ satisfy:\\
(a) $\sum_{m=1}^{M}v_m^j = O(\Delta)$ for all $j=1,\ldots,M$ and
(b) $\sum_{j=1}^{M}v_m^j = O(\Delta)$ for all $m=1,\ldots,M$. 
\label{lem:func-class-prop}
\end{mylemma}

\subsection{Analysis of expected cumulative regret}
Observe that $\mathbb{E}_m[f(x_t)] \le \sum_{j=1}^{M}\mathbb{P}_m[x_t \in \cR_j]v_m^j$. This implies
\beqn
\mathbb{E}_m\left[\sum_{t=1}^{T}f(x_t)\right] \le \sum_{j=1}^{M}v_m^j\mathbb{E}_m[N_j] \le \sum_{j=1}^{M}v_m^j\left(\mathbb{E}_0[N_j]+T\sqrt{\sum_{j'=1}^{M}\mathbb{E}_0[N_{j'}]D_m^{j'}}\right), 
\eeqn
where the last inequality follows from Lemma \ref{lem:change-of-measure}. Now averaging over $m=1,\ldots,M$ we obtain the following:
\beq
\expect{\sum_{t=1}^{T}f(x_t)} \le \frac{1}{M} \sum_{m=1}^{M}\sum_{j=1}^{M}v_m^j\left(\mathbb{E}_0[N_j]+T\sqrt{\sum_{j'=1}^{m}\mathbb{E}_0[N_{j'}]D_m^{j'}}\right).
\label{eqn:lbcomb1}
\eeq
We can bound the first term as follows:
\beq
\frac{1}{M} \sum_{m=1}^{M}\sum_{j=1}^{M}v_m^j\mathbb{E}_0[N_j] = \frac{1}{M} \sum_{j=1}^{M}\sum_{m=1}^{M}v_m^j\mathbb{E}_0[N_j]\stackrel{(a)}{=}O\left(\frac{\Delta}{M}\right)\sum_{j=1}^{M}\mathbb{E}_0[N_j]\stackrel{(b)}{=} O\left(\frac{T\Delta}{M}\right),
\label{eqn:lbcomb2}
\eeq
where $(a)$ follows from part $(a)$ of Lemma \ref{lem:func-class-prop}, and $(b)$ follows from $\sum_{j=1}^{M}N_j=T$. In order to bound the second term, first
we note that $y_0(x)=0$ for all $x \in \cX$. Therefore, we have
\beqan
 D_{\text{KL}}\left(P_0(\cdot|x)||P_m(\cdot|x)\right)=\ln \frac{1}{1-\left(\frac{2\Delta}{v}\right)^{\frac{1}{\alpha}}\abs{f_m(x)}} &\stackrel{(a)}{\le}& \frac{\left(\frac{2\Delta}{v}\right)^{\frac{1}{\alpha}}\abs{f_m(x)}}{1-\left(\frac{2\Delta}{v}\right)^{\frac{1}{\alpha}}\abs{f_m(x)}}\\
 &\stackrel{(b)}{\le}& \frac{\left(\frac{2\Delta}{v}\right)^{\frac{1}{\alpha}}\abs{f_m(x)}}{1-(2\Delta)^{\frac{1+\alpha}{\alpha}}v^{-\frac{1}{\alpha}}}\\
 &\stackrel{(c)}{\le}& 2\left(\frac{2\Delta}{v}\right)^{\frac{1}{\alpha}}\abs{f_m(x)}.
 \eeqan
Here $(a)$ holds because $\ln(x) \le x-1$ for all $x \ge 1$, $(b)$ holds as $\abs{f(x)} \le 2\Delta$ and $(c)$ holds for $\Delta \le \frac{1}{2}\left(\frac{1}{2}\right)^{\frac{\alpha}{1+\alpha}}v^{\frac{1}{1+\alpha}}$. Observe that this choice of $\Delta$ is compatible with \ref{eqn:reward-dist}. This implies that for all $j=1,\ldots,M$,
\beq
D_m^j \le 2^{\frac{1+\alpha}{\alpha}}\left(\frac{\Delta}{v}\right)^{\frac{1}{\alpha}}v_m^j \; \text{if} \; \Delta \le \frac{1}{2}\left(\frac{1}{2}\right)^{\frac{\alpha}{1+\alpha}}v^{\frac{1}{1+\alpha}}.
\label{eqn:KL-use}
\eeq
Now, we can bound the second term as follows:
\beqa
\frac{1}{M} \sum_{m=1}^{M}\sum_{j=1}^{M}v_m^j\sqrt{\sum_{j'=1}^{m}\mathbb{E}_0[N_{j'}]D_m^{j'}} &\stackrel{(a)}{=}& O(\Delta)\frac{1}{M}\sum_{m=1}^{M}\sqrt{ \sum_{j'=1}^{M}\mathbb{E}_0[N_{j'}]D_m^{j'}}\nonumber\\
&\stackrel{(b)}{\le}& O(\Delta)\sqrt{\frac{1}{M}\sum_{m=1}^{M} \sum_{j'=1}^{M}\mathbb{E}_0[N_{j'}]D_m^{j'}}\nonumber\\
&\stackrel{(c)}{\le}& O(\Delta)2^{\frac{1+\alpha}{2\alpha}}\left(\frac{\Delta}{v}\right)^{\frac{1}{2\alpha}}\sqrt{\frac{1}{M}\sum_{m=1}^{M} \sum_{j'=1}^{M}\mathbb{E}_0[N_{j'}]v_m^{j'}}\nonumber\\
&\stackrel{(d)}{=}&  O(\Delta)2^{\frac{1+\alpha}{2\alpha}}\left(\frac{\Delta}{v}\right)^{\frac{1}{2\alpha}}\sqrt{O\left(\frac{\Delta}{M}\right) \sum_{j'=1}^{M}\mathbb{E}_0[N_{j'}]}\nonumber\\
&\stackrel{(e)}{=}&
O\left(\Delta\frac{\left(2\Delta\right)^{\frac{1+\alpha}{2\alpha}}}{v^{\frac{1}{2\alpha}}}\sqrt{\frac{T}{M}}\right).
\label{eqn:lbcomb3}
\eeqa
Here $(a)$ follows from part $(b)$ of Lemma \ref{lem:func-class-prop}, $(b)$ follows from Jensen's inequality, $(c)$ follows from \ref{eqn:KL-use} if  $\Delta \le \frac{1}{2}\left(\frac{1}{2}\right)^{\frac{\alpha}{1+\alpha}}v^{\frac{1}{1+\alpha}}$, $(d)$ follows from part $(a)$ of Lemma \ref{lem:func-class-prop}, and $(e)$ follows from $\sum_{j=1}^{M}N_j=T$.
Substituting \ref{eqn:lbcomb2} and \ref{eqn:lbcomb3} in \ref{eqn:lbcomb1} gives
\beq
\expect{\sum_{t=1}^{T}f(x_t)} \le C T\Delta \left(\frac{1}{M}+\frac{\left(2\Delta\right)^{\frac{1+\alpha}{2\alpha}}}{v^{\frac{1}{2\alpha}}}\sqrt{\frac{T}{M}}\right) \; \text{for} \; \Delta \le \frac{1}{2}\left(\frac{1}{2}\right)^{\frac{\alpha}{1+\alpha}}v^{\frac{1}{1+\alpha}}.
\label{eqn:total-reward}
\eeq 
Since $f(x^\star)=2\Delta$, the expected cumulative regret
\beqn
\expect{R_T} = Tf(x^\star) - \expect{\sum_{t=1}^{T}f(x_t)} \ge T\Delta\left(2-\frac{C}{M}-\frac{C\left(2\Delta\right)^{\frac{1+\alpha}{2\alpha}}}{v^{\frac{1}{2\alpha}}}\sqrt{\frac{T}{M}}\right) \; \text{for} \; \Delta \le \frac{1}{2}\left(\frac{1}{2}\right)^{\frac{\alpha}{1+\alpha}}v^{\frac{1}{1+\alpha}}.
\eeqn
Since $M \ra \infty$ as $\frac{\Delta}{B} \ra 0$,
we have $\frac{C}{M}\le \frac{1}{2}$ for sufficiently small $\frac{\Delta}{B}$. Hence, we have 
\beqa
\expect{R_T} &\ge & T\Delta\left(\frac{3}{2}-C\frac{\left(2\Delta\right)^{\frac{1+\alpha}{2\alpha}}}{v^{\frac{1}{2\alpha}}}\sqrt{\frac{T}{M}}\right)\nonumber\\
& \ge& T\Delta \quad\text{for}\;\Delta \le \frac{1}{2}\left(\min \Big\lbrace \frac{1}{2},\frac{M}{4C^2T}\Big\rbrace\right)^{\frac{\alpha}{1+\alpha}}v^{\frac{1}{1+\alpha}}.
\label{eqn:regret-lower-bound}
\eeqa
Now, if $M \le 2C^2T$, then
\beq
\expect{R_T} = \Omega \left(v^{\frac{1}{1+\alpha}}M^{\frac{\alpha}{1+\alpha}}T^{\frac{1}{1+\alpha}}\right) \;\text{for}\; \frac{1}{4} \left(\frac{M}{4C^2T}\right)^{\frac{\alpha}{1+\alpha}}v^{\frac{1}{1+\alpha}} \le \Delta \le \frac{1}{2} \left(\frac{M}{4C^2T}\right)^{\frac{\alpha}{1+\alpha}}v^{\frac{1}{1+\alpha}}.
\label{eqn:regret-lower-bound-2}
\eeq

\subsubsection{Application to the squared exponential kernel} For the SE kernel, we have from the choice $M=\Theta\left((\ln \frac{B}{\Delta})^d\right)$, along with the upper and lower bounds on $\Delta$ in \ref{eqn:regret-lower-bound-2}, that $\Delta = \Theta \left(\left(\frac{1}{T}(\ln \frac{B}{\Delta})^d\right)^{\frac{\alpha}{1+\alpha}}v^{\frac{1}{1+\alpha}}\right)$. This, in turn, implies that $\ln \frac{B}{\Delta}=\ln \frac{B T^{\frac{\alpha}{1+\alpha}}}{v^{\frac{1}{1+\alpha}}} -\ln \left(\Theta(1)\left(\ln \frac{B}{\Delta}\right)^{\frac{d\alpha }{1+\alpha}}\right)$. Since $d=O(1)$ and $\frac{\alpha}{1+\alpha} \in (0,\frac{1}{2}]$, the second term behaves as $\Theta(\ln \ln \frac{B}{\Delta})$, which is $\Theta\left(\frac{1}{2}\ln \frac{B}{\Delta}\right)$ for sufficiently small $\frac{\Delta}{B}$ . This, implies that $\ln \frac{B}{\Delta}=\Theta \left(\ln \frac{B T^{\frac{\alpha}{1+\alpha}}}{v^{\frac{1}{1+\alpha}}}\right)$, and thus, in turn, $M=\Theta \left(\left(\ln \frac{B T^{\frac{\alpha}{1+\alpha}}}{v^{\frac{1}{1+\alpha}}}\right)^d\right)$ and $\Delta = \Theta \left(v^{\frac{1}{1+\alpha}}\left(\ln \frac{B T^{\frac{\alpha}{1+\alpha}}}{v^{\frac{1}{1+\alpha}}}\right)^{\frac{d\alpha}{1+\alpha}}T^{-\frac{\alpha}{1+\alpha}}\right)$. Note that the choice of $M$ ensures that $M \le 2C^2T$ and the choice of $\Delta$ ensures that $\frac{\Delta}{B}$ is indeed sufficiently small as long as $v^{\frac{1}{1+\alpha}} \le C' BT^{\frac{\alpha}{1+\alpha}}$ for some sufficiently small constant $C'$ \footnote{In our setting, $B$ and $v$ are constants that do not scale with $T$ and the condition is trivially satisfied.}.
Now, substituting $M$ in \ref{eqn:regret-lower-bound-2}, we obtain $\expect{R_T} = \Omega \left(v^{\frac{1}{1+\alpha}}\left(\ln \frac{B T^{\frac{\alpha}{1+\alpha}}}{v^{\frac{1}{1+\alpha}}}\right)^{\frac{d\alpha}{1+\alpha}}T^{\frac{1}{1+\alpha}}\right) = \Omega \left(v^{\frac{1}{1+\alpha}}\left(\ln \frac{B^{\frac{1+\alpha}{\alpha}} T}{v^{\frac{1}{\alpha}}}\right)^{\frac{d\alpha}{1+\alpha}}T^{\frac{1}{1+\alpha}}\right)$, since, generally, $d=O(1)$ and $\frac{\alpha}{1+\alpha} \in (0,\frac{1}{2}]$.

\subsubsection{Application to the Mat\'{e}rn kernel}
For the Mat\'{e}rn kernel, we have from the choice $M=\Theta\left((\frac{B}{\Delta})^{\frac{d}{\nu}}\right)$, along with the upper and lower bounds on $\Delta$ in \ref{eqn:regret-lower-bound-2}, that $\Delta = \Theta \left(\left(\frac{1}{T}\left( \frac{B}{\Delta}\right)^{\frac{d}{\nu}}\right)^{\frac{\alpha}{1+\alpha}}v^{\frac{1}{1+\alpha}}\right)$. This, in turn, implies that $\Delta =\Theta \left( v^{\frac{\nu/(1+\alpha)}{\nu+d\alpha/(1+\alpha)}}B^{\frac{d\alpha/(1+\alpha)}{\nu+d\alpha/(1+\alpha)}}T^{-\frac{\nu \alpha/(1+\alpha)}{\nu+d\alpha/(1+\alpha)}}\right)$ and $M=\Theta \left( v^{-\frac{d/(1+\alpha)}{\nu+d\alpha/(1+\alpha)}}B^{\frac{d}{\nu+d\alpha/(1+\alpha)}}T^{\frac{d \alpha/(1+\alpha)}{\nu+d\alpha/(1+\alpha)}}\right)$. Once again, we see that the choice of $M$ ensures that $M \le 2C^2T$ and the choice of $\Delta$ ensures that $\frac{\Delta}{B}$ is indeed sufficiently small as long as $v^{\frac{1}{1+\alpha}} \le C' BT^{\frac{\alpha}{1+\alpha}}$ for some sufficiently small constant $C'$. Now, substituting $M$ in \ref{eqn:regret-lower-bound-2}, we obtain $\expect{R_T}=\Omega \left( v^{\frac{\nu/(1+\alpha)}{\nu+d\alpha/(1+\alpha)}}B^{\frac{d\alpha/(1+\alpha)}{\nu+d\alpha/(1+\alpha)}}T^{\frac{1}{1+\alpha}\frac{\nu+d\alpha}{\nu+d\alpha/(1+\alpha)}}\right)=\Omega \left( v^{\frac{\nu}{\nu(1+\alpha)+d\alpha}}B^{\frac{d\alpha}{\nu(1+\alpha)+d\alpha}}T^{\frac{\nu+d\alpha}{\nu(1+\alpha)+d\alpha}}\right)$.

\section{Analysis of ATA-GP-UCB}
\subsection{Construction of tighter confidence set using data adaptive truncation}
The following lemma helps us to show that $(1+\alpha)$-th norm of $u_i \in \Real^t$ is $t^{\frac{1-\alpha}{2(1+\alpha)}}$, where $u_i^T,i \in [m_t]$ are the rows of $\tilde{V}_t^{-1/2}\tilde{\Phi}_t^T$.
\begin{mylemma}
\label{lem:col-norm}
Let $A \in \Real^{p \times q}$. Let $c_i \in \Real^p, i =1,\ldots,q$ be the $i$-th column of $A(A^TA + \lambda I_q)^{-1/2}$. Then for any $\beta \in [1,\infty)$, we have $\norm{c_i}_\beta \le p^{\frac{2-\beta}{2\beta}}$ for all $i \in [q]$.
\end{mylemma}

\begin{proof}
Let the singular value decomposition of $A$ be $U\Sigma V^T$, where $U$ and $V$ are unitary matrices. This implies $A(A^TA + \lambda I_q)^{-1/2}=U\Sigma(\Sigma^T \Sigma + \lambda I_q)^{-1/2}V^T$. Now, the $i$-th column of $A(A^TA + \lambda I_q)^{-1/2}$ is given by $c_i=U\Sigma(\Sigma^T \Sigma + \lambda I)^{-1/2}V^T e_i$. Therefore,
\beqan
\norm{c_i}_2 = \norm{U\Sigma(\Sigma^T \Sigma + \lambda I)^{-1/2}V^T e_i}_2&=&\norm{\Sigma(\Sigma^T \Sigma + \lambda I)^{-1/2}V^T e_i}_2\\
& \le & \norm{\Sigma(\Sigma^T \Sigma + \lambda I)^{-1/2}}_2 \norm{V^T e_i}_2
 \le 1. 
\eeqan
Now the result follows from the fact that for any $a \in \Real^p, \norm{a}_2 \le 1$ the maximum value of $\norm{a}_\beta$ for any $\beta \in [1,\infty)$ is $p^{\frac{2-\beta}{2\beta}}$ with the maximum attained at $[\frac{1}{\sqrt{p}},\ldots,\frac{1}{\sqrt{p}}]^T$.
\end{proof} 
%
Now, we will show that the data adaptive truncation of ATA-GP-UCB helps us to achieve tighter confidence sets than TGP-UCB.

\begin{mylemma}[Effect of data adaptive truncation]
For any $\delta \in (0,1]$, ATA-GP-UCB with $b_t = \left(v/\ln(2m_t T/\delta)\right)^{\frac{1}{1+\alpha}}t^{\frac{1-\alpha}{2(1+\alpha)}}$, ensures, with probability at least $1-\delta$, that uniformly over all $t \in [T]$,
\beqn
\norm{\tilde{V}_t^{-1}\tilde{\Phi}_t^T f_t-\tilde{\theta}_t}_{\tilde{V}_t} \le 4 \sqrt{m_t}\; v^{\frac{1}{1+\alpha}}\left(\ln(2m_tT/\delta)\right)^{\frac{\alpha}{1+\alpha}}t^{\frac{1-\alpha}{2(1+\alpha)}},
\eeqn
where $f_t=[f(x_1),\ldots,f(x_t)]^T$ is a vector containing $f$'s evaluations up to round $t$.
\label{lem:data-adaptive-truncation}
\end{mylemma}
\begin{proof}
The proof is inspired from \citet{shao2018almost}, with some changes. Fix any $t\in \mathbb{N}$. Let $u_{i}^T \in \Real^{1\times t}$, $i=1,\ldots,m_t$ denotes the $i$-th row of $\tilde{V}_t^{-1/2}\tilde{\Phi}_t^T$ where $\tilde{V}_t=\tilde{\Phi}_t^T\tilde{\Phi}_t+\lambda I_{m_t}$. Let $r_{i} = u_i^TY_t=\sum_{\tau =1}^{t}u_{i,\tau}y_\tau$ denotes the sum of weighted historical rewards in the $i$-th dimension of the feature space with the weight vector $u_i$ and $\hat{r}_{i}= \sum_{\tau =1}^{t}u_{i,\tau}y_\tau \mathds{1}_{\abs{u_{i,\tau}y_\tau} \le b_t}$ denotes the corresponding truncation. Let $\cF'_{t,\tau} = \sigma(\lbrace x_1,\ldots,x_t \rbrace\cup\lbrace y_1,\ldots,y_\tau \rbrace),\tau=0,1,2,\ldots,t$ denotes the $\sigma$-algebra generated by the arms played up to time $t$ and rewards obtained up to time $\tau$. Observe that $\cF'_{t,0} \subseteq \cF'_{t,1} \subseteq \cF'_{t,2} \subseteq \ldots$ and define $\cF'_t=\cF'_{t,0}$. Then, $\expect{Y_t|\cF'_t}=f_t$ and $u_{i}, i=1,\ldots,m_t$ are $\cF'_{t}$-measurable. Therefore, we have $\expect{r_{i}|\cF'_t}=u_i^Tf_t=\sum_{\tau =1}^{t}u_{i,\tau}f(x_\tau) = \sum_{\tau =1}^{t}\expect{u_{i,\tau}y_\tau|\cF'_{t,\tau-1}}$ for all $i \in [m_t]$. This implies
\beqan
&&\abs{\hat{r}_{i} - \expect{r_{i}|\cF'_t}}\\ 
&=&\abs{\sum_{\tau =1}^{t}u_{i,\tau}y_\tau \mathds{1}_{\abs{u_{i,\tau}y_\tau} \le b_t} - \sum_{\tau =1}^{t}\expect{u_{i,\tau}y_\tau|\cF'_{t,\tau-1}}}\\
&=& \abs{\sum_{\tau =1}^{t}u_{i,\tau}y_\tau \mathds{1}_{\abs{u_{i,\tau}y_\tau} \le b_t} - \sum_{\tau =1}^{t}\expect{u_{i,\tau}y_\tau\left(\mathds{1}_{\abs{u_{i,\tau}y_\tau} \le b_t}+\mathds{1}_{\abs{u_{i,\tau}y_\tau} > b_t} \right)|\cF'_{t,\tau-1}}}\\
& \le & \abs{\sum_{\tau =1}^{t}\left(u_{i,\tau}y_\tau \mathds{1}_{\abs{u_{i,\tau}y_\tau} \le b_t} - \expect{u_{i,\tau}y_\tau \mathds{1}_{\abs{u_{i,\tau}y_\tau} \le b_t}|\cF'_{t,\tau-1}}\right)} + \sum_{\tau =1}^{t}\expect{\abs{u_{i,\tau}y_\tau }\mathds{1}_{\abs{u_{i,\tau}y_\tau} > b_t}|\cF'_{t,\tau-1}}.
\eeqan
Now, we will bound the second term first. Observe that $\expect{\abs{u_{i,\tau}y_\tau} \mathds{1}_{\abs{u_{i,\tau}y_\tau} > b_t}|\cF'_{t,\tau-1}} \le b_t^{-\alpha}\expect{\abs{u_{i,\tau}y_\tau}^{1+\alpha} \mathds{1}_{\abs{u_{i,\tau}y_\tau} > b_t}|\cF'_{t,\tau-1}} \le b_t^{-\alpha}\abs{u_{i,\tau}}^{1+\alpha}\expect{\abs{y_\tau}^{1+\alpha}|\cF'_{t,\tau-1}}$. Now since the noise variables are sampled independent of the arms played, it holds that $\expect{\abs{y_\tau}^{1+\alpha}|\cF'_{t,\tau-1}}=\expect{\abs{y_\tau}^{1+\alpha}|\cF_{\tau-1}}$ and therefore
\beqn
 \sum_{\tau=1}^{t}\expect{\abs{u_{i,\tau}y_\tau} \mathds{1}_{\abs{u_{i,\tau}y_\tau} > b_t}|\cF'_{t,\tau-1}}\le v b_t^{-\alpha}\sum_{\tau=1}^{t}\abs{u_{i,\tau}}^{1+\alpha}. 
\eeqn
Now, we will bound the first term. For that, we define $M_{t,\tau} \bydef u_{i,\tau}y_\tau \mathds{1}_{\abs{u_{i,\tau}y_\tau} \le b_t}-\expect{u_{i,\tau}y_\tau \mathds{1}_{\abs{u_{i,\tau}y_\tau} \le b_t}\given \cF'_{t,\tau -1}},\tau=1,2,\ldots,t$. It is easy to see that $(M_{t,\tau})_{\tau \ge 1}$ is a martingale difference sequence with respect to the filtration $(\cF'_{t,\tau})_{\tau \ge 0}$ and $\abs{M_{t,\tau}} \le 2b_t$ almost surely. Further, $\Var[M_\tau\given \cF'_{t,\tau -1}]= \Var[u_{i,\tau}y_\tau \mathds{1}_{\abs{u_{i,\tau}y_\tau}\le b_t}\given \cF'_{t,\tau -1} ] \le \expect{u_{i,\tau}^2y_\tau^2 \mathds{1}_{\abs{u_{i,\tau}y_\tau}\le b_t}\given \cF'_{t,\tau -1}} \le b_t^{1-\alpha}\abs{u_{i,\tau}}^{1+\alpha}\expect{\abs{y_\tau}^{1+\alpha}\given \cF'_{t,\tau -1}} \le v b_t^{1-\alpha}\abs{u_{i,\tau}}^{1+\alpha}$. Then by Bernstein's inequality \cite{seldin2012pac}, we have that for any $\gamma \in [0,1/2b_t]$ and $\delta \in (0,1]$, with probability at least $1-\delta$,
\beqn
\abs{\sum_{\tau =1}^{t}\left(u_{i,\tau}y_\tau \mathds{1}_{\abs{u_{i,\tau}y_\tau} \le b_t} - \expect{u_{i,\tau}y_\tau \mathds{1}_{\abs{u_{i,\tau}y_\tau} \le b_t}}\right)} \le \dfrac{1}{\gamma} \ln(2/\delta) + \gamma (e-2)\sum_{\tau =1}^{t}v b_t^{1-\alpha}\abs{u_{i,\tau}}^{1+\alpha}.
\eeqn
Now setting $\gamma = 1/2b_t$, we obtain that for any $i \in [m_t]$ and $\delta \in (0,1]$, with probability at least $1-\delta$,
\beqa
\abs{\hat{r}_{i}-\expect{r_{i}|\cF'_t}}
 &\le & 2b_t \ln(2/\delta)+2v b_t^{-\alpha} \sum_{\tau=1}^{t}\abs{u_{i,\tau}}^{1+\alpha}\nonumber\\
&= & 2b_t \ln(2/\delta)+2v b_t^{-\alpha} \norm{u_i}_{1+\alpha}^{1+\alpha}\nonumber\\
& \stackrel{(a)}{\le}& 2b_t \ln(2/\delta)+2v b_t^{-\alpha}t^{\frac{1-\alpha}{2}}\nonumber\\
&\stackrel{(b)}{\le}& 4v^{\frac{1}{1+\alpha}}\left(\ln(2/\delta)\right)^{\frac{\alpha}{1+\alpha}}t^{\frac{1-\alpha}{2(1+\alpha)}} \label{eqn:proj-dev}.
\eeqa
Here $(a)$ follows from Lemma \ref{lem:col-norm} and $(b)$ holds for $b_t = \left(v/\ln(2/\delta)\right)^{\frac{1}{1+\alpha}}t^{\frac{1-\alpha}{2(1+\alpha)}}$. 
Now observe that $\tilde{V}_t^{1/2}\tilde{\theta}_t=[\hat{r}_{1},\ldots,\hat{r}_{m_t}]^T$ and $\tilde{V}_t^{-1/2}\tilde{\Phi}_t^T f_t=[u_1^Tf_t,\ldots,u_{m_t}^Tf_t]^T=\left[\expect{r_{1}|\cF'_t},\ldots,\expect{r_{m_t}|\cF'_t}\right]^T$. This implies
\beqn
\norm{\tilde{V}_t^{-1}\tilde{\Phi}_t^T f_t-\tilde{\theta}_t}_{\tilde{V}_t}=\norm{\tilde{V}_t^{-1/2}\tilde{\Phi}_t^T f_t-\tilde{V}_t^{1/2}\tilde{\theta}_t}_2 = \sqrt{\sum_{i=1}^{m_t}\left(\hat{r}_{i}-\expect{r_{i}|\cF'_{t-1}}\right)^2}.
\eeqn
Therefore, by taking an union bound over all $i\in [m_t]$ and setting $\delta=\delta/m_t$ in \ref{eqn:proj-dev}, we obtain that for any $t \in \mathbb{N}$ and $\delta \in (0,1]$, with probability at least $1-\delta$,
\beqn
\norm{\tilde{V}_t^{-1}\tilde{\Phi}_t^T f_t-\tilde{\theta}_t}_{\tilde{V}_t} \le 4 \sqrt{m_t}\; v^{\frac{1}{1+\alpha}}\left(\ln(2m_t/\delta)\right)^{\frac{\alpha}{1+\alpha}}t^{\frac{1-\alpha}{2(1+\alpha)}}.
\label{eqn:truncation-bound}
\eeqn
Now the result follows by taking another union bound over all $t \in [T]$ and setting $\delta = \delta/T$.
\end{proof}

\subsection{Analysis of ATA-GP-UCB under quadrature Fourier features (QFF) approximation}
\subsubsection{Error due to Fourier feature approximation}
\begin{mydefinition}[Uniform Approximation \cite{mutny2018efficient}]
Let $k:\cX \times \cX \ra \Real, \cX \subset \Real^d$ be a kernel, then a feature map $\tilde{\phi}:\cX \ra \Real^m$ uniformly approximates $k$ within an accuracy $\epsilon_m$ if and only if,
\beq
\sup\limits_{x,y \in \cX}\abs{k(x,y)-\tilde{\phi}(x)^T\tilde{\phi}(y)} \le \epsilon_m.
\label{eqn:uniform-approx}
\eeq
\end{mydefinition}

\begin{mylemma}[QFF error]\cite[Theorem 1]{mutny2018efficient}
Let $\cX=[0,1]^d$, $k=k_{\text{SE}}$ and $\tilde{\phi}$ be as in \ref{eqn:qff-embedding}. Then, 
\beqn
\epsilon_m \le d2^{d-1}\frac{1}{\sqrt{2}\bar{m}^{\bar{m}}}\left(\frac{e}{4l^2}\right)^{\bar{m}}.
\eeqn
\label{lem:qff}
\end{mylemma}
Lemma \ref{lem:qff} implies that QFF embedding (\ref{eqn:qff-embedding}) of $k_{\text{SE}}$ satisfies $\epsilon_m =O\left(\frac{d2^{d-1}}{(\bar{m}l^2)^{\bar{m}}}\right)$ where $m=\bar{m}^d$. We can achieve exponential decay only when $\bar{m} > 1/l^2$, and in that case $O\left((d+\ln(d/\epsilon_m))^d\right)$ features are required to obtain an $\epsilon_m$-accurate approximation of the SE kernel. In contrast, \citet{sriperumbudur2015optimal} show that for any compact $\cX \subset \Real^d$, the uniform approximation error using RFF is $\epsilon_m=O_p(\sqrt{d\ln\abs{\cX}/m})$, i.e. at least $O(d\ln \abs{\cX}/\epsilon_m^2)$ features are required to obtain an $\epsilon_m$- accurate approximation of $k$. In most of the BO applications either $d=O(1)$, or there are enough structure (e.g. generalized additive models) such that effective dimensionality of the problem is low. In that case $O(1/\epsilon_m^2)$ and $O((\ln(1/\epsilon_m)^d)$ features are needed to obtain $\epsilon_m$-accuracy with RFF and QFF approximations, respectively.

Now, recall that the posterior mean and variance of a GP prior $GP_{\cX}(0,k)$ with iid Gaussian noise $\cN(0,\lambda)$ are given by $\mu_t(x)=k_t(x)^T(K_t+\lambda I_t)^{-1}Y_t$ and $\sigma_t^2(x)= k(x,x)-k_t(x)^T(K_t+\lambda I_t)^{-1}k_t(x)$, respectively. Let $\alpha_t(x)=k_t(x)^T(K_t+\lambda I_t)^{-1}f_t$ denotes the expected posterior mean and  $\tilde{\alpha}_t(x)=\tilde{k}_t(x)^T(\tilde{K}_t+\lambda I_t)^{-1}f_t$ denotes the approximation of $\alpha_t(x)$, where $\tilde{k}_t(x)=\tilde{\Phi}_t \tilde{\phi}(x)$ and $\tilde{K}_t = \tilde{\Phi}_t\tilde{\Phi}_t^T$. Define $\tilde{k}(x,y)=\tilde{\phi}(x)^T\tilde{\phi}(y)$. Then, the approximate posterior variance under QFF approximation is $\tilde{\sigma}_t^2(x)=\lambda \tilde{\phi_t}(x)^T \tilde{V}_t^{-1}\tilde{\phi_t}(x)= \tilde{k}(x,x)-\tilde{k}_t(x)^T(\tilde{K}_t+\lambda I_t)^{-1}\tilde{k}_t(x)$. Now, we will show that the error introduced by uniform approximation reflects in the approximation of the posterior variance and the expected posterior mean.
\begin{mylemma}[Error in posterior mean and variance approximations]
\label{lem:approx-error}
Let $f \in \cH_k(\cX)$, $\norm{f}_\cH \le B$ and  $k(x,x) \le 1$ for all $x \in \cX$. Let  $\tilde{\phi}:\cX \ra \Real^m$ be a feature map such that \ref{eqn:uniform-approx} holds for some $\epsilon_m < 1$, and $\tilde{\phi}(x)^T\tilde{\phi}(y) \le 1$ for all $x,y \in \cX$. Then for all $x \in \cX$ and $t \ge 1$, we have
\beqn
(i)\quad\abs{\alpha_t(x)-\tilde{\alpha}_t(x)} =O(B\epsilon_m t^2/\lambda) \quad \text{and} \quad (ii)\quad\abs{\sigma_t(x) - \tilde{\sigma}_t(x)} =   O (\epsilon_m^{1/2} t/\lambda).
\eeqn
\end{mylemma} 
\begin{proof}
This proof is inspired from \cite{mutny2018efficient}, with some notable changes. First, observe that
\beqan
&& \abs{k_t(x)^T(K_t+\lambda I_t)^{-1}f_t - \tilde{k}_t(x)^T(\tilde{K}_t+\lambda I_t)^{-1}f_t}\\
& \stackrel{(a)}{\le} & \abs{\left(k_t(x)-\tilde{k}_t(x)\right)^T(K_t+\lambda I_t)^{-1}f_t} + \abs{\tilde{k}_t(x)^T\left((K_t+\lambda I_t)^{-1}-(\tilde{K}_t+\lambda I_t)^{-1}\right)f_t}\\
& \stackrel{(b)}{\le} & \norm{k_t(x)-\tilde{k}_t(x)}_2 \norm{(K_t+\lambda I_t)^{-1}f_t}_2 + \norm{\tilde{k}_t(x)}_2 \norm{\left((K_t+\lambda I_t)^{-1}-(\tilde{K}_t+\lambda I_t)^{-1}\right)f_t}_2\\
& \stackrel{(c)}{\le} & \norm{k_t(x)-\tilde{k}_t(x)}_2 \norm{(K_t+\lambda I_t)^{-1}}_2 \norm{f_t}_2 + \norm{\tilde{k}_t(x)}_2 \norm{(K_t+\lambda I_t)^{-1}-(\tilde{K}_t+\lambda I_t)^{-1}}_2 \norm{f_t}_2,
\eeqan
where $(a)$ uses triangle inequality, $(b)$ uses Cauchy-Schwartz inequality and $(c)$ uses the definition of operator norm. By our hypothesis, $\norm{f_t}_2 \le B t^{1/2}$, $\norm{\tilde{k}_t(x)}_2 \le t^{1/2}$ and $\norm{k_t(x)-\tilde{k}_t(x)}_2 \le \epsilon_m t ^{1/2}$. Now
\beqan
\norm{(K_t+\lambda I_t)^{-1}-(\tilde{K}_t+\lambda I_t)^{-1}}_2&=&\norm{(K_t+\lambda I_t)^{-1}\left((\tilde{K}_t +\lambda I_t )-(K_t +\lambda I_t )\right)(\tilde{K}_t +\lambda I_t )^{-1}}_2 \\&=& \norm{(K_t+\lambda I_t)^{-1}(\tilde{K}_t-K_t)(\tilde{K}_t +\lambda I_t )^{-1}}_2\\
&\stackrel{(a)}{\le} &\norm{(K_t+\lambda I_t)^{-1}}_2\norm{\tilde{K}_t-K_t}_2\norm{(\tilde{K}_t +\lambda I_t )^{-1}}_2\\
& \stackrel{(b)}{\le} &  \epsilon_m t /\lambda^2,
\eeqan
where $(a)$ follows from the sub-multiplicative property of operator norm and $(b)$ follows from the facts that $\norm{K_t-\tilde{K_t}}_2 \le \sqrt{\sum_{1 \le i,j \le t}(k(x_i,x_j)-\tilde{k}(x_i,x_j))^2} \le \epsilon_m t$, and that for any p.s.d. matrix $A \in \Real^{t\times t}$, $\norm{(A+\lambda I_t)^{-1}}_2=\lambda_{\max}\lbrace(A+\lambda I_t)^{-1} \rbrace = 1/\lambda_{\min}\lbrace A+\lambda I_t \rbrace \le 1/\lambda$.
Therefore, for all $x \in \cX$ and $t \ge 1$, we have
\beqn
\abs{\alpha_t(x)-\tilde{\alpha}_t(x)} \le \left(\epsilon_m t^{1/2}/\lambda + \epsilon_m t^{3/2}/\lambda^2\right)B t^{1/2} = O(B\epsilon_m t^2/\lambda).
\label{eqn:mean-error}
\eeqn

Now, since $\abs{k(x,y)-\tilde{k}(x,y)} \le \eps_m$ for all $x,y \in \cX$, we have $\tilde{k}_t(x)=k_t(x)+a_t(x)$ where $\norm{a_t(x)}_{\infty} \le \epsilon_m$. This implies
\beqan
 && \abs{\sigma_t^2(x) -\tilde{\sigma}_t^2(x)}\\ &=& \abs{k(x,y)-\tilde{k}(x,y)}+ \abs{\tilde{k}_t(x)^T(\tilde{K}_t+\lambda I_t)^{-1}\tilde{k}_t(x) - k_t(x)^T(K_t+\lambda I_t)^{-1}k_t(x)}\\
&\le & \epsilon_m + \abs{k_t(x)^T\left((\tilde{K}_t+\lambda I_t)^{-1}-(K_t+\lambda I_t)^{-1}\right)k_t(x)} + 2 \abs{a_t(x)^T(\tilde{K}_t+\lambda I_t)^{-1}k_t(x)}\\ && \hspace*{85mm}+\abs{a_t(x)^T (\tilde{K}_t+\lambda I_t)^{-1} a_t(x)}\\
& \stackrel{(a)}{\le}& \epsilon_m+\norm{(\tilde{K}_t+\lambda I_t)^{-1}-(K_t+\lambda I_t)^{-1}}_2 \norm{k_t(x)}_2^2 + 2 \norm{a_t(x)}_2 \norm{(\tilde{K}_t+\lambda I_t)^{-1}}_2 \norm{k_t(x)}_2 \\  && \hspace*{85mm}+ \norm{(\tilde{K}_t+\lambda I_t)^{-1}}_2 \norm{a_t(x)}_2^2\\
& \stackrel{(b)}{\le}& \epsilon_m+ \epsilon_m t^2/\lambda^2 + 2\epsilon_m t/\lambda + \epsilon_m^2 t/\lambda
 = O (\epsilon_m t^2/\lambda^2) \;\text{for}\; \epsilon_m < 1.
\eeqan
Here $(a)$ is due to Cauchy-Schwartz inequality and definition of operator norm. $(b)$ uses $\norm{k_t(x)}_2 \le t^{1/2}$, $\norm{a_t(x)}_2 \le \epsilon_mt^{1/2}$, $\norm{(\tilde{K}_t+\lambda I_t)^{-1}-(K_t+\lambda I_t)^{-1}}_2 \le \epsilon_m t/\lambda^2$ and $\norm{(\tilde{K}_t+\lambda I_t)^{-1}}_2 \le 1/\lambda$. Now, the result follows from the fact that for any $a,b \ge 0$, $(a+b)^{1/2} \le a^{1/2}+b^{1/2}$.
\end{proof}
Now, we are ready to prove Lemma \ref{lem:func-conc-ata-gp-ucb}.
\subsubsection{Proof of Lemma \ref{lem:func-conc-ata-gp-ucb}}
Under the QFF approximation, we have $\tilde{\phi}_t=\tilde{\phi}$ and $m_t =m$ for all $t \ge 1$. Hence, we have $\tilde{\mu}_t(x)=\tilde{\phi}(x)^T\tilde{\theta}_t$ and $\tilde{\alpha}_t(x)=\tilde{\phi}(x)^T\tilde{\Phi}_t^T(\tilde{\Phi}_t\tilde{\Phi}_t^T+\lambda I_t)^{-1}f_t=\tilde{\phi}(x)^T\tilde{V}_t^{-1}\tilde{\Phi}_t^T f_t$, where the last equality follows from \ref{eqn:dim-change-1}. Now, by Cauchy-Schwartz inequality,
\beqn
\abs{\tilde{\alpha}_t(x)-\tilde{\mu}_t(x)}  \le   \norm{\tilde{V}_t^{-1}\tilde{\Phi}^T f_t-\tilde{\theta}_t}_{\tilde{V}_t} \norm{\tilde{\phi}(x)}_{\tilde{V}_t^{-1}} = \lambda^{-1/2}\norm{\tilde{V}_t^{-1}\tilde{\Phi}_t^T f_t-\tilde{\theta}_t}_{\tilde{V}_t}\tilde{\sigma}_t(x).
\eeqn
Hence, from Lemma \ref{lem:data-adaptive-truncation}, we have, for any $\delta \in (0,1]$, with probability at least $1-\delta$, uniformly over all $x \in \cX$ and $t \in [T]$, that
\beq
\abs{\tilde{\alpha}_t(x)-\tilde{\mu}_t(x)}  \le 4 \sqrt{m/\lambda}\; v^{\frac{1}{1+\alpha}}\left(\ln(2mT/\delta)\right)^{\frac{\alpha}{1+\alpha}}t^{\frac{1-\alpha}{2(1+\alpha)}}\tilde{\sigma}_t(x).
\label{eqn:approx-error}
\eeq
By triangle inequality,
\beqn
\abs{f(x)-\tilde{\mu}_t(x)} \le \abs{f(x)-\alpha_t(x)} +\abs{\alpha_t(x)-\tilde{\alpha}_t(x)} + \abs{\tilde{\alpha}_t(x)-\tilde{\mu}_t(x)}.
\eeqn
Now, from \ref{eqn:func-error}, $\abs{f(x)-\alpha_t(x)} \le B \sigma_t(x)$ and thus, in turn, from Lemma \ref{lem:approx-error}, $\abs{f(x)-\alpha_t(x)} = B \tilde{\sigma}_t(x)+O(B\epsilon_m^{1/2} t/\lambda)$. Also, from Lemma \ref{lem:approx-error}, $\abs{\alpha_t(x)-\tilde{\alpha}_t(x)} =O(B\epsilon_m t^2/\lambda)$. Now combining these with 
\ref{eqn:approx-error}, we obtain, for any $\delta \in (0,1]$, with probability at least $1-\delta$, uniformly over all $x \in \cX$ and $t \in [T]$, that
\beqan
\abs{f(x)-\tilde{\mu}_t(x)} 
&\le &\left(B+4 \sqrt{m/\lambda}\; v^{\frac{1}{1+\alpha}}\left(\ln(2mT/\delta)\right)^{\frac{\alpha}{1+\alpha}}t^{\frac{1-\alpha}{2(1+\alpha)}}\right)\tilde{\sigma}_t(x) + O(B\epsilon_m^{1/2} t/\lambda)+  O(B\epsilon_m t^2/\lambda)\\
&= &\left(B+4 \sqrt{m/\lambda}\; v^{\frac{1}{1+\alpha}}\left(\ln(2mT/\delta)\right)^{\frac{\alpha}{1+\alpha}}t^{\frac{1-\alpha}{2(1+\alpha)}}\right)\tilde{\sigma}_t(x) + O(B\epsilon_m^{1/2} t^2/\lambda)
\eeqan
for $\epsilon_m < 1$. 
Further observe that $\abs{f(x)-\tilde{\mu}_0(x)}=\abs{f(x)} \le Bk^{1/2}(x,x) = B\sigma_0(x) \le B\tilde{\sigma}_0(x)+B\epsilon_m^{1/2}$.
Now the result follows by setting $\beta_{t+1} = B+4 \sqrt{m/\lambda}\; v^{\frac{1}{1+\alpha}}\left(\ln(2m T/\delta)\right)^{\frac{\alpha}{1+\alpha}}t^{\frac{1-\alpha}{2(1+\alpha)}}$ for all $t \ge 0$.
\subsubsection{Proof of Theorem \ref{thm:regret-bound-qff}}
For any $\delta \in (0,1]$, we have, with probability at least $1-\delta$, uniformly over all $t \in [T]$, the instantaneous regret 
\beqan
r_t &=& f(x^\star)-f(x_t)\\
& \stackrel{(a)}{\le} & \tilde{\mu}_{t-1}(x^\star) +  \beta_{t}\tilde{\sigma}_{t-1}(x^\star) + O(B\epsilon_m^{1/2} t^2/\lambda)- f(x_t)\\
& \stackrel{(b)}{\le} & \tilde{\mu}_{t-1}(x_t) +  \beta_{t}\tilde{\sigma}_{t-1}(x_t) -f(x_t) + O(B\epsilon_m^{1/2} t^2/\lambda)\\
& \stackrel{(c)}{\le} & 2\beta_{t}\tilde{\sigma}_{t-1}(x_t) + O(B\epsilon_m^{1/2} t^2/\lambda)).
\eeqan
Here $(a)$ and $(c)$ follow from Lemma \ref{lem:func-conc-ata-gp-ucb} and $(b)$ is due to the choice of ATA-GP-UCB (Algorithm \ref{algo:itgp-ucb}). Now Observe that $(\beta_t)_{t \ge 1}$ is an increasing sequence in $t$. Further,
\beqn
\sum_{t=1}^{T}\tilde{\sigma}_{t-1}(x_t) \stackrel{(a)}{\le} \sqrt{T\sum_{t=1}^{T}\tilde{\sigma}_{t-1}^2(x_t)} \stackrel{(b)}{\le} \sqrt{2(1+\lambda) T\tilde{\gamma}_T}=O(\sqrt{mT\ln T}).
\eeqn
Here $(a)$ follows from Cauchy-Schwartz inequality, $(b)$ from Lemma \ref{lem:pred-var}, and $(c)$ from Lemma \ref{lem:info-gain-bound} noting that $\tilde{k}$ is a linear kernel defined on $\Real^{2m}$. Hence for any $\delta \in (0,1]$, with probability at least $1-\delta$, the cumulative regret of ATA-GP-UCB after $T$ rounds is
\beqan
R_T &= & O\left(\beta_T \sqrt{Tm\ln T}\right) + \sum_{t=1}^{T} O(B\epsilon_m^{1/2} t^2/\lambda)\\
& = &  O\left(B \sqrt{T m \ln T} + m v^{\frac{1}{1+\alpha}}\left(\ln(mT/\delta)\right)^{\frac{\alpha}{1+\alpha}}(\ln T)^{1/2}T^{\frac{1}{1+\alpha}}+B\epsilon_m^{1/2} T^3\right).
\eeqan
For the QFF approximation, from Lemma \ref{lem:qff}, we have  $\epsilon_m = O((e/4)^{\bar{m}})$ if $\bar{m} > 1/l^2$ and $d=O(1)$. Now for $\bar{m}=2\log_{4/e}(T^3)$, we have $\epsilon_m^{1/2}T^3=O(1)$ and $m = O((\ln T)^d)$ \footnote{For the RFF approximation, we have $\epsilon_m = O_p(1/\sqrt{m})$ if $d=O(1)$. Now in order to make the last term $\epsilon_m^{1/2}T^3$ behave as $O(1)$, we have to take $m=O(T^{12})$ features which will eventually blow up the first two terms by the same order. Hence, we will never achieve sub-linear regret bound using RFF approximation.}. Therefore for any $\delta \in (0,1]$, with probability at least $1-\delta$, the cumulative regret of ATA-GP-UCB under QFF approximation after $T$ rounds is
\beqn
R_T = O\left(B\sqrt{T(\ln T)^{d+1}}+v^{\frac{1}{1+\alpha}}\left(\ln\left(T(\ln T)^d/\delta\right)\right)^{\frac{\alpha}{1+\alpha}}\sqrt{\ln T}(\ln T)^{d}T^{\frac{1}{1+\alpha}}\right).
\eeqn
\subsection{Analysis of ATA-GP-UCB under Nystr\"{o}m approximation}
%
%
\subsubsection{Construction of dictionary and its properties} 

Given the kernel matrix $K_t$, we define an accurate dictionary as follows.

\begin{mydefinition}[$\epsilon$-accurate dictionary \cite{calandriello2019gaussian}]
For any $\epsilon \in (0,1)$, a dictionary $\cD_t \subseteq \lbrace x_1,\ldots,x_t \rbrace$ is said to be $\epsilon$-accurate with respect to the kernel matrix $K_t$ if 
\beqn
\norm{(K_t+\lambda I)^{-1/2}K_t^{1/2}(I_t-S_t^2)K_t^{1/2}(K_t+\lambda I)^{-1/2}}_2 \le \epsilon, 
\eeqn
where $S_t$ is the selection matrix associated with the dictionary $\cD_t$ such that $[S_t]_{i,i}=1/\sqrt{p_{t,i}}$ if $x_i \in \cD_t$, and $0$, elsewhere.
\label{def:dictionary}
\end{mydefinition}
The following lemma states two more equivalent condition for a dictionary to be accurate.
\begin{mylemma}
Let $V_{\cD_t}=\Phi_t^TS_t^2\Phi_t + \lambda I_\cH$. Then, the following are equivalent:
\begin{enumerate}
\item $\norm{(K_t+\lambda I)^{-1/2}K_t^{1/2}(I_t-S_t^2)K_t^{1/2}(K_t+\lambda I)^{-1/2}}_2 \le \epsilon$,
\item $\norm{(\Phi_t^T\Phi_t+\lambda I_\cH)^{-1/2}\Phi_t^T(I_t-S_t^2)\Phi_t(\Phi_t^T\Phi_t+\lambda I_\cH)^{-1/2}}_\cH \le \epsilon$,
\item $(1-\epsilon)V_t \preceq V_{\cD_t} \preceq (1+\epsilon)V_t$.
\end{enumerate}
\label{lem:equivalence}
\end{mylemma}
\begin{proof}
Let $\Phi_t = U\Sigma V^T$ be the singular value decomposition of $\Phi_t$. Then $\Phi_t(\Phi_t^T\Phi_t+\lambda I_\cH)^{-1/2}=U\Sigma(\Sigma^T\Sigma+\lambda I_\cH)^{-1}V^T$, $(\Phi_t^T\Phi_t+\lambda I_\cH)^{-1/2}\Phi_t^T = V(\Sigma^T\Sigma+\lambda I_\cH)^{-1}\Sigma^TU^T$ and $K_t = U \Sigma \Sigma^T U^T$. Therefore
\beqan
&&\norm{(\Phi_t^T\Phi_t+\lambda I_\cH)^{-1/2}\Phi_t^T(I_t-S_t^2)\Phi_t(\Phi_t^T\Phi_t+\lambda I_\cH)^{-1/2}}_\cH \\
&=&
\norm{V(\Sigma^T\Sigma+\lambda I_\cH)^{-1/2}\Sigma^TU^T(I_t-S_t^2)U\Sigma(\Sigma^T\Sigma+\lambda I_\cH)^{-1/2}V^T}_\cH \\
&=& \norm{(\Sigma^T\Sigma+\lambda I_\cH)^{-1/2}\Sigma^TU^T(I_t-S_t^2)U\Sigma(\Sigma^T\Sigma+\lambda I_\cH)^{-1/2}}_\cH\\
 &=& \norm{(\Sigma\Sigma^T+\lambda I_t)^{-1/2}(\Sigma\Sigma^T)^{1/2}U^T(I_t-S_t^2)U(\Sigma\Sigma^T)^{1/2}(\Sigma \Sigma^T+\lambda I_t)^{-1/2}}_2 \\
&=&\norm{U(\Sigma\Sigma^T+\lambda I_t)^{-1/2}(\Sigma\Sigma^T)^{1/2}U^T(I_t-S_t^2)U(\Sigma\Sigma^T)^{1/2}(\Sigma \Sigma^T+\lambda I_t)^{-1/2}U^T}_2\\
&=& \norm{(K_t+\lambda I)^{-1/2}K_t^{1/2}(I_t-S_t^2)K_t^{1/2}(K_t+\lambda I)^{-1/2}}_2,
\eeqan 
which proves that 1 $\Longleftrightarrow$ 2. Now, Observe that 
\beqan
&&\norm{(\Phi_t^T\Phi_t+\lambda I_\cH)^{-1/2}\Phi_t^T(I_t-S_t^2)\Phi_t(\Phi_t^T\Phi_t+\lambda I_\cH)^{-1/2}}_\cH \le \epsilon \\ 
&\Longleftrightarrow& -\epsilon I_\cH \preceq (\Phi_t^T\Phi_t+\lambda I_\cH)^{-1/2} (\Phi_t^T\Phi_t-\Phi_t^TS_t^2\Phi_t)(\Phi_t^T\Phi_t+\lambda I_\cH)^{-1/2} \preceq \epsilon I_\cH\\
&\Longleftrightarrow& -\epsilon I_\cH \preceq V_t^{-1/2}(V_t-V_{\cD_t})V_t^{-1/2} \preceq \epsilon I_\cH\\
&\Longleftrightarrow& -\epsilon V_t \preceq V_t-V_{\cD_t} \preceq \epsilon V_t\\
&\Longleftrightarrow& (1-\epsilon)V_t \preceq V_{\cD_t} \preceq (1+\epsilon)V_t,
\eeqan
which proves 2 $\Longleftrightarrow$ 3.
\end{proof}
An $\epsilon$-accurate dictionary can be obtained by including points proportional to their $\lambda$-ridge leverage scores defined as follows.
\begin{mydefinition}[Ridge leverage score \cite{alaoui2015fast}] 
For a set of points $\lbrace x_1,\ldots,x_t \rbrace$ and a constant $\lambda > 0$, the $\lambda$- ridge leverage score of the point $x_i,i \in [t]$ is defined as 
\beqn
l_{t,i}=e_i^TK_t(K_t+\lambda I_t)^{-1}e_i,
\eeqn
where $e_i \in \Real^t$ is the $i$-th standard basis vector.
\end{mydefinition}
Ridge leverage score (RLS) can be interpreted in many ways and it is well studied in the literature. Here we observe that 
\beqn
e_i^TK_t(K_t+\lambda I_t)^{-1}e_i=e_i^T\Phi_t\Phi_t^T(\Phi_t\Phi_t^T+\lambda I_t)^{-1}e_i=e_i^T\Phi_t(\Phi_t^T\Phi_t+\lambda I_\cH)^{-1}\Phi_t^Te_i=\norm{\phi(x_i)}^2_{V_t^{-1}}.
\eeqn
Therefore $l_{t,i} = \frac{1}{\lambda}\sigma_t^2(x_i)$, i.e., the RLS of $x_i$ is proportional its posterior variance $\sigma_t^2(x_i)$ under the GP prior $GP_{\cX}(0,k)$.  However, the exact computation of $\lambda$-ridge leverage scores in turn requires inverting the kernel matrix $K_t$ which requires $O(t^3)$ time. This motivates the need for a fast approximation of RLS such that it can be used to construct an $\epsilon$-accurate dictionary. \citet{calandriello2019gaussian} show that, instead of using the exact ridge leverage scores (or, equivalently, posterior variances) if we use the approximate variances from the previous round to sample points in the current round, then we will be able to obtain an accurate dictionary. Not only that, the dictionary size will grow no faster than the maximum information gain of the underlying kernel. Now, we present the Nystr\"{o}mEmbedding procedure which is used in Algorithm \ref{algo:itgp-ucb}.
 \begin{algorithm}[H]
        \renewcommand\thealgorithm{3}
        \caption{ Nystr\"{o}mEmbedding\label{subroutine:Nystrom}}

         \begin{algorithmic}
         \STATE \textbf{Input:} $\{(x_i,\tilde{\sigma}_{t-1}(x_i))\}_{i=1}^{t}$, $q$
         \STATE \textbf{Set:} $\cD_t = \emptyset$
         \FOR{$i = 1, 2, 3 \ldots, t$}
            \STATE Sample $z_{t,i} \sim \cB \left(\min \lbrace q \tilde{\sigma}^2_{t-1}(x_i), 1 \rbrace\right)$
            \STATE If $z_{t,i}=1$, set $\cD_t = \cD_t \cup \lbrace x_i \rbrace$
            \ENDFOR
            \vspace*{-1mm}
            \STATE \textbf{Return} $\tilde{\phi}_t(x)=\left(K_{\cD_t}^{1/2}\right)^{+}k_{\cD_t}(x)$
         \end{algorithmic}         \addtocounter{algorithm}{-3}
         \end{algorithm}

 The following lemma states the properties of the dictionaries $\cD_t$ constructed using Algorithm \ref{subroutine:Nystrom}.

 \begin{mylemma}[Properties of the dictionary]
For any $\epsilon \in (0,1)$ and $\delta \in (0,1]$, set $\rho=\frac{1+\epsilon}{1-\epsilon}$ and $q=\frac{6\rho\ln(2T/\delta)}{\epsilon^2}$. Then, with probability at least $1-\delta$, uniformly over all $t \in [T]$,
\beqn
(1-\epsilon)V_t \preceq V_{\cD_t} \preceq (1+\epsilon)V_t \quad \text{and} \quad m_t \le 6\rho \left(1+\frac{1}{\lambda}\right) \; q \gamma_t.
\eeqn
\label{lem:dictionary-prop}
\end{mylemma}
Lemma \ref{lem:dictionary-prop} is a restatement of \cite[Theorem 1]{calandriello2019gaussian} and it is presented in this form for the sake of brevity and completeness. 
Now, we will show that using the Nystr\"{o}m embeddings $\tilde{\phi}_t(x)$, we can prevent the variance starvation which generally arises due to approximation.
\subsubsection{Preventing variance starvation with Nystr\"{o}m embeddings}
Recall that the posterior mean and variance of a GP prior $GP_{\cX}(0,k)$ with iid Gaussian noise $\cN(0,\lambda)$ are given by $\mu_t(x)=k_t(x)^T(K_t+\lambda I_t)^{-1}Y_t$ and $\sigma_t^2(x)= k(x,x)-k_t(x)^T(K_t+\lambda I_t)^{-1}k_t(x)$, respectively. Let $\alpha_t(x)=k_t(x)^T(K_t+\lambda I_t)^{-1}f_t$ denotes the expected posterior mean and $\tilde{\alpha}_t(x)=\tilde{k}_t(x)^T(\tilde{K}_t+\lambda I_t)^{-1}f_t$ denotes the approximation of $\alpha_t(x)$, where $\tilde{k}_t(x)=\tilde{\Phi}_t \tilde{\phi}(x)$ and $\tilde{K}_t = \tilde{\Phi}_t\tilde{\Phi}_t^T$. Then, we have $\alpha_t(x)=\inner{\phi(x)}{V_t^{-1}\Phi_t^T f_t}_\cH$ and $\tilde{\alpha}_t(x)=\tilde{\phi}_t(x)^T\tilde{V}_t^{-1}\tilde{\Phi}_t^T f_t$. Now, we can rewrite the posterior variance as $\sigma_t^2(x)=\lambda \norm{\phi(x)}^2_{V_t^{-1}}$, whereas the approximate posterior variance under Nystr\"{o}m approximation is given by $\tilde{\sigma}_t^2(x)=k(x,x)-\tilde{\phi_t}(x)^T\tilde{\phi_t}(x)+\lambda \tilde{\phi_t}(x)^T \tilde{V}_t^{-1}\tilde{\phi_t}(x)$. This choice of $\tilde{\sigma}_t^2(x)$ helps us to negate the variance starvation which arises due to feature approximation. Now, we will justify this choice of $\tilde{\sigma}_t^2(x)$ by showing that it can be derived by projecting $\phi(x)$ to a smaller RKHS. The idea is inspired from \citet{calandriello2019gaussian}.

\textbf{Projection to a smaller RKHS:} For any dictionary $\cD_t =\lbrace x_{i_1},\ldots,x_{i_{m_t}}\rbrace, i_j\in [t]$,
define the operator $\Phi_{\cD_t}: \cH_k(\cX) \ra \Real^{m_t}$ such that for any $h \in \cH_k(\cX)$, $\Phi_{\cD_t} h = \left[\inner{\phi(x_{i_1})}{h}_\cH,\ldots,\inner{\phi(x_{i_{m_t}})}{h}_\cH\right]^T$ and denote its adjoint by $\Phi_{\cD_t}^T : \Real^{m_t} \ra \cH_k(\cX)$. Let $\hat{\phi}_t(x)=P_t\phi(x)$ be the projection of $\phi(x)$ to the subspace spanned by the columns of the operator $\Phi_{\cD_t}^T$, where the projection operator $P_t:\cH_k(\cX) \ra \text{Col}(\Phi_{\cD_t}^T)$ is given by $P_t=\Phi_{\cD_t}^T(\Phi_{\cD_t}\Phi_{\cD_t}^T)^{+}\Phi_{\cD_t}$. It is easy to see that $P_t^T=P_t$ and $P_t^2=P_t$. Now, for any set $\lbrace x_1,\ldots,x_t \rbrace \subset \cX$ define the operator $\hat{\Phi}_t : \cH_k(\cX) \ra \Real^t$ such that for any $h \in \cH_k(\cX)$, $\hat{\Phi}_t h = [\inner{\hat{\phi}_t(x_1)}{h}_\cH,\ldots,\inner{\hat{\phi}_t(x_t)}{h}_\cH]^T$, and denote its adjoint by $\hat{\Phi}_t^T : \Real^t \ra \cH_k(\cX)$.  

\begin{mylemma}[Approximate posterior variance and mean under projection]
Let $\hat{V}_t = \hat{\Phi}_t^T\hat{\Phi}_t + \lambda I_\cH$ for any $\lambda > 0$. Then, we have
\beqn
\tilde{\sigma}^2_t(x)=\lambda \norm{\phi(x)}^2_{\hat{V}_t^{-1}} \quad \text{and} \quad \tilde{\alpha}_t(x)=\inner{\phi(x)}{\hat{V}_t^{-1}\hat{\Phi}_t^T f_t}_\cH.
\eeqn
\label{lem:projection-values}
\end{mylemma}
\begin{proof}
Since $K_{\cD_t}=\Phi_{\cD_t}\Phi_{\cD_t}^T$, we have the projection  $P_t=\Phi_{\cD_t}^T(K_{\cD_t})^{+}\Phi_{\cD_t}$. Now, observe that $\inner{\phi(x)}{\phi(y)}_{P_t}=\left((K_{\cD_t}^{1/2})^{\dagger}\Phi_{\cD_t}\phi(x)\right)^T\left((K_{\cD_t}^{1/2})^{\dagger}\Phi_{\cD_t}\phi(y)\right)=\left((K_{\cD_t}^{1/2})^{\dagger}k_{\cD_t}(x)\right)^T\left((K_{\cD_t}^{1/2})^{\dagger}k_{\cD_t}(y)\right)=\tilde{\phi}_t(x)^T\tilde{\phi}_t(y)$. Also, note that
$\hat{\Phi}_t^T=P_t\Phi_t^T$.
This implies $\hat{\Phi}_t\phi(x)=\Phi_tP_t\phi(x)=[\inner{\phi(x_1)}{\phi(x)}_{P_t},\ldots,\inner{\phi(x_t)}{\phi(x)}_{P_t}]^T=[\tilde{\phi}_t(x_1)^T\tilde{\phi}_t(x),\ldots,\tilde{\phi}_t(x_t)^T\tilde{\phi}_t(x)]^T=\tilde{\Phi}_t\tilde{\phi}_t(x)$. Further, the $(i,j)$-th entry of $\hat{\Phi}_t\hat{\Phi}_t^T$ is given by $[\hat{\Phi}_t\hat{\Phi}_t^T]_{i,j}=\inner{P_t\phi(x_i)}{P_t\phi(x_j)}_\cH=\inner{\phi(x_i)}{\phi(x_j)}_{P_t}=\tilde{\phi}_t(x_i)^T\tilde{\phi}_t(x_j)$ and hence, $\hat{\Phi}_t\hat{\Phi}_t^T=\tilde{\Phi}_t \tilde{\Phi}_t^T$.  
Then, we have 
\beqan
\lambda \norm{\phi(x)}^2_{\hat{V}_t^{-1}}&=& \lambda \inner{\phi(x)}{(\hat{\Phi}_t^T\hat{\Phi}_t + \lambda I_\cH)^{-1}\phi(x)}_\cH\\
&\stackrel{(a)}{=}& \inner{\phi(x)}{\left(I_\cH-\hat{\Phi}_t^T(\hat{\Phi}_t\hat{\Phi}_t^T+\lambda I_t)^{-1}\hat{\Phi}_t\right)\phi(x)}_\cH\\
& \stackrel{(b)}{=} & k(x,x)-\tilde{\phi}_t(x)^T\tilde{\Phi}_t^T(\tilde{\Phi}_t\tilde{\Phi}_t^T+\lambda I_t)^{-1}\tilde{\Phi}_t\tilde{\phi}_t(x)\\
&\stackrel{(c)}{=}& k(x,x)-\tilde{\phi}_t(x)^T\tilde{\Phi}_t^T\tilde{\Phi}_t(\tilde{\Phi}_t^T\tilde{\Phi}_t+\lambda I_{m_t})^{-1}\tilde{\phi}_t(x)\\
& = & k(x,x)-\tilde{\phi_t}(x)^T\tilde{\phi_t}(x)+\lambda \tilde{\phi_t}(x)^T \tilde{V}_t^{-1}\tilde{\phi_t}(x)=\tilde{\sigma}^2_t(x).
\eeqan
Here $(a)$ follows from \ref{eqn:dim-change-2}, $(b)$ is due to $\hat{\Phi}_t\phi(x)=\tilde{\Phi}_t\tilde{\phi}_t(x)$ and $\hat{\Phi}_t\hat{\Phi}_t^T=\tilde{\Phi}_t\tilde{\Phi}_t^T$, and $(c)$ follows from \ref{eqn:dim-change-1}.
Now observe that
\beqan
\inner{\phi(x)}{\hat{V}_t^{-1}\hat{\Phi}_t^T f_t}_\cH &=&\inner{\phi(x)}{(\hat{\Phi}_t^T\hat{\Phi}_t+\lambda I_\cH)^{-1}\hat{\Phi}_t^T f_t}_\cH \\ &\stackrel{(a)}{=}&\inner{\phi(x)}{\hat{\Phi}_t^T(\hat{\Phi}_t \hat{\Phi}_t^T+\lambda I_t)^{-1}f_t}_\cH\\&\stackrel{(b)}{=}&\tilde{\phi}_t(x)^T\tilde{\Phi}_t^T(\tilde{\Phi}_t\tilde{\Phi}_t^T+\lambda I_t)^{-1}f_t\\
&\stackrel{(c)}{=}&\tilde{\phi}_t(x)^T(\tilde{\Phi}_t^T\tilde{\Phi}_t+\lambda I_{m_t})^{-1}\tilde{\Phi}_t^Tf_t\\
&=&\tilde{\phi}_t(x)^T\tilde{V}_t^{-1}\tilde{\Phi}_t^Tf_t=\tilde{\alpha}_t(x).
\eeqan
Here $(a)$ and $(c)$ follow from \ref{eqn:dim-change-1}, and $(b)$ is due to $\hat{\Phi}_t\phi(x)=\tilde{\Phi}_t\tilde{\phi}_t(x)$ and $\hat{\Phi}_t\hat{\Phi}_t^T=\tilde{\Phi}_t\tilde{\Phi}_t^T$.
\end{proof}
\begin{mylemma}[Accuracy of approximate posterior variance]
For an $\epsilon$-accurate dictionary (Definition \ref{def:dictionary}), we have 
\beqn
\frac{1-\epsilon}{1+\epsilon} \sigma_t^2(x) \le \tilde{\sigma}_t^2(x) \le		 \frac{1+\epsilon}{1-\epsilon}\sigma_t^2(x).
\eeqn
\end{mylemma}

\begin{proof}
From Lemma \ref{lem:projection-values}, we have $\tilde{\sigma}_t^2(x) = \lambda \inner{\phi(x)}{{\hat{V}_t^{-1}}\phi(x)}_\cH$. Now, observe that $\hat{V}_t=P_t\Phi_t^T\Phi_tP_t+ \lambda I_\cH=P_tV_tP_t+\lambda (I_\cH-P_t)$. From Lemma \ref{lem:equivalence}, we have $(1-\epsilon)V_t \preceq V_{\cD_t} \preceq (1+\epsilon)V_t$ for any $\epsilon$-accurate dictionary $\cD_t$.
This implies that
\beqan
\hat{V}_t &\preceq & \frac{1}{1-\epsilon}P_tV_{\cD_t}P_t+\lambda (I_\cH-P_t)\\ &=& \frac{1}{1-\epsilon}P_t\Phi_t^TS_t^2\Phi_tP_t + \frac{\lambda \epsilon}{1-\epsilon} P_t + \lambda I_\cH\\
& \stackrel{(a)}{\preceq} & \frac{1}{1-\epsilon} (\Phi_t^TS_t^2\Phi_t +\lambda I_\cH)\\  &\preceq & \frac{1+\epsilon}{1-\epsilon}V_t,
\eeqan
where $(a)$ follows from $P_t\Phi_t^TS_t=\Phi_t^TS_t$ and $P_t \preceq I_\cH$. Therefore, we have
\beqn
\tilde{\sigma}_t^2(x) \ge \frac{1-\epsilon}{1+\epsilon}\lambda \inner{\phi(x)}{{V_t^{-1}}\phi(x)}_\cH=\frac{1-\epsilon}{1+\epsilon}\sigma_t^2(x).
\eeqn
Similarly, we can show that $\hat{V}_t \succeq  \frac{1-\epsilon}{1+\epsilon} V_t$ and thus, in turn, $\tilde{\sigma}_t^2(x) \le \frac{1+\epsilon}{1-\epsilon}\sigma_t^2(x)$.
\end{proof}
Now, we will show that the confidence sets formed by ATA-GP-UCB (Algorithm \ref{algo:itgp-ucb}) under Nystr\"{o}m approximation is tighter compared to that of TGP-UCB.
\subsubsection{Confidence sets of ATA-GP-UCB under Nystr\"{o}m approximation}
First, we define the following two events. Fix any $\epsilon \in (0,1)$ and $\delta \in (0,1]$. Let $E_{1,t}$ denotes the event that the dictionary $\cD_t$ is $\epsilon$-accurate, i.e,
\beqn
(1-\epsilon)V_t \preceq V_{\cD_t} \preceq (1+\epsilon)V_t,
\eeqn 
and $E_{2,t}$ denotes the event that the size of the dictionary $\cD_t$ is at most $6\rho (1+\frac{1}{\lambda})q\gamma_t$, i.e.,
\beqn
m_t \le 6\rho \left(1+\frac{1}{\lambda}\right) \; q \gamma_t,
\eeqn
where $\rho=\frac{1+\epsilon}{1-\epsilon}$ and  $q=\frac{6\rho\ln(2T/\delta)}{\epsilon^2}$.
Then from Lemma \ref{lem:dictionary-prop}, we have $\prob{\cap_{t=1}^T (E_{1,t}\cap E_{2,t})} \ge 1-\delta$. 
Let $\cG_t=\sigma \left(\lbrace x_i, ( z_{i,j} )_{j=1}^i \rbrace_{i=1}^t\right), t\ge 1$ denotes the $\sigma$-algebra generated by the arms played and the outcomes of the Nystr\"{o}mEmbedding procedure(Algorithm \ref{subroutine:Nystrom}) up to time $t$. See that $(\cG_t)_{t \ge 1}$ defines a filtration, and both $E_{1,t}$ and $E_{2,t}$ are $\cG_t$ measurable.

\begin{mylemma}[Tighter confidence sets with Nystr\"{o}m embedding]
Fix any $\delta \in (0,1]$, $\epsilon \in (0,1)$ and set $\rho=\frac{1+\epsilon}{1-\epsilon}$. Then, ATA-GP-UCB under Nystr\"{o}m approximation, and with parameters $q=6\rho\ln(4T/\delta)/\epsilon^2$, $b_t = \left(v/\ln(4m_tT/\delta)\right)^{\frac{1}{1+\alpha}}t^{\frac{1-\alpha}{2(1+\alpha)}}$ and $\beta_{t+1}=B(1+\frac{1}{\sqrt{1-\epsilon}})+4 \sqrt{m_t/\lambda}\; v^{\frac{1}{1+\alpha}}\left(\ln(4m_tT/\delta)\right)^{\frac{\alpha}{1+\alpha}}t^{\frac{1-\alpha}{2(1+\alpha)}}$, ensures, with probability at least $1-\delta$, uniformly over all $t \in [T]$ and $x \in \cX$, that
\beqn
\abs{f(x)-\tilde{\mu}_{t-1}(x)} \le \beta_t\tilde{\sigma}_{t-1}(x), 
\eeqn
where $m_t$ is the dimension of the Nystr\"{o}m embedding $\tilde{\phi}_t$ constructed at round $t$.
\label{lem:func-conc-ata-gp-ucb-2}
\end{mylemma}

\begin{proof}
From Lemma \ref{lem:projection-values}, we have
$\tilde{\alpha}_t(x)=\inner{\phi(x)}{\hat{V}_t^{-1}\hat{\Phi}_t^T f_t}_\cH$. Therefore, 
\beqan
\abs{f(x)-\tilde{\alpha}_t(x)} &=& \abs{\inner{\phi(x)}{f-\hat{V}_t^{-1}\hat{\Phi}_t^T f_t}_\cH}\\
&\stackrel{(a)}{\le}&\norm{\phi(x)}_{\hat{V}_t^{-1}} \norm{f-\hat{V}_t^{-1}\hat{\Phi}_t^T f_t}_{\hat{V}_t}\\
&=&\lambda^{-1/2} \norm{(\hat{\Phi}_t^T\hat{\Phi}_t+\lambda I_\cH)f -\hat{\Phi}_t^T \Phi_t f}_{\hat{V}_t^{-1}}\tilde{\sigma}_t(x)\\
&\stackrel{(b)}{=}& \lambda^{-1/2} \norm{\lambda f-\hat{\Phi}_t^T\Phi_t(I_\cH-P_t)f}_{\hat{V}_t^{-1}}\tilde{\sigma}_t(x)\\
&\stackrel{(c)}{\le}& \left(\lambda^{1/2}\norm{\hat{V}_t^{-1/2}f}_\cH+\lambda^{-1/2}\norm{{\hat{V}_t^{-1/2}}\hat{\Phi}_t^T\Phi_t(I_\cH-P_t)f}_\cH\right)\tilde{\sigma}_t(x)\\
&\stackrel{(d)}{\le}& \left(\norm{f}_\cH+\lambda^{-1/2}\norm{{\hat{V}_t^{-1/2}}\hat{\Phi}_t^T}_\cH\norm{\Phi_t(I_\cH-P_t)}_\cH\norm{f}_\cH\right)\tilde{\sigma}_t(x)\\
&\stackrel{(e)}{\le}& B\left(1+\lambda^{-1/2}\norm{\Phi_t(I_\cH-P_t)}_\cH\right)\tilde{\sigma}_t(x).
\eeqan
Here $(a)$ is by Cauchy-Schwartz inequality, $(b)$ uses the fact that $\hat{\Phi}_t=\Phi_t P_t$, $(c)$ is by triangle inequality, $(d)$ follows from $\norm{\hat{V}_t^{-1/2}f}_\cH \le \lambda^{-1/2}\norm{f}_{\cH}$, and $(e)$ follows from the fact that $\norm{{\hat{V}_t^{-1/2}}\hat{\Phi}_t^T}^2_\cH = \lambda_{\max}\left(\hat{\Phi}_t(\hat{\Phi}_t^T\hat{\Phi}_t+\lambda I_\cH)^{-1}\hat{\Phi}_t^T\right)=\lambda_{\max}\left(\hat{\Phi}_t\hat{\Phi}_t^T(\hat{\Phi}_t\hat{\Phi}_t^T+\lambda I_t)^{-1}\right) \le 1$, and that $\norm{f}_\cH \le B$.
Now see that $\text{Col}(\Phi_{\cD_t}^T)=\text{Col}(\Phi_t^T S_t)$, and hence $P_t=\Phi_t^T S_t(S_t\Phi_t\Phi_t^T S_t)^{+}S_t\Phi_t$.  Therefore
\beqn
I_\cH-P_t \preceq I_\cH - \Phi_t^T S_t(S_t\Phi_t\Phi_t^T S_t+\lambda I_\cH)^{-1}S_t\Phi_t\stackrel{(a)}{=}\lambda (\Phi_t^T S_t^2\Phi_t+\lambda I_\cH)^{-1}=\lambda V_{\cD_t}^{-1},
\eeqn
where $(a)$ follows from \ref{eqn:dim-change-2}. Now given a filtration $\cG_t$ such that $E_{1,t}$ is true, we have $I_\cH-P_t \preceq \frac{\lambda}{1-\epsilon}V_t^{-1}$, and hence $\norm{\Phi_t(I_\cH-P_t)}^2_\cH=\lambda_{\max}\left(\Phi_t(I_\cH-P_t)\Phi_t^T\right) \le \frac{\lambda}{1-\epsilon} \lambda_{\max}\left(\Phi_t(\Phi_t^T\Phi_t+\lambda I_\cH)^{-1}\Phi_t^T\right)=\frac{\lambda}{1-\epsilon} \lambda_{\max}\left(\Phi_t\Phi_t^T(\Phi_t\Phi_t^T+\lambda I_t)^{-1}\right) \le \frac{\lambda}{1-\epsilon}$. Therefore, given a filtration $\cG_t$ such that $E_{1,t
}$ is true, 
\beq
\abs{f(x)-\tilde{\alpha}_t(x)} \le B\left(1+\frac{1}{\sqrt{1-\epsilon}}\right)\tilde{\sigma}_t(x).
\label{eqn:func-error-2}
\eeq
Now, we have $\tilde{\mu}_t(x)=\tilde{\phi}_t(x)^T\tilde{\theta}_t$ and $\tilde{\alpha}_t(x)=\tilde{\phi}_t(x)^T\tilde{V}_t^{-1}\tilde{\Phi}_t^T f_t$. Also observe that $\lambda \norm{\tilde{\phi}_t(x)}^2_{\tilde{V}_t^{-1}}=\tilde{\sigma}_t^2(x)+\tilde{\phi}_t(x)^T \tilde{\phi}_t(x)-k(x,x)=\tilde{\sigma}_t^2(x)-\inner{\phi(x)}{(I_\cH-P_t)\phi(x)}_\cH \le \tilde{\sigma}_t^2(x)$, since by definition $P_t \preceq I_\cH$. Then, by Cauchy-Schwartz inequality
\beqn
\abs{\tilde{\alpha}_t(x)-\tilde{\mu}_t(x)}  \le   \norm{\tilde{V}_t^{-1}\tilde{\Phi}^T f_t-\tilde{\theta}_t}_{\tilde{V}_t} \norm{\tilde{\phi}(x)}_{\tilde{V}_t^{-1}} \le \lambda^{-1/2}\norm{\tilde{V}_t^{-1}\tilde{\Phi}_t^T f_t-\tilde{\theta}_t}_{\tilde{V}_t}\tilde{\sigma}_t(x).
\eeqn
Now, Lemma \ref{lem:data-adaptive-truncation} implies that for any $\delta \in (0,1]$, with probability at least $1-\delta$, uniformly over all $t \in [T]$ and $x \in \cX$,
\beq
\abs{\tilde{\alpha}_t(x)-\tilde{\mu}_t(x)}  \le 4 \sqrt{m_t/\lambda}\; v^{\frac{1}{1+\alpha}}\left(\ln(2m_tT/\delta)\right)^{\frac{\alpha}{1+\alpha}}t^{\frac{1-\alpha}{2(1+\alpha)}}\tilde{\sigma}_t(x).
\label{eqn:approx-error-2}
\eeq
By triangle inequality,
\beqn
\abs{f(x)-\tilde{\mu}_t(x)} \le \abs{f(x)-\tilde{\alpha}_t(x)}  + \abs{\tilde{\alpha}_t(x)-\tilde{\mu}_t(x)}.
\eeqn
Now, combining \ref{eqn:func-error-2} and \ref{eqn:approx-error-2}, for any $\delta \in (0,1]$ and given a filtration $(\cG_t)_{t \ge 1}$ such that $E_{1,t}$ is true for all $t \in [T]$, we have, with probability at least $1-\delta$, uniformly over all $t \in [T]$ and $x \in \cX$,
\beqn
\abs{f(x)-\tilde{\mu}_t(x)} 
\le \left(B\left(1+\frac{1}{\sqrt{1-\epsilon}}\right)+4 \sqrt{m_t/\lambda}\; v^{\frac{1}{1+\alpha}}\left(\ln(2m_tT/\delta)\right)^{\frac{\alpha}{1+\alpha}}t^{\frac{1-\alpha}{2(1+\alpha)}}\right)\tilde{\sigma}_t(x).
\eeqn
From Lemma \ref{lem:dictionary-prop}, the event $E_{1,t}$ is true for all $t\in [T]$ with probability at least $1-\delta$. 
Now taking an union bound, we obtain that for any $\delta \in (0,1]$, with probability at least $1-\delta$,  uniformly over all $t \in [T]$ and $x \in \cX$,
\beqn
\abs{f(x)-\tilde{\mu}_t(x)} 
\le \left(B\left(1+\frac{1}{\sqrt{1-\epsilon}}\right)+4 \sqrt{m_t/\lambda}\; v^{\frac{1}{1+\alpha}}\left(\ln(4m_tT/\delta)\right)^{\frac{\alpha}{1+\alpha}}t^{\frac{1-\alpha}{2(1+\alpha)}}\right)\tilde{\sigma}_t(x).
\label{eqn:true-func-conc-2}
\eeqn
Further observe that $\abs{f(x)-\tilde{\mu}_0(x)}=\abs{f(x)} \le Bk^{1/2}(x,x) \le B(1+1/\sqrt{1-\epsilon})\tilde{\sigma}_0(x)$. Now,
the result follows by setting $\beta_{t+1} = B\left(1+\frac{1}{\sqrt{1-\epsilon}}\right)+4 \sqrt{m_t/\lambda}\; v^{\frac{1}{1+\alpha}}\left(\ln(4m_tT/\delta)\right)^{\frac{\alpha}{1+\alpha}}t^{\frac{1-\alpha}{2(1+\alpha)}}$ for all $t \ge 0$.
\end{proof}
Now we are ready to prove the regret bound of ATA-GP-UCB under Nystr\"{o}m approximation.
\subsubsection{Proof of Theorem \ref{thm:regret-bound-nystrom}}
For any $\delta \in (0,1]$, we have, with probability at least $1-\delta$, uniformly over all $t \in [T]$, the instantaneous regret  
\beqan
r_t &=& f(x^\star)-f(x_t)\\
& \stackrel{(a)}{\le} & \tilde{\mu}_{t-1}(x^\star) +  \beta_{t}\tilde{\sigma}_{t-1}(x^\star)- f(x_t)\\
& \stackrel{(b)}{\le} & \tilde{\mu}_{t-1}(x_t) +  \beta_{t}\tilde{\sigma}_{t-1}(x_t) -f(x_t)\\
& \stackrel{(c)}{\le} & 2\beta_{t}\tilde{\sigma}_{t-1}(x_t)
\eeqan
Here $(a)$ and $(c)$ follow from Lemma \ref{lem:func-conc-ata-gp-ucb-2}, and $(b)$ is due to the choice of ATA-GP-UCB (Algorithm \ref{algo:itgp-ucb}). From Lemma \ref{lem:dictionary-prop},
given a filtration $(\cG_t)_{t \ge 1}$ such that the event $E_{2,t}$ is true for all $t \in [T]$, we have $m_t=O\left(\frac{\rho^2}{\epsilon^2}\gamma_t\ln(T/\delta)\right)$. This, in turn, implies that 
\beqn
\beta_t = O\left(B\left(1+\frac{1}{\sqrt{1-\epsilon}}\right)+\frac{\rho}{\epsilon}\sqrt{\gamma_t\ln(T/\delta)}\;v^{\frac{1}{1+\alpha}}\left(\ln\left(\frac{\gamma_t \ln(T/\delta)T}{\delta}\right)\right)^{\frac{\alpha}{1+\alpha}}t^{\frac{1-\alpha}{2(1+\alpha)}}\right).
\eeqn
Further, given a filtration $(\cG_t)_{t \ge 1}$ such that the event $E_{1,t}$ is true for all $t \in [T]$, we have
\beqn
\sum_{t=1}^{T}\tilde{\sigma}_{t-1}(x_t) \stackrel{(a)}{\le} \rho \sum_{t=1}^{T}\sigma_{t-1}(x_t) \stackrel{(b)}{\le} \rho \sqrt{T\sum_{t=1}^{T}\sigma_{t-1}^2(x_t)} \stackrel{(c)}{\le} \rho\sqrt{2(1+\lambda)T\gamma_T}=O\left(\rho\sqrt{T\gamma_T}\right).
\eeqn
Here $(a)$ follows from Lemma \ref{lem:equivalence}, $(b)$ follows from Cauchy-Schwartz inequality, and $(c)$ follows from Lemma \ref{lem:pred-var}.
Now from Lemma \ref{lem:dictionary-prop}, with probability at least $1-\delta$, both $E_{1,t}$ and $E_{2,t}$ are true for all $t\in[T]$. Hence, by virtue of an union bound, we obtain that for any $\delta \in (0,1]$, with probability at least $1-\delta$, the cumulative regret of ATA-GP-UCB under Nystr\"{o}m appproximation after $T$ rounds is
\beqn
R_T = O\left(\rho B\left(1+\frac{1}{\sqrt{1-\epsilon}}\right)\sqrt{T \gamma_T} + \frac{\rho^2}{\epsilon} \; v^{\frac{1}{1+\alpha}}\left(\ln\left(\frac{\gamma_T \ln(T/\delta)T}{\delta}\right)\right)^{\frac{\alpha}{1+\alpha}}\sqrt{\ln(T/\delta)}\gamma_TT^{\frac{1}{1+\alpha}}\right).
\eeqn

\end{appendix}

\end{document}